\documentclass[british]{scrartcl}
\usepackage[T1]{fontenc}
\usepackage[utf8]{inputenc}
\usepackage{float}
\usepackage{bm}
\usepackage{amsmath}
\usepackage{accents}
\usepackage{amssymb}
\usepackage{graphicx}
\usepackage{pdfpages}
\usepackage{esint}
\usepackage[authoryear]{natbib}
\usepackage{amsthm}
\usepackage{multirow}
\usepackage{url}
\usepackage{color}

\usepackage{hyperref}
\usepackage{caption}
\usepackage{tikz}
\usepackage{pdfsync}
\usepackage{subcaption}
\usepackage{todonotes}
\usepackage{tcolorbox}
\usetikzlibrary{positioning}
\usetikzlibrary{patterns}
\usetikzlibrary{shapes}
\usetikzlibrary{plotmarks}
\usepackage[shortlabels]{enumitem}

\makeatletter

\floatstyle{ruled}
\newfloat{algorithm}{tbp}{loa}
\providecommand{\algorithmname}{Algorithm}
\floatname{algorithm}{\protect\algorithmname}

\usepackage{dsfont}\usepackage{algpseudocode}
\usepackage{mathtools}

\newcounter{rmq}[section]
\newcommand{\R}{\mathbb{R}}

\newcommand{\setY}{\mathsf{Y}}

\newcommand{\sigY}{\mathcal{Y}}

\newcommand{\csigma}{c_{\text{\tiny{\(\Sigma\)}}}}
\newcommand{\cess}{c_{\text{\tiny{\(\mathrm{ESS}\)}}}}
\renewcommand{\P}{\mathbb{P}}
\newcommand{\E}{\mathbb{E}}
\newcommand{\var}{\mathrm{Var}}

\newcommand{\F}{\mathcal{F}} 

\newcommand{\Unif}{\mathcal{U}} 

\newcommand{\ind}{\mathds{1}}
\newcommand{\dd}{\mathrm{d}}

\newcommand{\bigO}{\mathcal{O}} 
\newcommand{\smallo}{{\scriptscriptstyle\mathcal{O}}} 
\renewcommand{\emptyset}{\varnothing} 

\newtheorem{theorem}{Theorem}
\newtheorem{proposition}{Proposition}
\newtheorem{corollary}{Corollary}
\newtheorem{lemma}{Lemma}

\newtheorem{remark}{Remark}

\theoremstyle{plain}
\newtheorem{assumption}{Assumption}

\theoremstyle{plain}

\theoremstyle{plain}

\let\oldReturn\Return
\renewcommand{\Return}{\State\oldReturn}

\newcommand{\reqstart}{
    \begin{list}{\thereqcount}{\usecounter{reqcount}}
    \setcounter{reqcount}{\value{reqcountbackup}}
}

\newcommand{\reqend}{
    \setcounter{reqcountbackup}{\value{reqcount}}
    \end{list}
}

\DeclareMathOperator*{\argmin}{argmin}
\DeclareMathOperator*{\argmax}{argmax}

\makeatletter
\newcommand*{\vneq}{%
  \mathrel{%
    \mathpalette\@vneq{=}%
  }%
}
\newcommand*{\@vneq}[2]{%
  \sbox0{\raisebox{\depth}{$#1\neq$}}%
  \sbox2{\raisebox{\depth}{$#1|\m@th$}}%
  \ifdim\ht2>\ht0 %
    \sbox2{\resizebox{\vneqxscale\width}{\vneqyscale\ht0}{\unhbox2}}%
  \fi
  \sbox2{$\m@th#1\vcenter{\copy2}$}%
  \ooalign{%
    \hfil\phantom{\copy2}\hfil\cr
    \hfil$#1#2\m@th$\hfil\cr
    \hfil\copy2\hfil\cr
  }%
}
\newcommand*{\vneqxscale}{1}
\newcommand*{\vneqyscale}{1}

\makeatother

\usepackage{babel}

\title{A Global Stochastic Optimization  Particle Filter Algorithm }

\author{Mathieu Gerber$^{(1)}$, Randal Douc$^{(2)}$
\\
\small{(1) School of Mathematics, University of Bristol, UK }
\\
\small{(2) SAMOVAR, Telecom SudParis, Institut Polytechnique de Paris, France}
}
\date{}

\begin{document}
\maketitle

\begin{abstract}
We introduce a new online algorithm for  expected log-likelihood maximization in situations where the objective function is multi-modal and/or has saddle points, that we term G-PFSO. The  key element underpinning G-PFSO  is  a   probability distribution   which (a)  is shown  to concentrate on the target parameter value as the sample size increases and (b)  can be efficiently estimated  by means of a standard particle filter  algorithm.  This distribution  depends  on a learning rate, where the faster  the  learning rate  the quicker it  concentrates on the desired element of the search space,  but the less  likely G-PFSO is to escape from a local optimum of the objective function. In order to achieve a fast convergence rate with a slow learning rate, G-PFSO exploits the acceleration property of averaging, well-known in the stochastic gradient  literature. Considering several challenging estimation problems, the numerical experiments show that,  with high probability,  G-PFSO    successfully  finds the highest mode of the objective function  and   converges to its global maximizer at the optimal rate.  While the focus of this work is  expected log-likelihood maximization, the proposed methodology    and its theory apply  more generally for optimizing a function defined through an expectation.
\end{abstract}

\section{Introduction\label{sec:intro}}

\subsection{Set-up and problem formulation}
Let $(Y_t)_{t\geq 1}$  be a sequence of i.i.d.\  random variables defined on the same probability space $(\Omega,\F,\P)$ and taking values in a measurable space $(\setY,\sigY)$. Below we will often refer to $(Y_t)_{t\geq 1}$ as the  observations and to $t$ as the time at which $Y_t$ is observed. We let $\{f_{\theta},\,\theta\in \Theta\subseteq\R^d\}$ be a collection of probability density functions  on $\setY$  with respect to  some $\sigma$-finite measure  $\eta(\dd y)$  and we assume that $\theta_\star=\argmax_{\theta\in\Theta}\E\{\log f_\theta(Y_1)\}$ is well-defined, with $\E$ the expectation operator associated to $(\Omega,\F,\P)$.  In this work we consider the problem of estimating $\theta_\star$ in
an online fashion, where by online we mean that the memory and computational requirement  to process  $Y_t$ is finite and bounded  uniformly in $t$. Developing  methods  with a computational complexity scaling at most linearly with the number of data points  is particularly useful for parameter inference in large datasets 
\citep[][Chap.\ 2]{national2013frontiers}. We stress that we do not assume   the model   $\{f_{\theta},\,\theta\in \Theta\}$ to be  well-specified, i.e.\ that $Y_1\sim f_{\theta_\star}(y)\eta(\dd y)$. Consequently, as explained in  Section \ref{sub:setup},   the optimization problem  addressed in this paper is relevant to performing various learning tasks  and not only  to drawing inference in parametric models from i.i.d.\ observations.

Stochastic gradient   algorithms are popular  tools to learn the parameter value $\theta_\star$ on the fly  and at the optimal $t^{-1/2}$ rate \citep{toulis2015scalable}. However, when the objective function   $\theta\mapsto\E\{\log f_\theta(Y_1)\}$ is multi-modal and/or has saddle points these algorithms are only guaranteed to converge to one of its stationary points \citep{tadic2015convergence}.   \citet{gelfand1991recursive} have shown that, for such objective functions, we can  ensure the convergence of a standard stochastic gradient algorithm towards   $\theta_\star$ by adding some  extra noise  at each iteration. Unfortunately, this approach results in an algorithm    converging towards $\theta_\star$ at an extremely slow rate $r_t\rightarrow 0$ such that $(\log t)^{-1/2}=\smallo(r_t)$ \citep{pelletier1998weak, yin1999rates}. By contrast, the algorithm  recently introduced in \citet{gerber2020online} is proven to converge to the global optimum of the mapping $\theta\mapsto\E\{\log f_\theta(Y_1)\}$  at a nearly optimal rate. However, as its computational cost increases exponentially fast with the dimension of the parameter space $\Theta$, this algorithm can only be used in practice to tackle small or moderate dimensional optimization problems.

\subsection{Contribution of the paper}

We introduce in this work  a new global stochastic optimization method  to learn    the parameter value $\theta_\star$ in an online manner, that we term {\em Global Particle Filter  Stochastic Optimization}  (G-PFSO). Unlike the algorithm proposed in \citet{gelfand1991recursive},  we observe  empirically that G-PFSO converges to $\theta_\star$ at the optimal $t^{-1/2}$ rate and,  unlike the method developed in \citet{gerber2020online}, G-PFSO is implementable for   any dimension  of  $\Theta$. It should however be clear that, in practice, we cannot expect G-PFSO to perform well on any  optimization problems since, depending on the landscape of the objective function and   on the dimension of the search space,  finding  a small neighbourhood of  $\theta_\star$ with a reasonable computational cost  may simply be an intractable task  \citep{loshchilov2016sgdr}.

In G-PFSO the parameter $\theta_\star$ is learnt through a sequence $(\tilde{\pi}_t)_{t\geq 1}$ of probability distributions  which  is proven to  converge   to the Dirac mass $\delta_{\{\theta_\star\}}$ as $t\rightarrow\infty$. Each distribution  $\tilde{\pi}_t$ depends only on $\{Y_s\}_{s=1}^t$  and the sequence $(\tilde{\pi}_t)_{t\geq 1}$   can be easily  estimated in an online fashion by means of a standard particle filter   algorithm   \citep[][Chap.~10]{chopin2020introduction}. For every $t\geq 1$ the resulting particle filter approximation  $\tilde{\pi}_t^N$ of $\tilde{\pi}_t$, which has a finite support of size $N\in\mathbb{N}=\{1,2,\dots\}$, is then   used to compute an estimate of $\theta_\star$. 

The sequence $(\tilde{\pi}_t)_{t\geq 1}$ depends on a learning rate $h_t\rightarrow 0$, with the slower $h_t\rightarrow 0$   the greater  the ability of   $\tilde{\pi}_t^N$  to escape from a local optimum  of the objective function, but the slower   the rate at which this distribution can concentrate  on $\theta_\star$. Polyak–Ruppert  averaging is  a well-known   method for accelerating  stochastic gradient methods \citep{polyak1992acceleration} and, following this idea,  the G-PFSO estimator of $\theta_\star$ is 
$\bar{\theta}^N_t=t^{-1}\sum_{s=1}^t\tilde{\theta}_s^N$,  with $\tilde{\theta}_s^N=\int_{\Theta}\theta\tilde{\pi}_s^N(\dd\theta)$. In all the experiments presented below we observe that, for a fixed $N$ and a slow learning rate $h_t=t^{-1/2}$, the estimator $\bar{\theta}^N_t$ converges at the optimal $t^{-1/2}$ rate.

The theoretical analysis of the convergence behaviour of $\bar{\theta}_t^N$ as $t\rightarrow\infty$,   for a fixed $N\in\mathbb{N}$, is a hard task. In this paper we focus on a key preliminary step towards making such an analysis possible,  namely the study of the sequence $(\tilde{\pi}_t)_{t\geq 1}$.  More precisely, we derive a consistency result  for this sequence that holds under weak conditions on the statistical model and on the rate at which $h_t\rightarrow 0$, dealing  carefully   with the case where the parameter space $\Theta$ is unbounded. 

An \texttt{R} package implementing  the proposed global stochastic optimization algorithm is available  on GitHub at \texttt{github.com/mathieugerber/PFoptim}.

\subsection{Scope of the theoretical analysis\label{sub:setup}}

Our theoretical analysis of $(\tilde{\pi}_t)_{t\geq 1}$ does not assume that the statistical model $\{f_{\theta},\,\theta\in \Theta\}$ is well-specified, making it valid when G-PFSO is used to compute the maximum likelihood estimator  for a large class of statistical models.  To clarify this point consider   $n$ observations $\{ (z_i, x_i)\}_{i=1}^n$   in $\R^{d_z}\times\R^{d_x}$    assumed to be such  that, for all $i$, the conditional distribution of $Z_i$ given $X_i=x$ belongs to the set $\{f_\theta(\cdot\mid x),\,\theta\in\Theta\}$, and   let $f_\theta(z, x)=f_\theta(z\mid x)f_X(x)$ for some arbitrary  probability density function $f_X$  on $\R^{d_x}$. Then, letting $Y_1=(Z^{(n)}_{1},X^{(n)}_{1})$  be a random variable distributed according to the empirical distribution  of the observations  $\{ (z_i, x_i)\}_{i=1}^n$, we have
$\theta_\star= \argmax_{\theta\in\Theta}\E\big\{\log f_\theta(Y_1)\big\}=\argmax_{\theta\in \Theta}\sum_{i=1}^n \log f_{\theta}(z_i\mid x_i)=\hat{\theta}_{\mathrm{mle},n}$. Notice that this optimization problem does not depend  on $f_X$ and, as one may expect, using G-PFSO to compute $\hat{\theta}_{\mathrm{mle},n}$ does not require one to    specify this probability density function; see Algorithm \ref{algo:pi_tilde_online}.

A second important consequence of not assuming   the model $\{f_{\theta},\,\theta\in \Theta\}$ to be well-specified is that our main result is applicable to  estimation problems that are not limited to parameter inference in  parametric models. In particular, in order to compute $\theta_\star=\argmin_{\theta\in\Theta}\E\{\varphi(\theta,Y_1)\}$ for a  measurable function $\varphi:\Theta\times\setY\rightarrow [0,\infty)$  such that   $\int_{\setY}\exp\{-\varphi(\theta,y)\}\eta(\dd y)=1$ for every $\theta\in\Theta$  and for some $\sigma$-finite measure $\eta(\dd y)$,  G-PFSO can be used, and its theoretical guarantees apply, with the probability density function   $f_\theta(y)=\exp\{-\varphi(\theta,y)\}$. For instance, a classical machine learning task is to train  a function $\gamma_\theta$ to predict a response variable $Z$ from a vector $x$ of features, which amounts to computing $\theta_\star=\argmin_{\theta\in\Theta}\E[ L\{\gamma_{\theta}(X),Z\}]$ for some loss function $L(\hat{z},z)$. For this problem, and with $y=(z,x)$, the condition $\int_{\setY}\exp\{-\varphi(\theta,y)\}\eta(\dd y)=1$ holds  for $\varphi(\theta,y)=L\{\gamma_{\theta}(x),z\}$ when    $L$ is    the quadratic loss,  i.e.\ $L(\hat{z},z)=\|\hat{z}-z\|^2$, the absolute error loss, i.e.\ $L(\hat{z},z)=|\hat{z}-z|$  assuming $z\in\R$, or the cross-entropy loss, i.e.\ $L(\hat{z},z)=- z\log(\hat{z})-(1-z)\log(1-\hat{z})$  assuming $z\in\{0,1\}$ and $\hat{z}\in(0,1)$.

We stress that neither the definition of $\tilde{\pi}_t$, see Section \ref{sub:pi_tilde}, nor the   proof of our convergence result for the sequence $(\tilde{\pi}_t)_{t\geq 1}$ requires  that  $\int_\setY f_\theta(y)\eta(\dd y)=1$. Consequently,  our main theorem provides conditions on a measurable function $\varphi:\Theta\times\setY\rightarrow [0,\infty)$  which ensure that, for $f_\theta(y)=\exp\{-\varphi(\theta,y)\}$, the distribution $\tilde{\pi}_t$  concentrates  on $\theta_\star=\argmin_{\theta\in\Theta}\E\{\varphi(\theta,Y_1)\}$ as $t\rightarrow\infty$.  In particular, our theoretical analysis   applies when $\varphi(\theta,y)=\tilde{\varphi}(\theta)$ for some function $\tilde{\varphi}:\Theta\rightarrow\R$, that is, when G-PFSO is used to address a standard optimization task  where the objective function can be  easily  evaluated pointwise. 
We however stress that, our proof technique  coming from the literature on Bayesian asymptotics, our assumptions on $\{f_\theta,\,\theta\in\Theta\}$ are particularly standard  when this set of functions is a collection of probability density functions.

\subsection{Additional notation and outline of the paper\label{sub:set-up}}
 We let $\|\cdot\|_\infty$ be the maximum norm on $\R^d$,  $B_\epsilon(x)$  be the open ball of size $\epsilon>0$ around $x\in\R^d$ w.r.t.\ $\|\cdot\|$, the Euclidean norm on $\R^d$, $V_\epsilon=\Theta\setminus B_\epsilon(\theta_\star)$, $t_{d,\nu}(m,\Sigma)$ be the $d$-dimensional Student't-distribution  with $\nu>0$ degrees of freedom, location vector $m$ and scale matrix $\Sigma$, and  $\mathcal{N}_d(m,\Sigma)$ be the $d$-dimensional Gaussian distribution with mean $m$ and covariance matrix $\Sigma $.  If $A$ is a Borelian set of $\R^d$  we denote by $\mathcal{B}(A)$ the Borel $\sigma$-algebra on $A$, by  $\mathcal{P}(A)$   the set of probability measures on  $(A,\mathcal{B}(A))$ and by  $\mathcal{P}_L(A)$   the set of probability measures on  $(A,\mathcal{B}(A))$ that are absolutely continuous  w.r.t.\ $\dd\theta$, the Lebesgue measure on $\R^d$. Moreover, for a sequence $(\mu_t)_{t \geq 1}$ of probability measures on $\R^d$, the notation $\mu_t \Rightarrow \mu$ means that the sequence $(\mu_t)_{t \geq 1}$ converges weakly to the probability measure $\mu\in \mathcal{P}(\R^d)$. We assume throughout the paper that $\Theta\in\mathcal{B}(\R^d)$ and, for any probability measure $\mu$ on $\Theta$ and any sequence $(\mu_t)_{t \geq 1}$ of probability measures on  $\Theta$, implicitly indexed by random variables, we say that   $\mu_t \Rightarrow \mu$ in $\P$-probability  if the sequence of random variables $(\mathcal{D}_{\mathrm{pro}}(\mu_t,\mu))_{t \geq 1}$ converges in $\P$-probability to $0$, where  $\mathcal{D}_{\mathrm{pro}}$ denotes the Prohorov distance between probability measures on $\R^d$.  
 
The rest of the paper is organized as follows. The G-PFSO  approach is introduced   in Section \ref{sec:approach}, our convergence result  for the sequence  $(\tilde{\pi}_t)_{t\geq 1}$ is given in Section \ref{sec:main} and Section \ref{sec:num} proposes   some numerical experiments. All the proofs are gathered in Appendix \ref{ap-proofs}.

\section{Global particle filter stochastic optimization\label{sec:approach}}

\subsection{The sequence \texorpdfstring{$(\tilde{\pi}_t)_{t\geq 1}$}{Lg}}\label{sub:pi_tilde}

Let $\nu\in(0,\infty)$,  $\Sigma$ be a $d\times d$ covariance matrix, $h_t\rightarrow 0$ be a learning rate, that is $(h_t)_{t\geq 1}$ is a sequence in $(0,\infty)$ such that $h_t\rightarrow 0$, and   let $(t_p)_{p\geq 0}$  be a strictly increasing sequence in  $\mathbb{N}$ verifying $t_{p+1}-t_p\rightarrow\infty$ as $p\rightarrow\infty$. Then,  for   $\tilde{\pi}_0\in\mathcal{P}_L(\Theta)$  and  letting
\begin{align}\label{eq:mu_def}
\widetilde{M}_t(\theta',\dd\theta)=
\begin{cases}
\mathcal{N}_d(\theta',h_t^2\Sigma), &t\not\in (t_p)_{p\geq 0}\\
t_{d,\nu}(\theta',h_t^2\Sigma), & t \in (t_p)_{p\geq 0}
\end{cases}, \quad  \theta'\in\R^d,\quad   t\geq 1,
\end{align}
the sequence $(\tilde{\pi}_t)_{t\geq 1}$ is  defined by
\begin{align}\label{eq:pi_tilde}
 \tilde{\pi}_t(\dd\theta)=\frac{ f_{\theta}(Y_t)\big\{\int_{\R^d}\tilde{\pi}_{t-1}(\dd\theta') \widetilde{M}_{t-1}(\theta',\dd\theta)\big\}\mid_\Theta}{\int_{\R^d}f_{\theta}(Y_t)\big\{\int_{\R^d}\tilde{\pi}_{t-1}(\dd\theta') \widetilde{M}_{t-1}(\theta',\dd\theta) \big\}\mid_\Theta}\in \mathcal{P}_L(\Theta) ,\quad  t\geq 1
\end{align}
with the convention $\int_{\R^d}\tilde{\pi}_{t-1}(\dd\theta') \widetilde{M}_{t-1}(\theta',\dd\theta)=\tilde{\pi}_0(\dd\theta)$ when $t=1$ and where  $ \mu\mid_\Theta$ denotes the restriction of $\mu\in\mathcal{P}(\R^d)$ on $\Theta$. In \eqref{eq:pi_tilde} we implicitly assume that $\int_{\Theta}\dd \theta>0$.

To motivate the above definition of $(\widetilde{M}_t)_{t\geq 1}$  let us focus on the behaviour of the subsequence $(\tilde{\pi}_{t_p})_{p\geq 0}$. To this aim,  for every $t\geq 1$ we let  $\tilde{\Psi}_t:\mathcal{P}(\Theta)\rightarrow \mathcal{P}(\Theta)$ and $\Psi_t:\mathcal{P}(\Theta)\rightarrow \mathcal{P}(\Theta)$ be the  random mappings such that, for all $\pi\in\mathcal{P}(\Theta)$,
$$
\tilde{\Psi}_t(\pi)(\dd\theta)\propto f_\theta(Y_t)\Big\{\int_{\R^d} \pi(\dd\theta') \widetilde{M}_{t-1}(\theta',\dd\theta)\Big\}\mid_\Theta,\quad \Psi_t(\pi)(\dd\theta)\propto f_\theta(Y_t) \pi(\dd\theta).
$$
In this notation,  $\tilde{\pi}_{t_{p+1}}=\tilde{\Psi}_{t_{p+1}}\circ\dots\circ \tilde{\Psi}_{t_p+2}(\tilde{\pi}_{t_p+1})$ and, for  the sake of  argument, let $\pi'_{t_{p+1}}=\Psi_{t_{p+1}}\circ\dots\circ \Psi_{t_p+2}(\tilde{\pi}_{t_p+1})$  be the Bayesian posterior distribution associated with the observations $\{Y_t\}_{t=t_p+2}^{t_{p+1}}$ and  the prior distribution  $\tilde{\pi}_{t_p+1}$. Let $U$ be a small neighbourhood of $\theta_\star$. Then,   results on Bayesian asymptotics  \citep{Kleijn2012} ensure that if $t_{p+1}-t_p$ is large enough then  the mass of $\pi'_{t_{p+1}}$ on $U$ will be large with high probability, where the sample size $t_{p+1}-t_p$ required for this to be true depends on the  mass of the `prior distribution' $\tilde{\pi}_{t_p+1}$ around $\theta_\star$.  In particular, if this mass is small then $t_{p+1}-t_p$ needs to be large, i.e.\ a large sample size is necessary to compensate for a   poor  prior distribution. Informally speaking,  by letting    $\widetilde{M}_t$ be a  Gaussian kernel--that is a kernel with thin tails-- for all $t\not\in(t_p)_{p\geq 0}$,   and by letting $h_t\rightarrow 0$   fast enough, we can ensure  that  the mappings $\{\tilde{\Psi}_{t}\}_{t=t_{p}+2}^{t_{p+1}}$ are sufficiently close to the Bayes updates  $\{\Psi_t\}_{t=t_{p}+2}^{t_{p+1}}$ to enable $\tilde{\pi}_t$ to concentrate on $\theta_\star$ between time $t=t_p+2$ and time $t=t_{p+1}$. On the other hand, by taking  $\widetilde{M}_{t_p}$  to be a  Student's $t$-kernel--that is a kernel with heavy tails--we can compute a lower bound for  the mass of $\tilde{\pi}_{t_p+1}=\tilde{\Psi}_{t_p+1}(\tilde{\pi}_{t_p})$ around $\theta_\star$ that holds uniformly in $\tilde{\pi}_{t_p}$ and which does not converge  to zero too quickly as $p\rightarrow\infty$. Together with a suitable definition of $t_{p+1}-t_p$  this lower bound allows us to obtain a precise control of $\tilde{\pi}_{t_{p+1}}(U)$ which, in particular, makes it possible to  show that $\tilde{\pi}_{t_{p+1}}(U)\rightarrow 1$ as $p\rightarrow\infty$.

Our main result, Theorem \ref{t-thm:online} in Section \ref{sec:main},    provides  conditions on the learning rate  $h_t\rightarrow 0$ and on  $(t_p)_{p\geq 0}$ which ensure the convergence of $\tilde{\pi}_t$ towards $\delta_{\{\theta_\star\}}$ in $\P$-probability, under standard assumptions on the statistical  model. For instance, Corollary \ref{cor:intro} in Section \ref{sec:main} gives sufficient conditions on  $\{f_\theta,\,\theta\in\Theta\}$ to guarantee that $\tilde{\pi}_t\Rightarrow  \delta_{\{\theta_\star\}}$ in $\P$-probability when  $h_t=t^{-\alpha}$ and when, for some $\varrho\in(0,\alpha\wedge 1)$, the sequence $(t_p)_{p\geq 0}$ is defined by 
\begin{align*}
t_0\in  \mathbb{N}\quad   t_p= t_{p-1}+\lceil t^\varrho_{p-1} \log(t_{p-1})\rceil,\quad   p\geq 1.
\end{align*} 
It is important to mention at this stage that these conditions on $\{f_\theta,\,\theta\in\Theta\}$ do not depend on $\alpha$ when $\Theta$ is a bounded set. When $\Theta$ is unbounded the smaller $\alpha$ is  the stronger   the assumptions on the statistical model  imposed by our main result. However, as shown in Section \ref{sec:main}, Proposition \ref{prop:assumption}, even when $\Theta$ is unbounded we may have $\tilde{\pi}_t\Rightarrow\delta_{\{\theta_\star\}}$ in $\P$-probability for any learning rate of the form $h_t=t^{-\alpha}$. To summarize, whenever $\Theta$ is bounded, and for some models defined on an unbounded parameter space, we show that  $\tilde{\pi}_t$ converges  towards $\delta_{\{\theta_\star\}}$  even when $h_t\rightarrow 0$ at an arbitrarily  slow   polynomial rate. As explained in Section \ref{sub:h_t}, defining $\tilde{\pi}_t$ with a slow learning rate $h_t\rightarrow 0$  is of particular importance  when the  function $\theta\mapsto\E\{\log f_\theta(Y_1)\}$ has several modes.

\subsection{The G-PFSO algorithm  \label{sub:Monte_Carlo1}}

The G-PFSO algorithm is presented in Algorithm \ref{algo:pi_tilde_online}, which reduces to a simple particle filter algorithm for approximating  the sequence $(\tilde{\pi}_t)_{t\geq 1}$ in an online fashion.  For every $t\geq 1$    the particle filter estimate $\tilde{\pi}_t^N$ of $\tilde{\pi}_t$ is  $
\tilde{\pi}_t^N=\sum_{n=1}^N W_t^n\delta_{\theta_t^n}$ while the G-PFSO estimator $\bar{\theta}_t^N$  of $\theta_\star$ is computed on the last line of the algorithm, where $\tilde{\theta}_t^N=\int_\Theta \theta \tilde{\pi}_t^N(\dd \theta)$  is a particle filter approximation of $\tilde{\theta}_t=\int_{\Theta}\theta\tilde{\pi}_t(\dd\theta)$. Noting that   $\bar{\theta}_t^N=t^{-1}\sum_{s=1}^t\tilde{\theta}_s^N$, it follows that $\bar{\theta}_t^N$ is the Polyak–Ruppert  averaging   of $\{\tilde{\theta}^N_s\}_{s=1}^t$. Averaging is a well-known acceleration technique  in the literature on stochastic gradient algorithms which is illustrated in the next subsection.

 In Algorithm \ref{algo:pi_tilde_online}  the resampling algorithm $\mathcal{R}(\cdot)$ is such that $\mathcal{R}(\{x^n,p^n\}_{n=1}^N)$ is a probability distribution on the set $\{x^1,\dots,x^N\}^N$, where $(p^{1},\dots, p^N)\in[0,1]^N$, $\sum_{n=1}^N p^n=1$ and $x^n\in\R^d$ for all $n=1,\dots,N$. We refer the reader to Chap.\ 9 of \citet{chopin2020introduction}  for a detailed discussion of resampling methods, and to  Chap.\ 10 of this reference for explanations concerning the role of the parameter $\cess\in(0,1]$ appearing in Algorithm \ref{algo:pi_tilde_online}.
 
\begin{algorithm}[t]
\begin{algorithmic}[1]
\Require $N\in\mathbb{N}$, $\cess\in(0,1]$ and a resampling algorithm $\mathcal{R}(\cdot, \cdot)$
\vspace{0,2cm}

\State\label{line:for}Let $\theta_1^n\sim \tilde{\pi}_0(\dd\theta_0)$ and set $w_1^n=  f_{\theta_1^n}(Y_1)$ and $W_{1}^n=w_1^n/\sum_{m=1}^N w_1^m$

\State\label{line_est1}Let  $\tilde{\theta}_1^N=\sum_{n=1}^N W_1^n\theta_1^n$ and $\bar{\theta}_1^N=\tilde{\theta}_1^N$

\For{$t\geq 2$}
\vspace{0,1cm}

\State Set $\mathrm{ESS}_{t-1}= 1/\sum_{m=1}^N (W_{t-1}^m)^2$
\vspace{0,1cm}

\If{$\mathrm{ESS}_{t-1}\leq  N\,\cess$}\label{ess}
\vspace{0,1cm}

\State\label{line:R}  Let  $\{\hat{\theta}^{1}_{t-1},\dots,\hat{\theta}^{N}_{t-1}\}\sim \mathcal{R}(\{\theta_{t-1}^n,W_{t-1}^{n}\}_{n=1}^N,N)$ and set $w_{t-1}^n=1$
\Else
\State Let     $\hat{\theta}^{n}_{t-1}=\theta^{n}_{t-1}$

\EndIf
\vspace{0,1cm}

\If{$(t-1)\in(t_p)_{p\geq 0}$}
\vspace{0,1cm}

\State\label{SS1}  Set $\theta_t^n=\hat{\theta}^n_{t-1} +  h_{t-1} \epsilon^n_{t-1}$  where $\epsilon^n_{t-1}\sim t_{d,\nu}(0, \Sigma)$
\Else
\State\label{SS2}  Set $\theta_t^n=\hat{\theta}^n_{t-1} + h_{t-1} \epsilon^n_{t-1}$  where $\epsilon^n_{t-1}\sim \mathcal{N}_d(0,\Sigma)$
\EndIf
\State\label{line:last} Set $w_t^n= w_{t-1}^n f_{\theta_t^n}(Y_t)$ if $\theta_t^n\in\Theta$ and $w_t^n=0$ otherwise, and $W_{t}^n=w_t^n/\sum_{m=1}^N w_t^m$

\State\label{line:est} Let  $\tilde{\theta}_t^N=\sum_{n=1}^N W_t^n\theta_t^n$ and $\bar{\theta}_t^N=t^{-1}\big((t-1)\bar{\theta}_{t-1}^N+ \tilde{\theta}_t^N\big)$
\EndFor
\vspace{0,2cm}

\end{algorithmic}
\caption{Global PF Stochastic Optimization\\
(\small{Operations with index $n$ must be performed for all $n\in \{1,\dots,N\}$.}) \label{algo:pi_tilde_online}}
\end{algorithm}

\subsection{Accelerating property of Polyak–Ruppert  averaging: an illustrative example}\label{subsec:example}

For every $\theta\in\Theta=\R$ we let $f_\theta(y)\eta(\dd y)=\mathcal{N}_1(\theta,1)$ and,  for some $\tilde{\theta}_0\in\R$ and $\tilde \sigma_0^2\in(0,\infty)$,   we  let $\tilde{\pi}_0=\mathcal{N}_1(\tilde{\theta}_0,\tilde{\sigma}_0^2)$.  To make the computation  of $(\tilde{\pi}_t)_{t\geq 1}$ tractable  we  relax the condition  that  the sequence $(\widetilde{M}_t)_{t\geq 1}$   contains infinitely many Student's $t$-distributions by letting $\widetilde{M}_t(\theta',\dd\theta)=  \mathcal{N}_1(\theta',h_t^2)$ for all $t\geq 1$. Then,  with    $g(x)=x/(1+x)$ for all $x\in (0,\infty)$, we have
\begin{align}\label{eq:algo_theta}
\tilde{\pi}_t=\mathcal{N}_1(\tilde{\theta}_t,\tilde{\sigma}_t^2),\quad \tilde{\sigma}_t^2=g(\tilde{\sigma}^2_{t-1}+h^2_{t-1}),\quad \tilde{\theta}_t=\tilde{\theta}_{t-1}+\tilde{\sigma}_t^2 (Y_t -\tilde{\theta}_{t-1}),\quad t\geq 1. 
\end{align}

In Table \ref{table:ex} we study the convergence rate  of $\tilde{\theta}_t$ and of  $\bar{\theta}_t=t^{-1}\sum_{s=1}^t\tilde{\theta}_s$ when  $h_t=t^{-\alpha}$. The results reported in this table, obtained for all $\alpha\in\{0.1,0.3,0.5,0.7,1\}$, suggest that $\tilde{\theta}_t$ converges to $\theta_\star$  at rate  $t^{-\alpha/2}$ while $\bar{\theta}_t$    converges to this target parameter value at the optimal $t^{-1/2}$ rate.

Noting that the update  \eqref{eq:algo_theta} for $\tilde{\theta}_t$ reduces to that of a standard stochastic gradient algorithm with  step sizes $(\tilde{\sigma}_t^2)_{t\geq 1}$, the $t^{-1/2}$ convergence rate for $\bar{\theta}_t$ observed in this example proceeds from well-known results on the acceleration property of averaging for this class of algorithms \citep{polyak1992acceleration}.  However, the numerical experiments presented  in Section \ref{sec:num} suggest that, within G-PFSO, the acceleration property of averaging holds beyond the simple example considered above,  although the total insensitivity of the convergence rate of $\bar{\theta}_t$ to $\alpha$ observed in Table \ref{table:ex} does not appear to hold in general.

\renewcommand{\arraystretch}{1.5}
\begin{table}
\centering
\begin{tabular}{c|ccccc}
$Z_t\backslash\alpha$ &0.1&0.3&0.5&0.7&1\\
\hline 
$|\tilde{\theta}_t-\theta_\star|$&0.05&0.15&0.24&0.36&0.45\\

$|\bar{\theta}_t-\theta_\star|$&0.48&0.48&0.48&0.48&0.50\
\end{tabular}\caption{Example of  Section \ref{subsec:example}. Ordinary least square estimate of $\beta_2$ in the  model $\log(Z_t)=\beta_1-\beta_2\log(t)+\epsilon_t$, where $Z_t$ is as defined in the table, $t\in\{10^5,10^5+1,\dots,10^7\}$ and where $\tilde{\pi}_t$ is  as defined in \eqref{eq:algo_theta}  with $h_t=t^{-\alpha}$, $(\tilde{\theta}_0,\tilde{\sigma}_0^2)=(0,25)$ and with $Y_1\sim \mathcal{N}_1(0,1)$ (so that $\theta_\star=0$). \label{table:ex}}
\end{table}

\subsection{Choosing the learning rate $h_t$\label{sub:h_t}}

As illustrated in  Section \ref{subsec:example},  the faster $h_t\rightarrow 0$ the closer  to the optimal $t^{-1/2}$ rate is the  rate at which $\tilde{\pi}_t$ learns $\theta_\star$. However, it is clear from Algorithm \ref{algo:pi_tilde_online} and from the definition of $(\widetilde{M}_t)_{t\geq 1}$ that, for a fixed value of $N$, the smaller  the value of $h_t$  the  higher  the probability of having the support of $\tilde{\pi}_{t+1}^N$  close to that of $\tilde{\pi}_{t}^N$. Consequently, the ability of G-PFSO to explore $\Theta$ deteriorates as the learning rate $h_t\rightarrow 0$ become faster. In particular, when the initial particles $\{\theta_1^n\}_{n=1}^N$ are far  from $\theta_\star$ and $h_t$ decreases quickly over time,  $\tilde{\pi}^N_t$ may fail to reach a small neighbourhood    of $\theta_\star$, and thus to approximate   $\tilde{\pi}_t$ well, even  for large values of $t$.

 With this  trade-off involved when choosing $h_t$, between statistical efficiency, i.e.\ enabling $\tilde{\pi}_t$ to concentrate quickly on $\theta_\star$, and computational efficiency, i.e.\ making $\tilde{\pi}_t^N$ close to $\tilde{\pi}_t$ for a reasonable value of $N$, our recommendation is to take   $h_t=t^{-1/2}$ as the default choice for the learning rate. This recommendation is based on the fact that, for $h_t=t^{-\alpha}$, in all of our  numerical experiments  the estimator $\bar{\theta}^N_t$ appears to  converge to $\theta_\star$ at the optimal $t^{-1/2}$ rate when $\alpha=0.5$, while  this is not  the case for a smaller value  of this parameter; see the example of Section \ref{sub:CQR_5}.  It is worth noting that for $h_t=t^{-1/2}$ we have $\sum_{t=1}^\infty h_t^2=\infty$, meaning that if in Algorithm \ref{algo:pi_tilde_online} we write $\theta_t^n$ as $\theta_t^n=\hat{\theta}_{t-1}^n+\epsilon_t^n$, with $\epsilon_t^n\sim \widetilde{M}_t(0,\dd\theta_t)$, then $\lim_{t\rightarrow\infty}\|\var\big(\sum_{s=1}^t\epsilon_s^n)\|=\infty$. In  words, at any given time $t$ the sum of all the future noise terms that   will be used to propagate a given particle has an infinite variance, a  property that may  help G-PFSO to escape from a local mode of the objective function even after having processed a large number of observations; see Sections \ref{sub:Multimodal}-\ref{sub:mixture} for examples where this phenomenon happens.

\subsection{Discussion\label{sub:extension}}

If in \eqref{eq:mu_def} we define  $(\widetilde{M}_{t})_{t\geq 1}$ using Gaussian and Student's $t$-distributions only, our theoretical analysis of $\tilde{\pi}_t$ applies more generally for Markov kernels $(\widetilde{M}_{t})_{t\geq 1}$ whose tails verify certain conditions, as discussed in Appendix \ref{psec:extensions}. 

In Appendix \ref{psec:extensions} we also establish that,  for $h_t=t^{-\alpha}$,  taking $\nu\leq 1/\alpha$ guarantees that   with $\P$-probability one  there is no value of $t$  beyond which all the particles generated by Algorithm \ref{algo:pi_tilde_online} remain stuck in a local mode of the mapping $\theta\mapsto\E\{\log f_\theta(Y_1)\}$. From a theoretical point of view this is an important  result, because  this property of G-PFSO is necessary to enable $\bar{\theta}_t^N$ to converge  in $\P$-probability  towards $\theta_\star$ for a fixed $N\in\mathbb{N}$. It is however important to stress that the capacity of the particles to escape from a local mode of the objective function deteriorates over time, and thus that even for $\nu\leq 1/\alpha$  all the particles may  be stuck in a local mode for a very long time. Finally, we remark that if we let $h_t=t^{-1/2}$, as recommended in the previous subsection, then the condition $\nu\leq 1/\alpha$ imposes to use, in \eqref{eq:mu_def}, Student's $t$-distributions having an infinite variance.

\subsection{Related approaches}\label{sub:lit_review}

The idea of using  particle filter algorithms, or more generally sequential Monte Carlo methods, for optimizing  a  function $\varphi:\Theta\rightarrow \R$ that can be easily evaluated pointwise   has been considered e.g.\ in \citet{zhou2008particle}, \citet{liu2016particle} and in \citet{giraud2013convergence,giraud2017nonasymptotic}. By contrast, it is only recently that    optimizing  a function defined through an expectation by means of a  particle filter  has been proposed, notably  in \citet{akyildiz2020parallel} and in \citet{liu2020particle}. When used to estimate $\theta_\star$,   the approaches introduced in these two references amount to approximating in an online fashion the Bayesian posterior distributions  $\big\{\pi_t\}_{t=1}^T$   and   then to using the resulting  estimate of  $\pi_T$  to learn the target  parameter value. 

More precisely, the     parallel sequential Monte Carlo optimizer  of \citet{akyildiz2020parallel} relies  on the estimate $\pi_{\mathrm{Jit},t}^N=\sum_{n=1}^N W_t^n\delta_{\{\theta_t^n\}}$ of $\pi_t$, where $\{(W_t^n,\theta_t^n)\}_{n=1}^N$ can be computed using Algorithm \ref{algo:pi_tilde_online} by replacing, for all $t\geq 1$, the Markov kernel $\widetilde{M}_t$ by a  jittering kernel $\mathsf{M}_N$ as introduced in \citet{crisan2018nested}.  The  distribution $\pi_{\mathrm{Jit},t}^N$ is shown to converge to $\pi_t$  as $N\rightarrow\infty$.  However, our numerical experiments reveal an important limitation of using $\pi_{\mathrm{Jit},t}^N$ to estimate $\theta_\star$, namely that for a fixed value of $N$ there exists some finite time $t_N$ after which  processing $Y_t$ does not allow $\pi_{\mathrm{Jit},t}^N$ to provide any new information  about the target parameter value; see Section \ref{sub:CQR_5} for an illustration. This issue arises because the Markov kernel $M_{N}$ on which  $\pi_{\mathrm{Jit},t}^N$ relies is time homogenous,  which prevents  the support of this distribution from concentrating on a particular element of $\Theta$ as $t$  increases. 

The   estimate $\pi_{\mathrm{K},t}^N$ of $\pi_t$ used in the kernel smoothing particle filter based stochastic optimization (KS-PFSO) algorithm of \citet{liu2020particle} is computed as the jittering approximation $\pi^N_{\mathrm{Jit},t}$, the only difference being that in KS-PFSO, for all $t\geq 1$, the Markov kernel $\widetilde{M}_t$ appearing in Algorithm \ref{algo:pi_tilde_online}    is  replaced by  the kernel $M_{\mathrm{K},t}$ defined  by
\begin{align}\label{eq:kpfso}
M_{\mathrm{K},t}(\theta',\dd\theta_t)=\mathcal{N}_d\Big(\sqrt{1-\iota^2}~ \theta'+(1-\sqrt{1-\iota^2})\theta_{\mathrm{K},t-1}^N, \iota^2 V_{\mathrm{K},t-1}^N\Big),\,\,\theta'\in\R^d
\end{align}
for some $\iota>0$ and where $\theta_{\mathrm{K},t-1}^N$ and $V_{\mathrm{K},t-1}^N$ respectively denotes the expectation and the covariance matrix of $\theta$ under $\pi_{\mathrm{K},t-1}^N$. Remark that, under this kernel  and conditionally to $\{(W_{t-1}^n,\theta_{t-1}^n)\}_{n=1}^N$, each  particle $\theta_{t}^n$ generated by Algorithm \ref{algo:pi_tilde_online}  has the same expectation and covariance matrix than a random draw from   $\pi_{\mathrm{K},t-1}^N$. This kernel was introduced by \citet{liu2001combined} in the context of online Bayesian state and parameter learning in state-space models, where it is used to rejuvenate the particle system without inflating the tails of the current approximation of the posterior distribution for the model parameter. Under standard regularity conditions  $\pi_t$  concentrates on $\theta_\star$ at rate $t^{-1/2}$ \citep{Kleijn2012} and  it is  therefore   expected that   $\pi_{\mathrm{K},t}^N$  concentrates at this rate on a particular element of $\Theta$.  The numerical experiments presented in Section \ref{sub:Multimodal}  show that this quick concentration of $\pi_{\mathrm{K},t}^N$    over time  limits considerably the ability of KS-PFSO to escape from a local optimum of the objective function, making this algorithm  a local rather than a global optimization method. It is worth mentioning that, unlike G-PFSO and the sequential Monte Carlo optimizer proposed by \citet{akyildiz2020parallel}, KS-PFSO is not introduced by \citet{liu2020particle}  as a global optimization method but as a stochastic optimization algorithm which bypasses the need to specify a learning rate.
 We also remark that  $\pi_{\mathrm{K},t}^N$  is actually the approximation of   $\pi_t$ computed by the  one-pass SMC  sampler of \cite{balakrishnan} and is not supported by any theoretical results. 

It is interesting to note  that the only difference between the algorithm used in G-PFSO to estimate $(\tilde{\pi}_t)_{t\geq 1}$ and  those used by \citet{akyildiz2020parallel} and by \citet{liu2020particle} to approximate the posterior distributions $(\pi_t)_{t\geq 1}$ is  the choice of the Markov kernels that are employed to rejuvenate the particle system. In  G-PFSO the variance of $\theta_{t+1}$ under $\widetilde{M}(\theta_t,\dd\theta_{t+1})$ is of size $\bigO(h_t^2)$ and thus converges to zero as $t\rightarrow \infty$. Our numerical experiments suggest that this property of  $\widetilde{M}_t$ enables  $\tilde{\pi}_t^N$  to concentrate on a single point of the parameter space as $t$ increases,  unlike the approximation $\pi_{\mathrm{Jit},t}^N$ of $\pi_t$. In addition, and as   mentioned in Section \ref{sub:h_t}, by choosing  $h_t\rightarrow 0$ such that $\tilde{\pi}_t$ concentrates on $\theta_\star$ at a sub-optimal rate, that is at a rate slower that $t^{-1/2}$, we can improve the ability of G-PFSO to escape from a local optimum of the  objective function $\theta\mapsto\E\{\log f_\theta(Y_1)\}$. For this reason, G-PFSO is more suitable than KS-PFSO for multi-modal  optimization tasks.

\section{Theoretical analysis of the sequence \texorpdfstring{$(\tilde{\pi}_t)_{t\geq 1}$}{Lg}}\label{sec:main}

\subsection{Assumptions on the statistical model}

To simplify the notation  we      use below the   shorthand $\E(g)$ for $\E\{g(Y_1)\}$  and, for every $\theta\in\R^d$, we let $\tilde{f}_\theta:\R^d\rightarrow\R$ be such that $\tilde{f}_\theta\equiv 0$ if $\theta\not\in\Theta$ and such that $\tilde{f}_\theta= f_\theta$ if $\theta\in\Theta$.

The following two assumptions impose some regularity on the random mapping $\theta\mapsto \log f_\theta(Y_1)$ around $\theta_\star$.

\begin{assumption}\label{m_star}
There exist  a  constant $\delta_{\star}>0$ and a measurable function  $m_{\star}\,{:}\,\setY\rightarrow \R$ such that $\E[m_\star^{2}]<\infty$ and such that $\P$-a.s.\ we have, for all  $\theta_1,\theta_2\in B_{\delta_\star}(\theta_\star)\cap\Theta$, 
$$
\big|\log\{f_{\theta_1}(Y_1)/f_{\theta_2}(Y_1)\}\big|\leq m_{\star}(Y_1)\|\theta_1-\theta_2\|.
$$
\end{assumption}

\begin{assumption}\label{taylor}
There exist constants $\delta_\star>0$ and  $C_\star<\infty$ such that, for all $\theta\in B_{\delta_\star}(\theta_\star)\cap\Theta$,
$$
\E\big\{\log(f_{\theta_\star}/f_{\theta})\big\}\leq C_\star \|\theta-\theta_\star\|^2.
$$

\end{assumption}

The next assumption notably implies that $\theta_\star$ is identifiable.

\begin{assumption}\label{test}
For every  compact set $W\in\mathcal{B}(\Theta)$ such that $\theta_\star\in W$ and every $\epsilon\in(0,\infty)$ there exists a sequence of measurable functions $(\psi'_t)_{t\geq 1}$, with $\psi'_t:\setY^t\rightarrow\{0,1\}$, such that
$$
\E\{\psi'_t(Y_{1:t})\}\rightarrow 0,\quad\sup_{\theta\in V_\epsilon\cap W}\E\Big[\big\{1-\psi'_t(Y_{1:t}) \big\}\prod_{s=1}^t (f_{\theta}/f_{\theta_\star})(Y_s)\Big]\rightarrow 0.
$$
\end{assumption}

Assumptions \ref{m_star}-\ref{test} are   standard,  see e.g.\ \citet{Kleijn2012}. It is easily checked that Assumption \ref{taylor} holds when the  mapping $\theta\mapsto\E(\log  f_{\theta})$ admits a second-order Taylor expansion in a  neighbourhood of $\theta_\star$.
By \citet[][Theorem 3.2]{Kleijn2012}, Assumption \ref{test} holds for instance when, for every compact set $W\in\mathcal{B}(\Theta)$  containing $\theta_\star$ and  every $  \theta'\in W$,    the mapping $W\ni \theta\mapsto \E(f_\theta f_{\theta'}^{-s}f_{\theta_\star}^{s-1})$
is continuous at $\theta'$  for every $s$ in a left neighbourhood of 1, and  $\E(f_{\theta'}/f_{\theta_\star})<\infty$.  Remark that if the model is well-specified then  $\E( f_{\theta}/f_{\theta_\star}  )=1<\infty$ for all $\theta\in\Theta$. If the model is miss-specified and  the distribution of $Y_1$ admits a density $f_\star$ w.r.t.\ $\eta(\dd y)$, the condition  $\E( f_{\theta}/f_{\theta_\star} )<\infty$ requires   the tails of $f_{\theta_\star}$ to be  not too thin compared to those of $f_\star$. For instance, if $f_\theta(y)\eta(\dd y)=t_{1,\nu}(\mu,\sigma^2)$, with $\theta=(\mu,\sigma^2,\nu)$, and $\lim_{|y|\rightarrow\infty} |y|^{\nu_\star+1}f_\star(y)<\infty$   then  $\E( f_{\theta}/f_{\theta_\star} )<\infty$ for all $\theta\in\R\times(0,\infty)^2$.

\begin{remark}
Assumption  \ref{test} is stronger than needed for our  main result to hold,  this latter requiring only   that Assumption \ref{test} holds for the specific compact set $W=\tilde{A}_\star\cap\Theta$, where $\tilde{A}_\star$ is as defined in  Assumption \ref{new} below.
\end{remark}
 
\begin{assumption}\label{new}
For some set  $A_\star\in\mathcal{B}(\R^d)$,
\begin{enumerate}

\item\label{new_part1} One of the following conditions holds:
\begin{enumerate}[a)]
\item\label{A41} $
\E(\sup_{\theta \not\in A_\star}\log  \tilde{f}_{\theta})< \E( \log  f_{\theta_\star})$,  %
\item\label{A42}  $\sup_{\theta \not\in A_\star }\E(\tilde{f}_\theta/f_{\theta_\star} )< 1$,
\item\label{A43}  $\log\{\sup_{\theta \not\in A_\star }\E(\tilde{f}_\theta)\}<  \E(\log f_{\theta_\star})$.
\end{enumerate}
\item\label{new_part2} There exists a  set  $\tilde{A}_\star\in\mathcal{B}(\R^d)$,  containing a neighbourhood  of $A_\star$, such that  $\tilde{A}_\star\cap\Theta$ is compact, the mapping $\theta\mapsto   f_\theta(Y_1)$ is $\P$-a.s.\ continuous on      $\tilde{A}_\star\cap\Theta$ and, for some  $\tilde{\delta}>0$,
$$
\E\Big[\sup_{(\theta,v)  \in (\tilde{A}_\star\cap\Theta)\times B_{\tilde{\delta}}(0)}  \big\{\log (\tilde{f}_{\theta+v}/f_{\theta})\big\}^2\Big]<\infty.
$$
\end{enumerate}
\end{assumption}
\begin{remark}
If $\Theta$ is compact then Assumption \ref{new} holds as soon as $\theta\mapsto   f_\theta(Y_1)$ is $\P$-a.s.\ continuous on  $\Theta$ and, for some  $\tilde{\delta}>0$, $\E\big[\sup_{ (\theta,\theta') \in \Theta^2:\,\,\|\theta'-\theta\|< \tilde{\delta}}  \{\log (f_{\theta'}/f_{\theta})\}^2\big]<\infty$.
\end{remark}

\begin{assumption}\label{new1}
One of the following three conditions hold:
\begin{enumerate}
\item\label{A51}  $\E(\sup_{\theta\in\Theta}\log f_\theta)<\infty$, 
\item\label{A52}  $\sup_{\theta\in\Theta}\E( f_\theta/f_{\theta_\star})<\infty$,
\item\label{A53}  $\sup_{\theta\in\Theta}\E(f_\theta)<\infty$.
\end{enumerate}

\end{assumption}

\begin{remark}
Assumption \ref{new1}.\ref{A52} always holds when the model is well-specified.
\end{remark}

The last assumption, Assumption \ref{new2} below,   is  used to obtain a convergence result for $\tilde{\pi}_t$ that holds when  $\Theta$ is unbounded.

\begin{assumption}\label{new2}
One of the following three conditions hold for some $ k_\star\in\{1/2\}\cup \mathbb{N}$:
\begin{enumerate}
\item\label{A61} There exists a constant $C_1\in(0,\infty)$ such that
$$
\sup_{C\geq C_1}\E\Big[\Big|\sup_{\theta\in V_{C}}\log (\tilde{f}_{\theta}/f_{\theta_\star})-\E\big\{\sup_{\theta\in V_{C}}\log (\tilde{f}_{\theta}/f_{\theta_\star})\big\}\Big|^{2k_\star}\Big]<\infty 
$$
and $\limsup_{C\rightarrow\infty} \zeta(C) (\log C)^{-1}<0$  with $ \zeta(C)=\E\big\{ \sup_{\theta\in V_{C}}\log (\tilde{f}_{\theta}/f_{\theta_\star}) \big\}$,
\item\label{A62} $\limsup_{C\rightarrow\infty}  \zeta(C) (\log C)^{-1}<0$  with $ \zeta(C)=\log \big\{\sup_{\theta\in V_C}\E (\tilde{f}_{\theta}/f_{\theta_\star}) \big\}$, 
\item\label{A63} $\E(\,|\log   f_{\theta_\star} \,|^{2k_\star})<\infty$ and $\limsup_{C\rightarrow\infty} \zeta(C) (\log C)^{-1}<0$  with
$ \zeta(C)= \log \big\{\sup_{\theta\in V_C}\E( \tilde{f}_{\theta})\big\} $.
\end{enumerate}
\end{assumption}

\begin{remark}\label{rem:bound}
If $\Theta$ is a bounded set then $V_C\cap \Theta=\{\theta\in\Theta:\,\|\theta-\theta_\star\|\geq C\}$ is empty for sufficiently large $C$ and Assumption \ref{new2}.\ref{A62} is satisfied. Since Assumption \ref{new2}.\ref{A62} does not depend on $k_\star$ it follows that  if $\Theta$ is bounded then Assumption \ref{new2} holds  for every  $ k_\star\in\{1/2\}\cup \mathbb{N}$.
\end{remark}

\begin{remark}
Assumption \ref{new2} implies the existence of a set $A_\star\in\mathcal{B}(\R^d)$ such that the first part of Assumption \ref{new} holds.
\end{remark}

Assumptions \ref{new}-\ref{new2} are non-standard but are reasonable, as illustrated with the next result.
\begin{proposition}\label{prop:assumption}
Let $\setY=\R$ and $\eta(\dd y)$ be the Lebesgue measure on $\R$.
\begin{enumerate}
\item Let $\Theta=\R\times[\underline{\sigma}^2,\infty)$ for some   $\underline{\sigma}^2\in(0,\infty)$  and, for every $\theta=(\mu,\sigma^2)\in\Theta$, let $f_\theta(y)\eta(\dd y)=\mathcal{N}_1(\mu,\sigma^2)$. Then, Assumptions \ref{new}-\ref{new1} hold if $\E(Y_1^4)<\infty$. If in addition we have $\E(e^{c|Y_1|})<\infty$ for some $c>0$ then Assumption \ref{new2} holds for all $k_\star\in\mathbb{N}$.

\item Let $\Theta=\R\times[\underline{c},\infty)^2$ for some   $\underline{c}\in(0,\infty)$  and, for every $\theta=(\mu,\sigma^2,\nu)\in\Theta$, let $f_\theta(y)\eta(\dd y)=t_{1,\nu}(\mu,\sigma^2)$. Then, Assumptions \ref{new}-\ref{new1} hold if  $\E\{\log(1+c Y_1^2)^2\}<\infty$    for all $c\in(0,\infty)$. If in addition we have $\E(|Y_1|^c)<\infty$ for some $c>0$ then Assumption \ref{new2} holds for all $k_\star\in\mathbb{N}$.
\end{enumerate}

\end{proposition}

\subsection{Main result}\label{sub:theory_main}

The following theorem provides  conditions on the sequences $(h_t)_{t\geq 1}$ and   $(t_p)_{p\geq 0}$ which guarantee that, under Assumptions \ref{m_star}-\ref{new2} on $\{f_\theta,\,\theta\in\Theta\}$, we have  $\tilde{\pi}_t\Rightarrow\delta_{\{\theta_\star\}}$ in $\P$-probability.

\begin{theorem}\label{t-thm:online}
Assume Assumptions \ref{m_star}-\ref{new2},  let $ k_\star\in\{1/2\}\cup \mathbb{N}  $ and $\zeta(C)$, $C\in(0,\infty)$, be as in Assumption \ref{new2}, and let  $(h_t)_{t\geq 1}$ and  $(t_p)_{p\geq 0}$  be such that
\begin{enumerate}
 \item\label{h-thm1.1}  $\log(h_{t_{p-1}})/(t_p-t_{p-1})\rightarrow 0$ and $(t_p-t_{p-1}) \sum_{s=t_{p-1}+1}^{t_p-1}h_s^2 \rightarrow 0$,
\item\label{h-thm1.2} $h_{t_p}<h_{t_{p-1}}$ for all $p\geq 0$ and $\liminf_{p\rightarrow\infty} (h_{t_p}t_p^{\alpha}+t_p^{-\alpha}/h_{t_p})>0$ for some $\alpha\in(0,\infty)$,
\item\label{h-thm1.3}  $\limsup_{p\rightarrow+\infty}(t_{p+1}-t_{p})/(t_{p}-t_{p-1})<\infty$,
\item\label{h-thm1.4}    $
 \big|\zeta(h_{t_{p}}^{-\beta_\star})\big|^{-2k_\star}\sum_{i=1}^p (t_i-t_{i-1})^{-k_\star\ind_{\mathbb{N}}(k_\star)}\rightarrow 0$ for some constant   $\beta_\star\in(0,\infty)$.
\end{enumerate}
Then,  $\tilde{\pi}_{t}\Rightarrow\delta_{\{\theta_\star\}}$ in $\P$-probability.
\end{theorem}

\begin{remark}
With a similar reasoning as in Remark \ref{rem:bound}, if $\Theta$ is a bounded set  then for sufficiently large $C$ the set $V_C\cap \Theta$ is empty and we have $|\zeta(C)|=\infty$. Hence, Condition \ref{h-thm1.4} of Theorem \ref{t-thm:online} always holds when the set $\Theta$ is   bounded.
\end{remark}

\subsection{Application of Theorem \ref{t-thm:online}}\label{sub_sec}

The following proposition can be used to explicitly define a sequence $(t_p)_{p\geq 0}$ that verifies the conditions of Theorem \ref{t-thm:online} when $h_t=t^{-\alpha}$ for some $\alpha>0$.

\begin{proposition}\label{t-prop:h_t}

For some constants $C\in[1,\infty)$, $\alpha\in(0,\infty)$,  $\varrho_\alpha\in(0,\alpha\wedge 1)$, $c\in(0,1)$ and $t_0\in\mathbb{N}$  let $h_t=t^{-\alpha}$ for all $t\geq 1$ and let $(t_p)_{p\geq 0}$ be such that, for  all $p\geq 1$,
\begin{align}\label{t-eq:tp_formula}
 t_p= t_{p-1}+\lceil C_{p-1} \log(t_{p-1})\vee C\rceil,\quad    p\geq 1
\end{align}
for some $C_{p-1}\in [c t^{\varrho_\alpha}_{p-1},  t_{p-1}^{\varrho_\alpha}/c]$. Then, the sequences $(h_t)_{t\geq 1}$ and $(t_p)_{p\geq 0}$ verify Conditions \ref{h-thm1.1}-\ref{h-thm1.3} of Theorem \ref{t-thm:online}. Moreover,   these two sequences  also verify  Condition \ref{h-thm1.4} of  Theorem    \ref{t-thm:online} if   Assumption \ref{new2} holds  for a $k_\star>(1+\varrho_\alpha)/\varrho_\alpha$.
\end{proposition}

The second part of the proposition notably implies that  if Assumption \ref{new2} holds for all $k_\star\in\mathbb{N}$, as is the case in the examples of Propositions \ref{prop:assumption} or when $\Theta$ is a bounded set, see Remark \ref{rem:bound}, then for every $\alpha\in(0,\infty)$ the learning rate $h_t=t^{-\alpha}$ and the sequence $(t_p)_{p\geq 0}$ defined in \eqref{t-eq:tp_formula} verify  Conditions \ref{h-thm1.1}-\ref{h-thm1.4} of Theorem \ref{t-thm:online}. 

The conclusions of Theorem \ref{t-thm:online} and of Proposition  \ref{t-prop:h_t}  are summarized in the following corollary.

\begin{corollary}\label{cor:intro}
Let  $\alpha\in(0,\infty)$,  $h_t=t^{-\alpha}$ for all $t\geq 1$ and $(t_p)_{p\geq 0}$ be as in Proposition \ref{t-prop:h_t} for some $\varrho_\alpha\in(0,\alpha\wedge 1)$. Then, under Assumptions  \ref{m_star}--\ref{new1}, and provided that either $\Theta$ is a bounded set or that Assumption \ref{new2}  holds for some  $k_\star>(1+\varrho_\alpha)/\varrho_\alpha$, we have  $\tilde{\pi}_{t}\Rightarrow\delta_{\{\theta_\star\}}$ in $\P$-probability. 
\end{corollary}

\section{Numerical experiments\label{sec:num}}

\subsection{Implementation of G-PFSO and reference measure $\eta(\dd y)$}

Throughout this section we  let $h_t=t^{-\alpha}$ for some $\alpha\in\{0.3,0.5,0.8\}$  and, following the result of Corollary \ref{cor:intro}, we let  $(t_p)_{p\geq 0}$ be defined by
\begin{align}\label{eq:tp}
t_p=t_{p-1}+\lceil A t_{t-1}^{\varrho} \log(t_{p-1})\vee B\rceil,\quad p\geq 1
\end{align}
 with $A=B=1$, $t_0=5$ and $\varrho=0.1$. 

Since our convergence result for $\tilde{\pi}_t$ imposes a strong constraint on how often Student's $t$-distributions can be present in the sequence $(\widetilde{M}_t)_{t\geq 1}$  it seems  judicious to assess  the ability of G-PFSO to  reach a small neighbourhood of $\theta_\star$ without relying on these  fat tail distributions. For this reason, unless otherwise mentioned,  we let $\nu=50$  so that  each $t_{d,\nu}(\theta',h_t^2\Sigma)$ distribution appearing in \eqref{eq:mu_def} is very close to the $\mathcal{N}_d(\theta',h_t^2\Sigma)$ distribution in the sense of the Kullback-Leibler divergence  \citep{villa2018objective}. However, as illustrated in Section \ref{sub:Multimodal}, increasing the tails and the frequency of the Student's $t$-distributions can   improve the performance of G-PFSO.

All the results presented in this section are obtained with  $\Sigma=\csigma I_d$  for some  constant $\csigma\in(0,\infty)$  and by letting, in Algorithm \ref{algo:pi_tilde_online}, $\mathcal{R}(\cdot)$ be the SSP resampling algorithm   \citep{gerber2019negative} and, somewhat arbitrarily,   $\cess=0.7$.  Finally, the examples considered below are all such that  $\eta(\dd y)$ is the Lebesgue measure on $\R^k$ for some $k\in\mathbb{N}$.

\subsection{A censored quantile regression model\label{sub:CQR_5}}

The main objective of this  example is to study the convergence rate, as $t\rightarrow\infty$, of the G-PFSO estimator $\bar{\theta}_t^N$ in a non-trivial statistical model. To this aim  we consider a censored quantile regression  model with only $d=5$ parameters  that  will be learnt by processing sequentially a set of  $T=10^7$ i.i.d.\ observations. The model and the simulation set-up are precisely described in Appendix \ref{sec:additional_info}, and  we let $\theta_\star^{(\tau)}$ be the target parameter value when the censored quantile regression model is used to estimate the conditional quantile of order $\tau\in(0,1)$ of the response variable. Below, results are presented for $\tau=0.5$ and for $\tau=0.99$. 

In Figure \ref{fig:cqr_1} we summarize the estimation error obtained when  $\theta_\star^{(\tau)}$  is estimated using Adagrad, an adaptive stochastic gradient algorithm introduced by \citet{duchi2011adaptive},   randomly initialized far  from the target parameter value. The results presented in this figure suggest that  for $\tau\in\{0.5,0.99\}$   the corresponding objective function  is uni-modal, at least in a large neighbourhood of $\theta_\star^{(\tau)}$. Therefore, we can study the convergence behaviour of $\bar{\theta}_t^N$ without the concern of being  trapped in a local optimum.

\begin{figure}
   \centering
   \begin{subfigure}[b]{0.31\textwidth}
        \centering
     \includegraphics[scale=0.23]{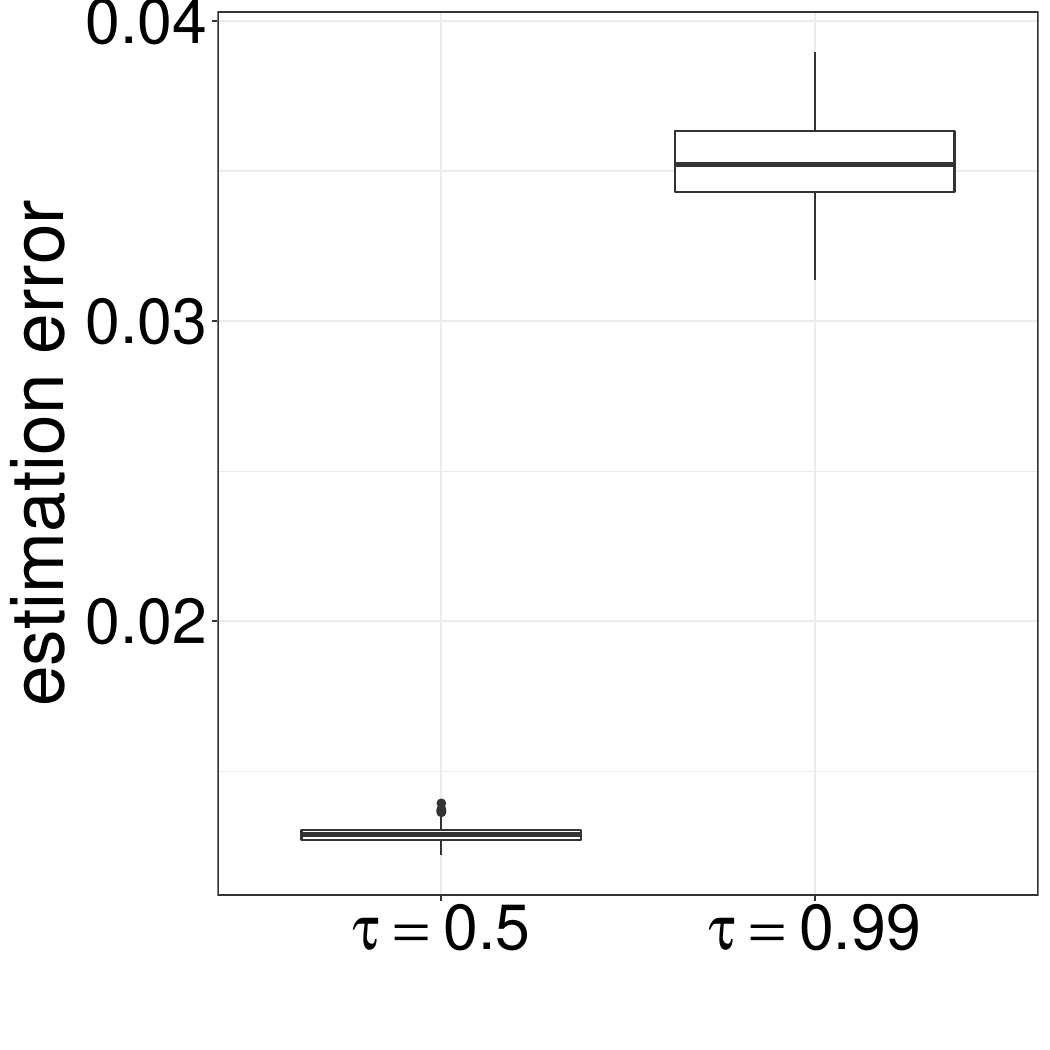}   
     \caption{\label{fig:cqr_1}}
     \end{subfigure}
     ~
     \begin{subfigure}[b]{0.31\textwidth}
        \centering
     \includegraphics[scale=0.23]{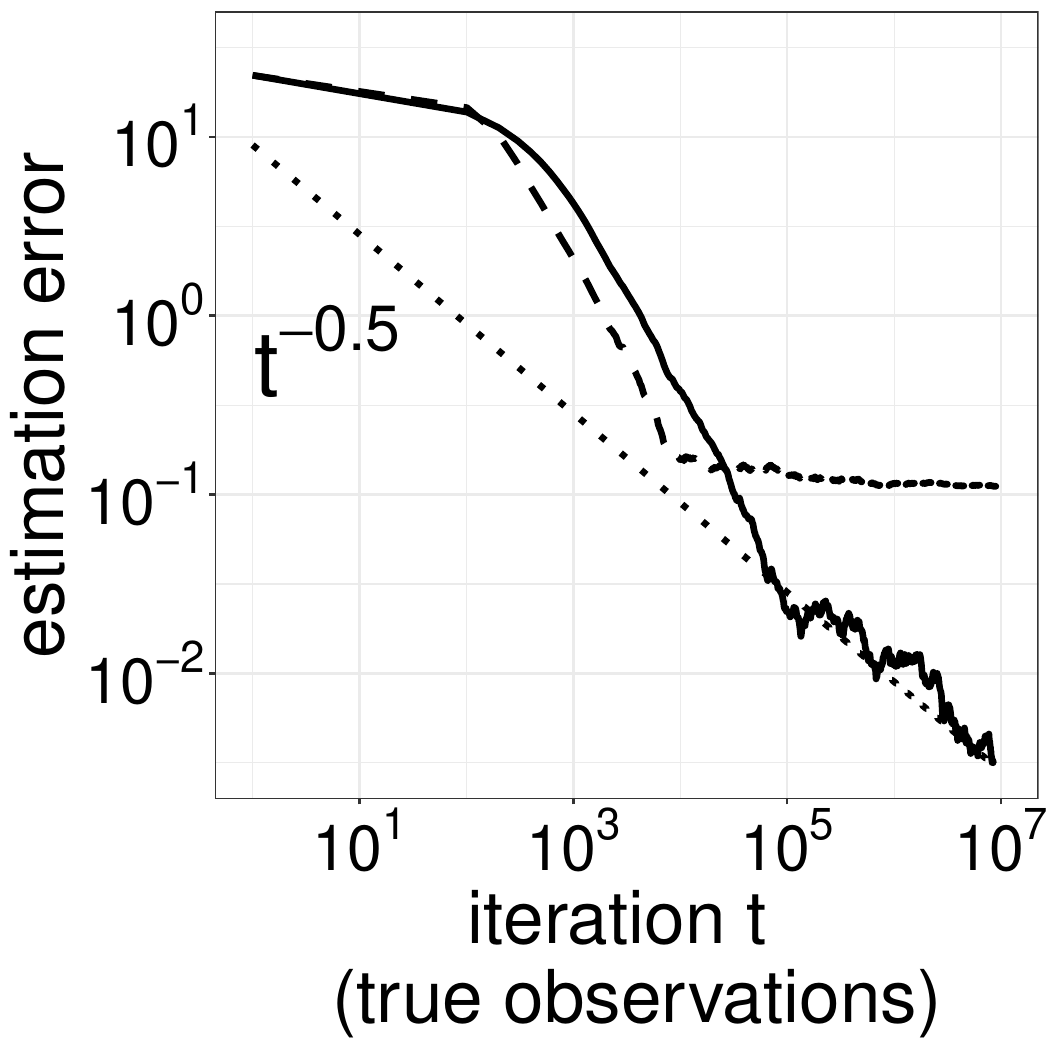} 
     \caption{\label{fig:cqr_2}}   
     \end{subfigure}
     \begin{subfigure}[b]{0.31\textwidth}
        \centering
     \includegraphics[scale=0.23]{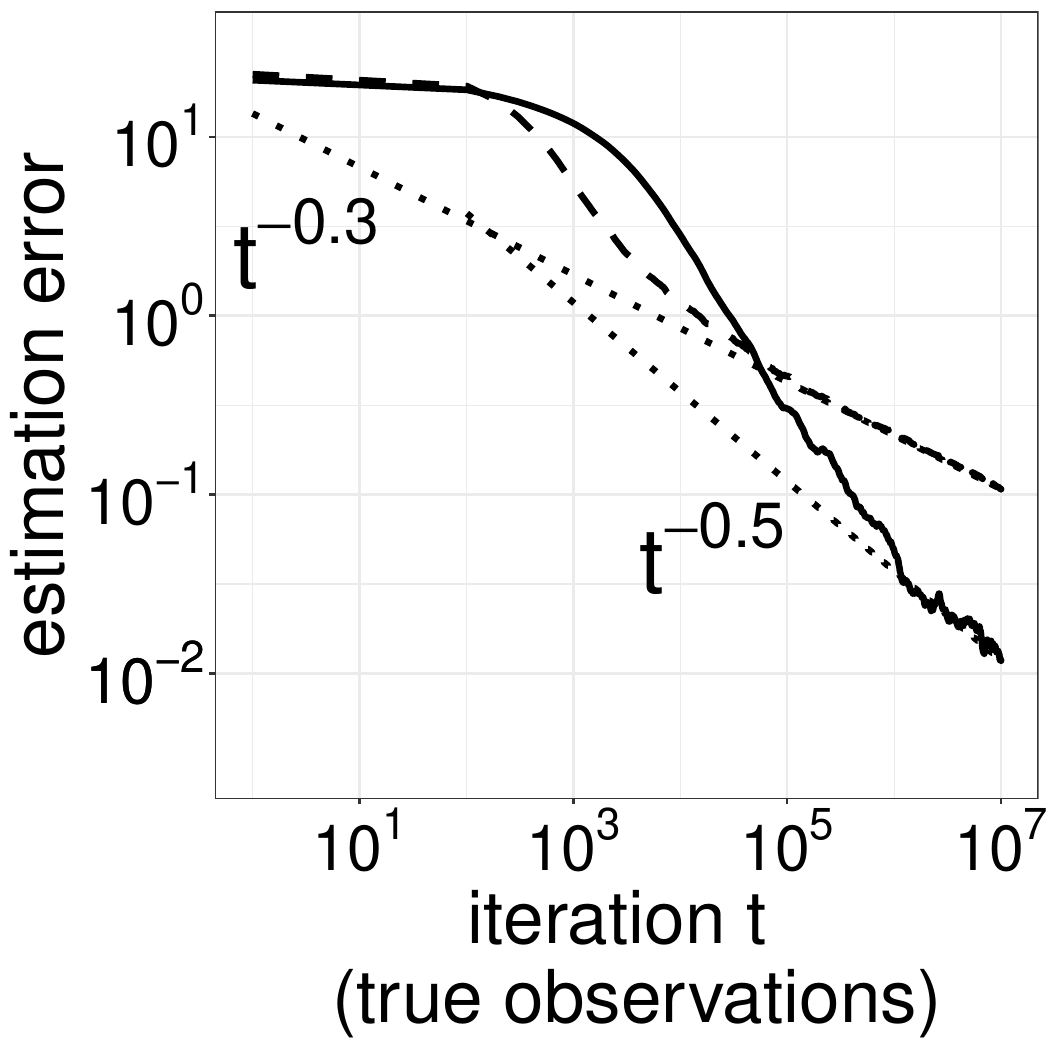}   
     \caption{\label{fig:cqr_3}}
     \end{subfigure}
  \caption{Example of Section \ref{sub:CQR_5}. Plot (\subref{fig:cqr_1}) summarizes the values of $\|\theta_{\mathrm{ada},T}-\theta_\star^{(\tau)}\|_\infty$ obtained with Adagrad  for $10^3$ starting values randomly sampled from the $\mathcal{N}_d(\theta_\star^{(0.5)}+10, 2I_d)$  distribution. In Plot  (\subref{fig:cqr_2})  the solid   line is for  $\bar{\theta}_t^N$   and  the dashed line for $\bar{\theta}^N_{\mathrm{Jit},t}$. In Plot (\subref{fig:cqr_3}) the solid line is for $\alpha=0.5$ and the dashed line for $\alpha=0.3$. In Plots (\subref{fig:cqr_2})-(\subref{fig:cqr_3}) the dotted lines show the function $f(t)=c t^{-\gamma}$ for some $c>0$ and with $\gamma$ as given in the plots.\label{fig:CQR_5}}
\end{figure}

For this example  G-PFSO is implemented with  $N=1\,000$,  $\tilde{\pi}_0(\dd\theta)=\mathcal{N}_d(\theta^{(0.5)}_{\star}+10,2 I_d)$  and $\csigma=1$.  In Figures \ref{fig:cqr_2}-\ref{fig:cqr_3} we report the evolution as $t$ increases of the average value of
 $\|\bar{\theta}_t^N-\theta_\star^{(\tau)}\|$ obtained in 20  runs of Algorithm \ref{algo:pi_tilde_online}, where $(\tau,\alpha)=(0.5,0.5)$ in Figure \ref{fig:cqr_2} and where $(\tau,\alpha)\in\{(0.99,0.3), (0.99,0.5)\}$ in Figure \ref{fig:cqr_3}. The results in these two figures suggest that  for $h_t=t^{-1/2}$   the estimator $\bar{\theta}_t^N$ converges at the optimal $t^{-1/2}$ rate to the target parameter value both when $\tau=0.5$ and when $\tau=0.99$. However, for the very slow learning rate $h_t=t^{-0.3}$,  and for $\tau=0.99$, the G-PFSO estimator appears to converge  towards $\theta_\star^{(\tau)}$ at the slow $t^{-0.3}$ rate. Consequently, these results  indicate that if $h_t\rightarrow 0$ too slowly then the convergence rate of  $\bar{\theta}_t^N$ may be sub-optimal.   We finally  use   the jittering estimate $\pi_{\mathrm{Jit},t}^N$ of $\pi_t$ to estimate $\theta_\star^{(\tau)}$ when $\tau=0.5$, using the jittering kernel $
\mathsf{M}_N(\theta',\dd\theta)=(1-N^{-1/2})\delta_{\{\theta'\}}+N^{-1/2}\mathcal{N}_d(\theta',   I_d)$.  Since this kernel   is homogenous  it is clear that for a fixed number  of particles, $N$, the estimator $\theta^N_{\mathrm{Jit},t}=\int_{\Theta}\theta \pi_{\mathrm{Jitt},t}^N(\dd\theta)$ cannot converge to $\theta_\star^{(\tau)}$ as $t\rightarrow\infty$. To study if  Polyak–Ruppert averaging  can resolve this issue we let  $\bar{\theta}^N_{\mathrm{Jit},t}=t^{-1}\sum_{s=1}^t \theta^N_{\mathrm{Jit},s}$ and report in Figure \ref{fig:cqr_2} the evolution as $t$ increases of $\|\bar{\theta}^N_{\mathrm{Jit},t}-\theta^{(\tau)}_\star\|$, averaged over 20   runs of the algorithm. We remark    that after $T_N\approx 10^5$ observations the  average  estimation error $\|\bar{\theta}^N_{\mathrm{Jit},t}-\theta^{(\tau)}_\star\|$ stabilizes around 0.10. As discussed above,  for this example the objective function  has apparently no local maxima,   and therefore these simulation  results suggest that Polyak–Ruppert averaging does not prevent   the inference based on  $\pi_{\mathrm{Jit},t}^N$   to stop improving after some finite time.

\subsection{A toy  multi-modal example\label{sub:Multimodal}}

In this second example we consider a sequence $\big(Y_t=(Z_t,X_t)\big)_{t\geq 1}$  of i.i.d.\ random variables   taking value  in $\R\times \R^d$, with $d=20$. Then, inspired by  an example in \cite{Hunter2000}, we let $\Theta\subset\R^d$ and $\mu(\theta,x)=\sum_{i=1}^d\big(e^{-x_{i}\theta^2_i}+x_{i}\theta_{d-i+1}\big)$ for all $(\theta,x)\in\Theta\times \R^d$, and our goal is to estimate $\theta_\star=\argmin_{\theta\in\Theta}\E[|Z_1-\mu(\theta, X_1)|]$. To cast this estimation problem into the set-up of this paper, for every $(\theta,x)\in\Theta\times\R^d$ we let $f_\theta(\cdot\mid x)$ be the density  of the Laplace distribution  with  scale parameter $b>0$ and location parameter $\mu(\theta,x)$, so that   $\theta_\star= \argmax_{\theta\in\Theta}\E[\log f_\theta(Z_1\mid X_1)]$. We let $b=0.5$ and simulate $T=10^6$ i.i.d.\ observations    using $Z_1|X_1\sim\mathcal{N}_1\big(\mu(\theta_\star,X_1),4)$,  $X_1\sim \mathrm{Unif}([-1,1]^d)$ and $\theta_\star=(-1,\dots,-1)$. As in the previous subsection the observations are processed sequentially while, to  avoid numerical problems, we let  $\Theta$ be the open ball of size 20 around $\theta_\star$ w.r.t.\ the maximum norm. For this example we let $N=2\,000$,  $\tilde{\pi}_0(\dd\theta)=\mathrm{Unif}(\Theta)$ and $\alpha=0.5$.

We first use  the KS-PFSO estimator $\theta^N_{\mathrm{K},t}=\int_{\Theta}\theta\pi^N_{\mathrm{K},t}(\dd\theta)$ to estimate $\theta_\star$ where, using  equation  (10) in \citet{balakrishnan},  we let $\iota\approx 0.68$ in \eqref{eq:kpfso}. In Figure \ref{fig:multi_1} we summarize  the values of  $\|\theta^N_{\mathrm{K},T'}-\theta_\star\|_\infty$ obtained in 100  runs of the algorithm, with $T'=10^5$.  We observe  that the estimation error obtained with KS-PFSO is always larger than $1.66$, suggesting the existence of  some local optima to which this algorithm converges. The existence of local optima can also be observed in Figure \ref{fig:multi_2}  which shows the  evolution   as $t$ increases   of the 14-th component of $\tilde{\theta}_t^N$   obtained in a single run of G-PFSO with $\csigma=10$. Notice that the results presented in Figure \ref{fig:multi_1}  for KS-PFSO illustrate the fact that this algorithm is a local optimization method, for reasons explained in Section \ref{sub:lit_review}.

\begin{figure}
\centering
 \centering
   \begin{subfigure}[b]{0.31\textwidth}
        \centering
     \includegraphics[scale=0.23]{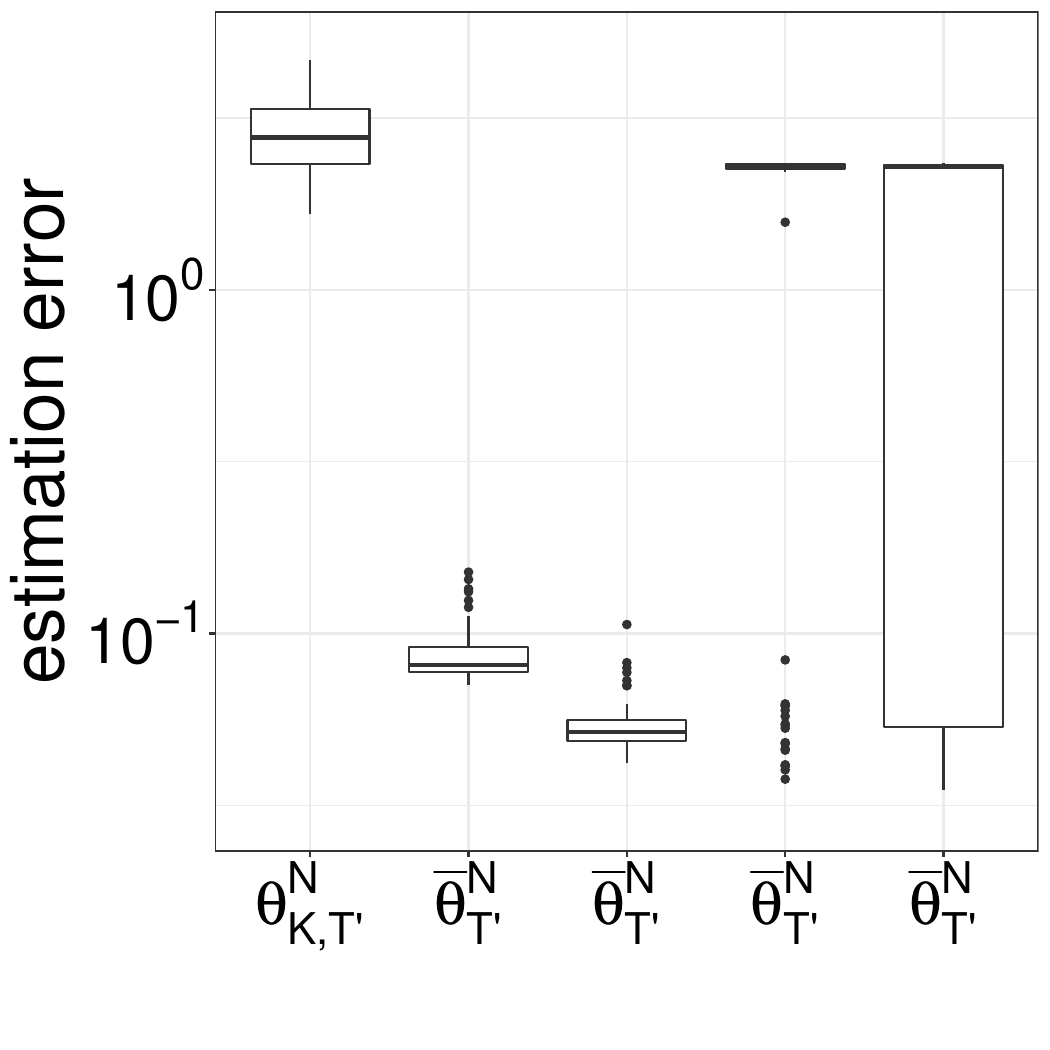} 
     \caption{\label{fig:multi_1}}
     \end{subfigure}
     ~
      \centering
   \begin{subfigure}[b]{0.31\textwidth}
        \centering
     \includegraphics[scale=0.23]{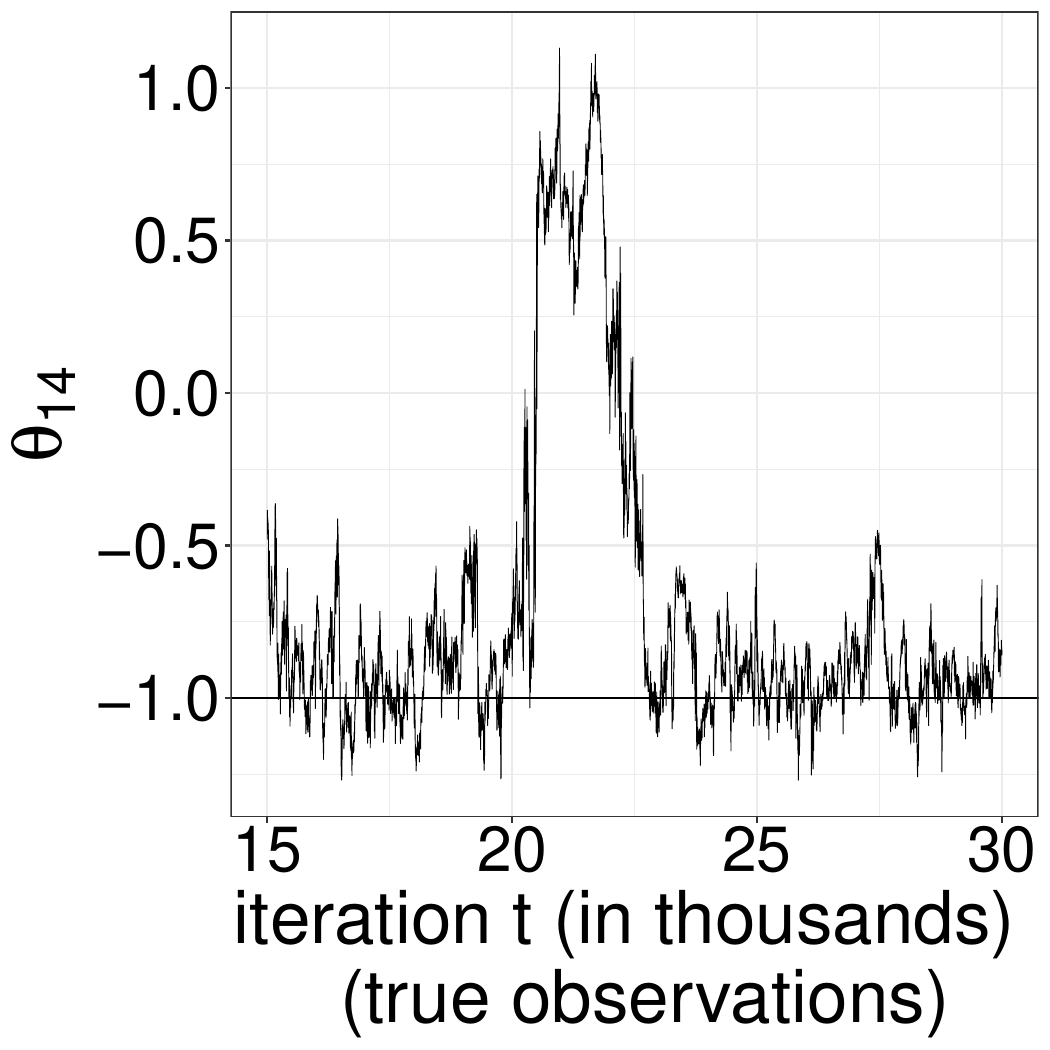}  
     \caption{\label{fig:multi_2}}
     \end{subfigure}
     ~
      \centering
   \begin{subfigure}[b]{0.31\textwidth}
        \centering
      \includegraphics[scale=0.23]{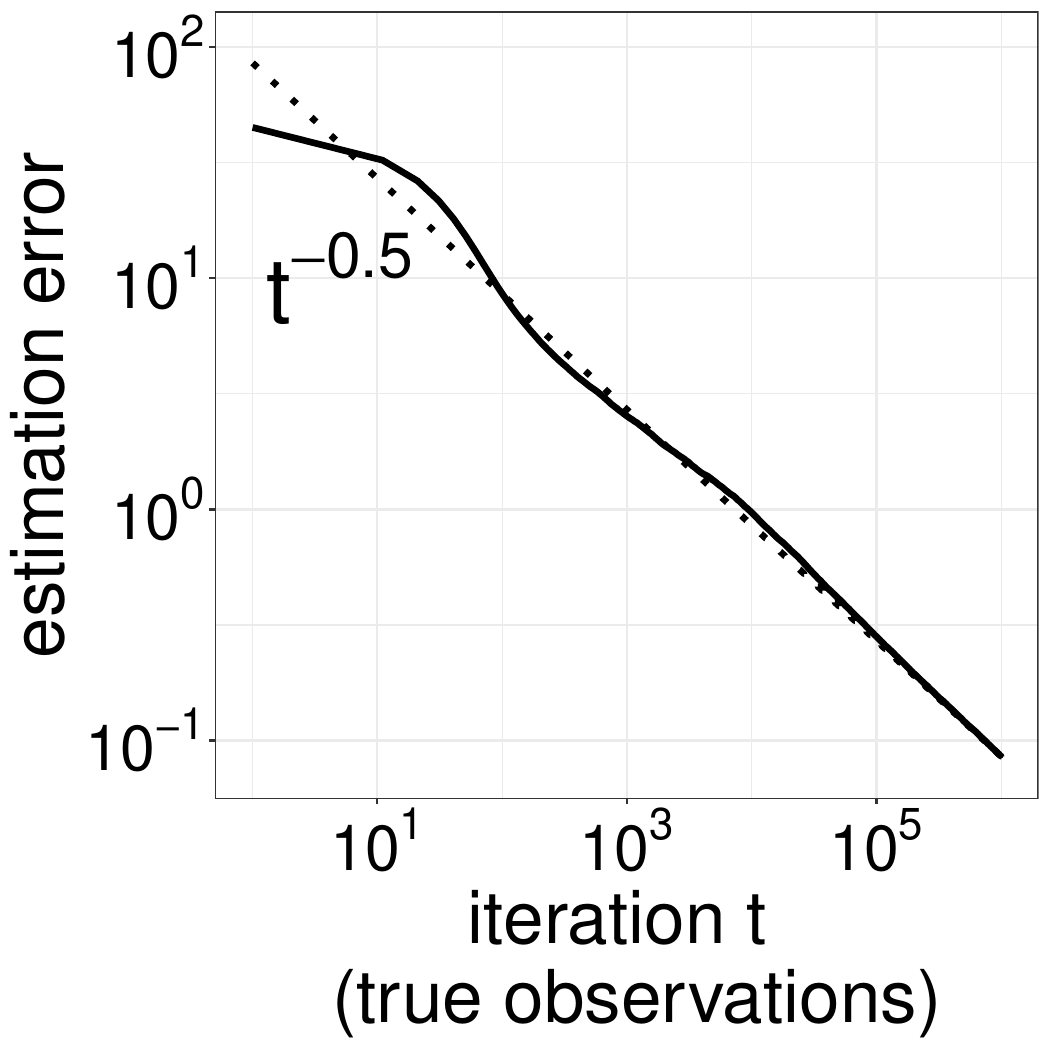}  
     \caption{\label{fig:multi_3}}
     \end{subfigure}
\caption{Example of Section \ref{sub:Multimodal}. In Plot (\subref{fig:multi_1}), for the estimator   $\bar{\theta}_{T'}^N$  the first boxplot is for $(\nu,\csigma)=(50,10)$, the second boxplot  is for $(\nu,\csigma)=(50,3)$, the third boxplot for $(\nu,\csigma)=(50,1)$ and the last  boxplot is  for $(\nu,\csigma)=(1.5,1)$.
In Plot (\subref{fig:multi_2}) the horizontal line represents the true parameter value and in Plot (\subref{fig:multi_3}) the dotted line is as in Figure \ref{fig:CQR_5}.\label{figs:multi}}
\end{figure}

The second boxplot in Figure \ref{fig:multi_1} summarizes the values of $\|\bar{\theta}_{T'}^N-\theta_\star\|_\infty$ obtained in 100  runs of  G-PFSO with $\csigma=10$ and where, as for KS-PFSO, $T'=10^5$.  For each run of the algorithm   the estimation error is smaller than $0.16$, showing that $\bar{\theta}_{T'}^N$  successfully finds the global optimum of the  function $\theta\mapsto\E[\log f_\theta(Z_1\mid X_1)]$ with very high probability. To assess the sensitivity of $\bar{\theta}_{T'}^N$ to the parameter $\csigma$   in the 3rd and 4th  boxplot  of Figure \ref{fig:multi_1} we repeat the  experiment with $\csigma=3$ and with $\csigma=1$. We observe that decreasing $\csigma$ from 10 to 3 improves the performance of the estimator $\bar{\theta}_{T'}^N$, for the following reason. On the one hand, decreasing the value of $\csigma$ reduces for all $t\geq 1$  the variance of the distributions $\{\widetilde{M}_t(\hat{\theta}^n_{t-1},\dd\theta)\}_{n=1}^N$ used  in Algorithm \ref{algo:pi_tilde_online} to generate $\{\theta_t^n\}_{n=1}^N$, which    enables $\tilde{\pi}_t^N$ to be   more concentrated around $\theta_\star$ and consequently  to reduce the estimation error. On the other hand, in this example  for $\csigma=3$ the variance of these Markov kernels is large enough to ensure that a small neighbourhood of $\theta_\star$ is quickly reached by G-PFSO. We however remark in Figure \ref{fig:multi_1} that this is no longer the case when $\csigma=1$   since,  for this value of $\csigma$, the estimation error of  $\bar{\theta}_{T'}^N$ is frequently large and similar to that obtained with  KS-PFSO. For a given choice of $\csigma$ the exploration of $\Theta$ can be improved by reducing $\nu$, the number of degrees of freedom of the Student's $t$-distributions used in \eqref{eq:mu_def}. This point is illustrated with the last boxplot in Figure \ref{fig:multi_1}, which shows that for  $\csigma=1$ decreasing $\nu$ from 50 to 1.5 increases the probability of $\bar{\theta}_{T'}^N$ having a small  estimation error. Notably, in this figure reducing $\nu$ from 50 to 1.5 doubles, from   18 to 36, the number of runs of G-PFSO for which  $\|\bar{\theta}_{T'}^N-\theta_\star\|_\infty<0.2$.

Finally, in Figure \ref{fig:multi_3} we show  the evolution of $\|\bar{\theta}_{t}^N-\theta_\star\|$ while processing the available $T=10^6$  data points, averaged over 10   runs of Algorithm \ref{algo:pi_tilde_online} with $\csigma=10$.  The results reported in this plot suggest that, as in the previous example, for $h_t=t^{-1/2}$  the estimator  $\bar{\theta}^N_t$  converges to $\theta_\star$ at the optimal $t^{-1/2}$ rate.

\subsection{A smooth adaptive Gaussian mixture model \label{sub:mixture}}

Let  $\big(Y_t=(Z_t,X_t)\big)_{t\geq 1}$ be a sequence of random variables taking values in $\R\times \R^{d_x}$  for some $d_x\geq 1$. Then, the   smooth adaptive Gaussian mixture  model   \citep{villani2009regression} with $K\geq 2$ components  assumes that, for every $t\geq 1$ and   with $d=d_x(3K-1)$, the conditional distribution of $Z_t$ given  $X_t=x$ belongs to the set $\{f_{\theta}(\cdot\mid x),\,\theta\in\Theta\subseteq \R^d\}$  where
\begin{equation}\label{eq:mixture}
\begin{split}
f_{\theta}(z\mid x)& =\sum_{k=1}^{K }w_k(x,\beta^\mathrm{w})\varphi_1\Big\{z;  x^\top\beta^\mu_{(k)},\exp(-x^\top\beta^{\sigma}_{(k)})\Big\},\quad  (\theta,z)\in\Theta\times\R
\end{split}
\end{equation}
with $\varphi_1(\cdot;\mu,\sigma)$ the probability density function of the $\mathcal{N}_1(\mu,\sigma^2)$ distribution, $w_K(x,\beta^\mathrm{w})=1-\sum_{k=1}^{K-1}w_k(x,\beta^\mathrm{w})$ and 
\begin{align}\label{eq:mixture2}
w_k(x,\beta^\mathrm{w})=\frac{\exp(-x^\top\beta_{(k)}^\mathrm{w})}{1+\sum_{k'=1}^{K-1}\exp(-x^\top\beta_{(k')}^\mathrm{w})},\quad k=1,\dots,K-1.
\end{align}
In \eqref{eq:mixture}-\eqref{eq:mixture2} we have  $\beta^{i}_{(k)}\in\R^{d_x}$ for all $k$ and all $i\in\{\mu,\sigma, \mathrm{w}\}$, while $\beta^\mathrm{w}=(\beta^\mathrm{w}_{(1)},\dots,\beta^\mathrm{w}_{(K-1)})$ and, letting $\beta^i=(\beta^{i}_{(1)},\dots,\beta^{i}_{(K)})$ for every $i\in\{\mu,\sigma\}$,   $\theta=(\beta^\mathrm{w},\beta^\mu,\beta^\sigma)$.

For this example we let  $K=2$ and $d_x=4$, resulting in a model with $d=20$ parameters that  will be learnt by processing sequentially a set of  $T=2\times 10^6$ i.i.d.\ observations. The simulation set-up is described in Appendix \ref{sec:additional_info}. Without loss of generality we let $\Theta=\{ \theta \in\R^{d}:\, \theta_1\geq 0\}$ and, since our chosen value for $\theta_\star$ is such that $\theta_{\star,1}\neq 0$, it follows that  $\theta_\star$ is the unique global maximizer of the  function $\Theta\ni\theta\mapsto \E[\log f_{\theta}(Z_1\mid X_1)]$. Finally, we  let $N=5\,000$, $\csigma=1$, $\alpha=0.5$ and  $\tilde{\pi}_0(\dd\theta)=\mathrm{Exp}(1)\otimes\mathcal{N}_{d-1}(0,I_{d-1})$.

\begin{figure}
\centering
\begin{subfigure}[b]{0.31\textwidth}
        \centering
       \includegraphics[scale=0.23]{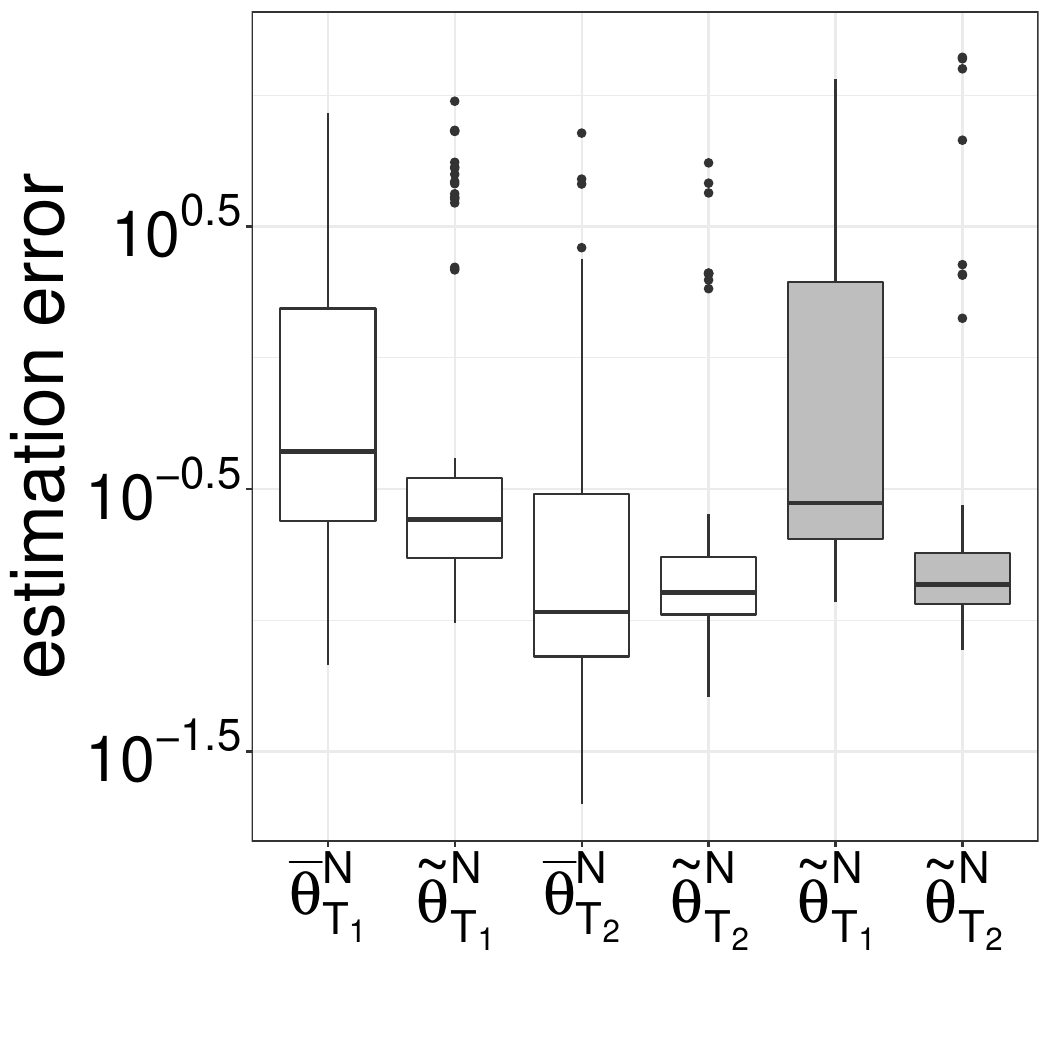}
       \caption{\label{fig:mix_1}}
    \end{subfigure}
     ~
     \begin{subfigure}[b]{0.31\textwidth}
        \centering
    \includegraphics[scale=0.23]{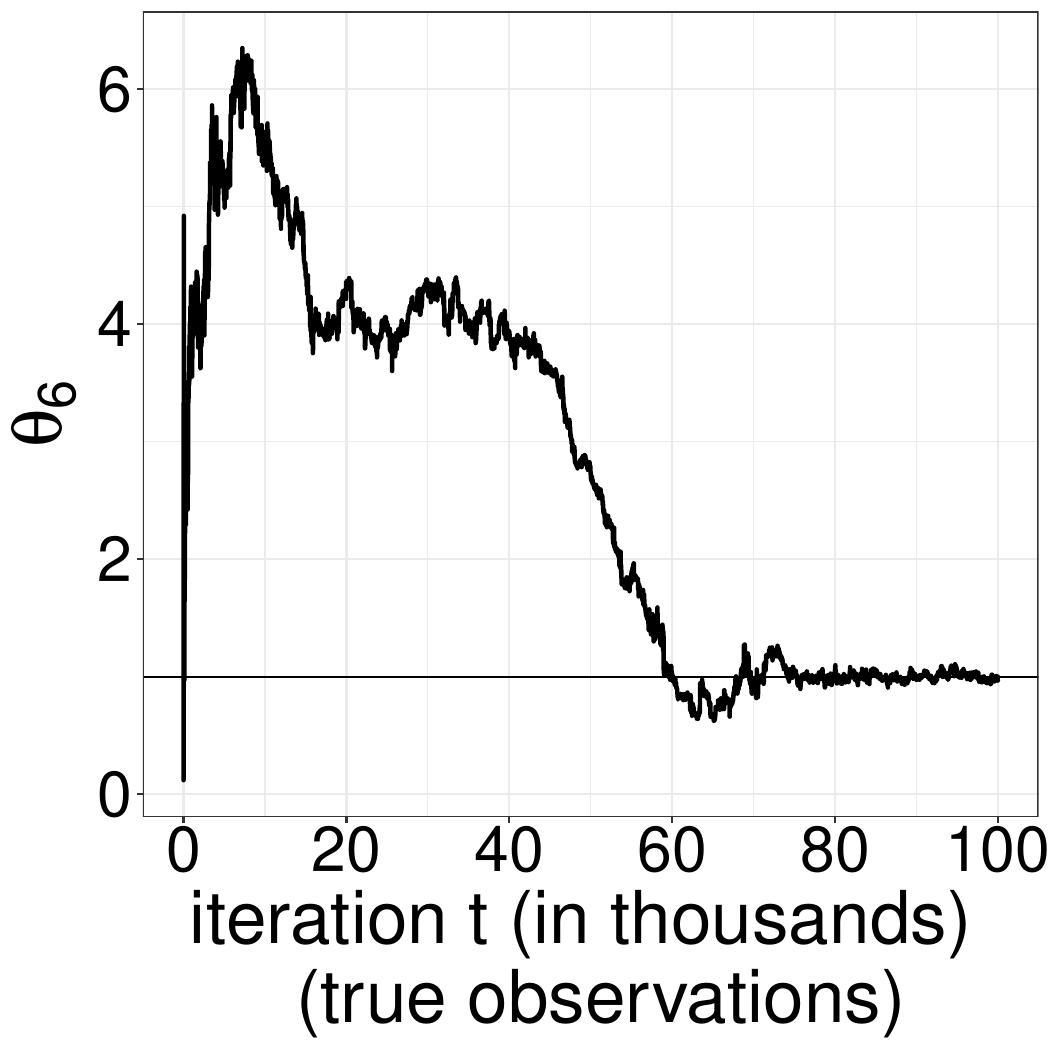}
     \caption{\label{fig:mix_2}}
     \end{subfigure}
     ~
     \begin{subfigure}[b]{0.31\textwidth}
        \centering
    \includegraphics[scale=0.23]{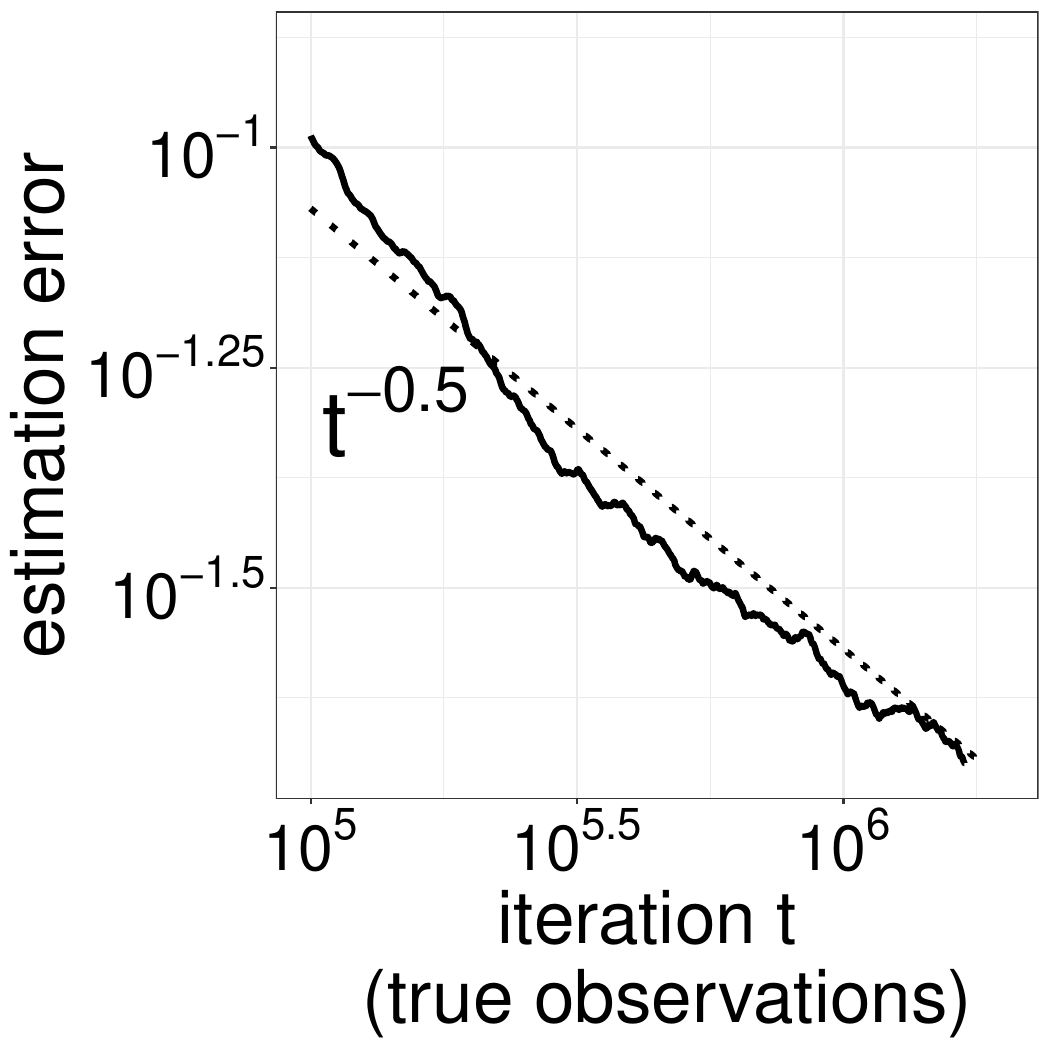}
      \caption{\label{fig:mix_3}}
    \end{subfigure}
\caption{Example of Section \ref{sub:mixture}. In Plot (\subref{fig:mix_1}), $T_1=20\,000$ and $T_2=10^5$,  the white boxplots are for $\nu=50$ and the grey boxplots for $\nu=2$.  In Plot (\subref{fig:mix_2}) the horizontal line represents the true parameter value and in  Plot (\subref{fig:mix_3})  the dotted line  is as in Figures \ref{fig:CQR_5}-\ref{figs:multi}. \label{fig:mixture_boxplot_K2}}
\end{figure}

We  report in Figure \ref{fig:mix_1} a summary of the values obtained for $\|\tilde{\theta}_{T'}^N-\theta_\star\|_\infty$ and for  $\|\bar{\theta}_{T'}^N-\theta_\star\|_\infty$ in  100   runs of G-PFSO, with $T'\in\{20\,000,10^5\}$.  Although the initial distribution $\tilde{\pi}_0$ has most of its mass on a ball of size 1 around $\theta_\star$, we remark that  the estimation error can in some cases be large for the two G-PFSO estimators.  We indeed observe in our experiments that,  for this   example, the value of $\|\tilde{\theta}_{t}^N-\theta_\star\|_\infty$ often  increases  sharply in the first few iterations of Algorithm \ref{algo:pi_tilde_online}  and may remain large for a very long time period. This phenomenon is illustrated in Figure \ref{fig:mix_2} where, for one of the 100 runs of G-PFSO, we  present  the evolution as $t$ increases of     $\tilde{\theta}_{t,6}^N$, the   6-th component of $\tilde{\theta}_t^N$. A close look at the results presented in Figure \ref{fig:mix_1} reveals that  the estimation error  $\|\tilde{\theta}_{10^5}^N-\theta_\star\|_\infty$ is smaller than 0.26 for 91 runs of Algorithm \ref{algo:pi_tilde_online} and larger than 1.8 for the 9 remaining runs, suggesting the existence of some local optima in which G-PFSO remains trapped after $10^5$ observations. The existence of some local optima seems confirmed by the evolution of $\tilde{\theta}_{t,6}^N$ reported in Figure \ref{fig:mix_2}, a figure which  also illustrates the ability of G-PFSO to escape from a local optimum   after a large number of iterations.   This ability of G-PFSO to escape  from a local optimum explains why  in the  18 runs of Algorithm \ref{algo:pi_tilde_online} for which we have  $\|\tilde{\theta}_{T'}^N-\theta_\star\|_\infty>2$ after having processed $T'=20\,000$ observations--the estimation error being smaller than 0.42 in the 82 other runs of G-PFSO--increasing the number of data points to $T'=10^5$ allows to have $\|\tilde{\theta}_{T'}^N-\theta_\star\|_\infty<0.23$ for 9 of them.
Next, and importantly, the results in Figure \ref{fig:mix_1}  show    that even for a large sample size $t $ the estimator $\tilde{\theta}_{t}^N$ can  outperform  $\bar{\theta}_{t}^N$ when many iterations are needed by G-PFSO to reach a small neighbourhood of $\theta_\star$, as it is often the case in this example. Indeed,  recalling that  $\bar{\theta}_t^N=t^{-1}\{(t-1)\bar{\theta}_{t-1}^N+ \tilde{\theta}_t^N\}$, it follows that if  $\|\tilde{\theta}_{t}^N-\theta_\star\|_\infty$  starts decreasing only after some time $\underline{t}\gg 1$ then, because  $\tilde{\theta}_t^N$ is multiplied by a factor $1/t$ in the definition of $\bar{\theta}_t^N$,  the number of iterations needed for $\|\bar{\theta}_{t}^N-\theta_\star\|_\infty$ to reach a  value close to zero is much larger than  for $\|\tilde{\theta}_{t}^N-\theta_\star\|_\infty$. In practice  this problem can be avoided by defining $\bar{\theta}_t^N$ as $\bar{\theta}_t^N=(t-\underline{t})^{-1}\sum_{s=\underline{t}+1}^t \tilde{\theta}_t^N$ for $t>\underline{t}$, where $\underline{t}\in\mathbb{N}$ is the time needed for the sequence $(\tilde{\theta}_t^N)_{t\geq 1}$ to stabilize around a given point in the parameter space.  In the last two boxplots of  Figure \ref{fig:mix_1} we repeat the same experiment with $\nu=2$. Unlike what we observed in Section \ref{sub:Multimodal}, for this example it is not desirable to use, in \eqref{eq:mu_def}, Student's $t$-distributions  with an infinite variance.   In particular, the results in Figure \ref{fig:mix_1} suggest that reducing $\nu$ from 50 to 2  decreases the probability of reaching a small neighbourhood of $\theta_\star$ with   $T'=20\,000$  iterations and increases the probability of having  a  large estimation error with $T'\in\{20\,000,10^5\}$ iterations.  By improving the ability of the particle system to move across the parameter space, reducing $\nu$  enables G-PFSO to find  the global mode of the objective function earlier but also increases the probability that  the estimator $\tilde{\theta}_t^N$ is pushed far away  from $\theta_\star$ at time $t\in(t_p)_{p\geq 0}$. Consequently, depending on the estimation problem at hand,   decreasing $\nu$  may  improve--as  in Section \ref{sub:Multimodal}-- or deteriorate--as in this example--the performance of G-PFSO.

Finally,  we study the convergence behaviour of $ \|\bar{\theta}_t^N-\theta_\star\|$ as $t$ increases.  The results in Figure \ref{fig:mix_3}, which show the time    evolution    of $\|\bar{\theta}_t^N-\theta_\star\|$  averaged over 5   runs of Algorithm \ref{algo:pi_tilde_online}, suggest again that  $\bar{\theta}_t^N$ converges to $\theta_\star$ at the optimal $t^{-1/2}$ rate when $h_t=t^{-1/2}$.

\subsection{The bivariate  g-and-k distribution\label{sub:gandk}}
 
In this last example we use a real dataset to illustrate the use of G-PFSO for approximating the maximum likelihood estimator.  To this aim  we consider the problem of parameter inference  in the bivariate extension of the  g-and-k distribution  proposed by \cite{drovandi2011likelihood}. For  $\theta=(a_1, a_2,b_1, b_2,g_1, g_2,k_1, k_2,\rho)\in\Theta\subset \R^9$  we let $f_\theta$ be the probability density function of the  bivariate g-and-k distribution  with location parameters $(a_1,a_2)$, scale parameters  $(b_1,b_2)$, skewness parameters  $(g_1,g_2)$, kurtosis parameters $(k_1,k_2)$ and correlation parameter $\rho$. Following the results in \citet{prangle2017gk}, to ensure that $f_\theta$ is a well-defined probability density function  we let $\Theta=\R^2\times(0,\infty)^2\times \Theta_{\mathrm{gk}}\times(-1,1)$ where $$
\Theta_{\mathrm{gk}}=\big\{(g_1,g_2,k_1,k_2)\in \R^4:\, |g_i|<5.5,\, k_i>-0.045-0.01 g_i^2,\,\,i=1,2\big\}.
$$

As in \citet{drovandi2011likelihood}, we let  $\{\tilde{y}_{i,1}\}_{i=1}^n$ and $\{\tilde{y}_{i,2}\}_{i=1}^n$ be respectively the exchange rate daily log returns from GBP to AUD and  from GBP to EURO, multiplied by 100. We consider data from 4 January 2000 to 1 January 2021 inclusive which, after having removed the dates for which only one of the two exchange rates is available, results in a sample of size  $n=7\,254$.   Letting  $\tilde{y}_i=(\tilde{y}_{i,1},\tilde{y}_{i,2})$ for all $i$, the corresponding log-likelihood function is defined by  $l_n(\theta)=\sum_{i=1}^n\log f_\theta(\tilde{y}_i)$ and, following the discussion in Section \ref{sub:set-up}, below we compute the maximum likelihood estimator  $\hat{\theta}_{\mathrm{mle},n}=\argmax_{\theta\in\Theta}l_n(\theta)$ 
by running Algorithm \ref{algo:pi_tilde_online} on a set of $T$ pseudo-observations $\{y_t\}_{t=1}^{T}$  sampled i.i.d.\ from   the empirical distribution of the observations $\{\tilde{y}_i\}_{i=1}^n$. For each run of Algorithm \ref{algo:pi_tilde_online}   a new set of pseudo-observations $\{y_t\}_{t=1}^{T}$ is sampled and, for this example, we let $N=500$ and $\tilde{\pi}_0\in\mathcal{P}(\Theta)$ be such that if $\theta\sim \tilde{\pi}_0(\dd\theta_0)$ then, for $i=1,2$, $a_i\sim \mathcal{N}_1(0,1)$, $b_i\sim \mathrm{Exp}(1)$, $g_i\sim \mathrm{Unif}(-5.5,5.5)$ and  $k_i+0.045+0.01 g_i^2\sim  \mathrm{Exp}(1)$, where all the random variables are  independent of each other, with the exception of $k_1$ that depends on $g_1$ and of $k_2$ that depends on $g_2$.

We  start by considering  100   runs of Algorithm \ref{algo:pi_tilde_online} with   $T=50\,000$, $\csigma=10$ and with a fast learning rate $h_t=t^{-0.8}$, so that   $\alpha=0.8$. In a first step  we optimize $l_n(\theta)$ using a quasi-Newton algorithm  initialized at the  G-PFSO estimate of $\hat{\theta}_{\mathrm{mle},n}$ that gives the largest log-likelihood value, and the resulting parameter value   is treated in what follows as the true value of $\hat{\theta}_{\mathrm{mle},n}$. From Figure \ref{fig:gandk_1}  we observe that the g-and-k distribution  $f_{\hat{\theta}_{\mathrm{mle},n}}(y)\eta(\dd y)$ fits  the data well, at least as far as the marginal distributions are concerned.
\begin{figure}
\centering
\begin{subfigure}[b]{0.31\textwidth}
        \centering
       \includegraphics[scale=0.24,trim=0cm -1.5cm 0cm 0cm]{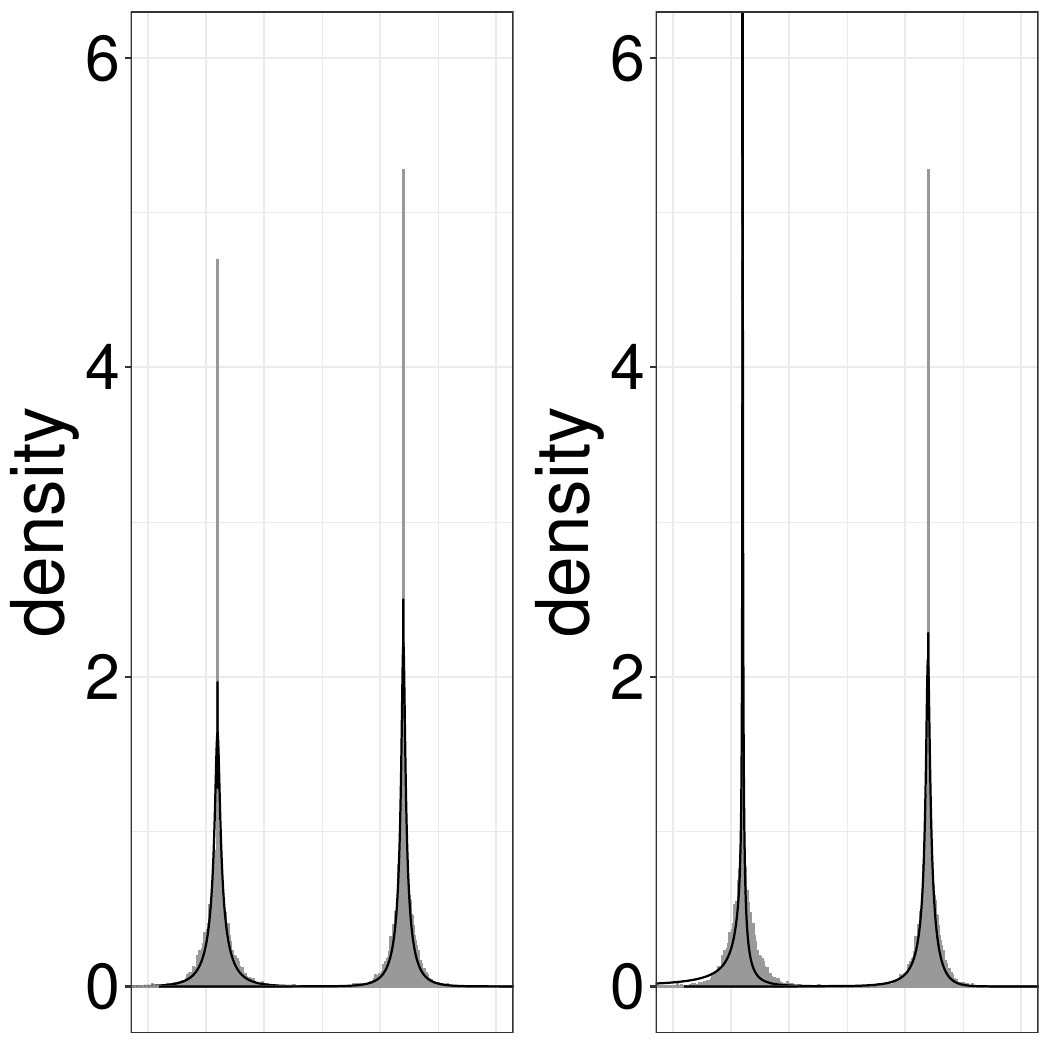}
       \caption{\label{fig:gandk_1}}
    \end{subfigure}
    ~
    \begin{subfigure}[b]{0.31\textwidth}
        \centering
       \includegraphics[scale=0.26]{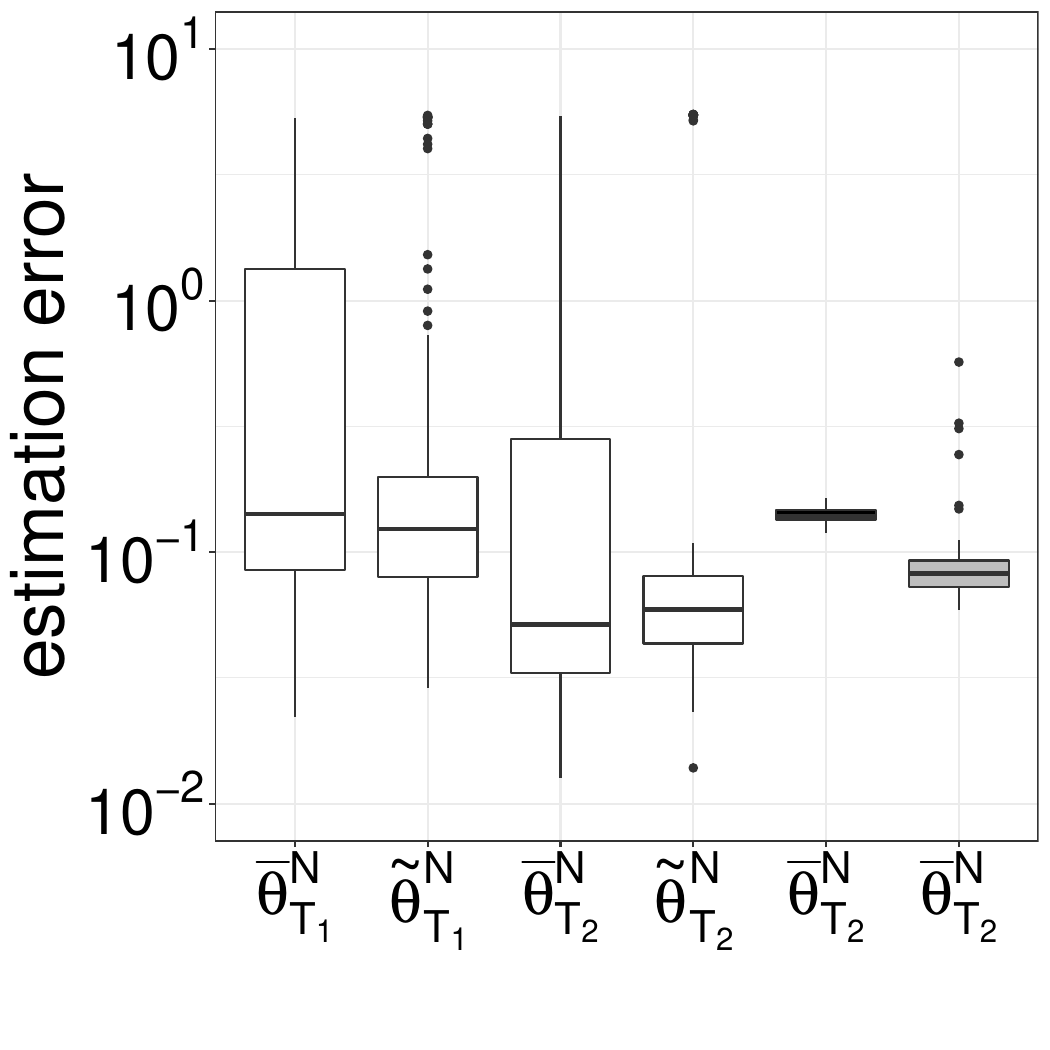}
       \caption{\label{fig:gandk_2}}
    \end{subfigure}
    ~
    \begin{subfigure}[b]{0.31\textwidth}
        \centering
       \includegraphics[scale=0.26]{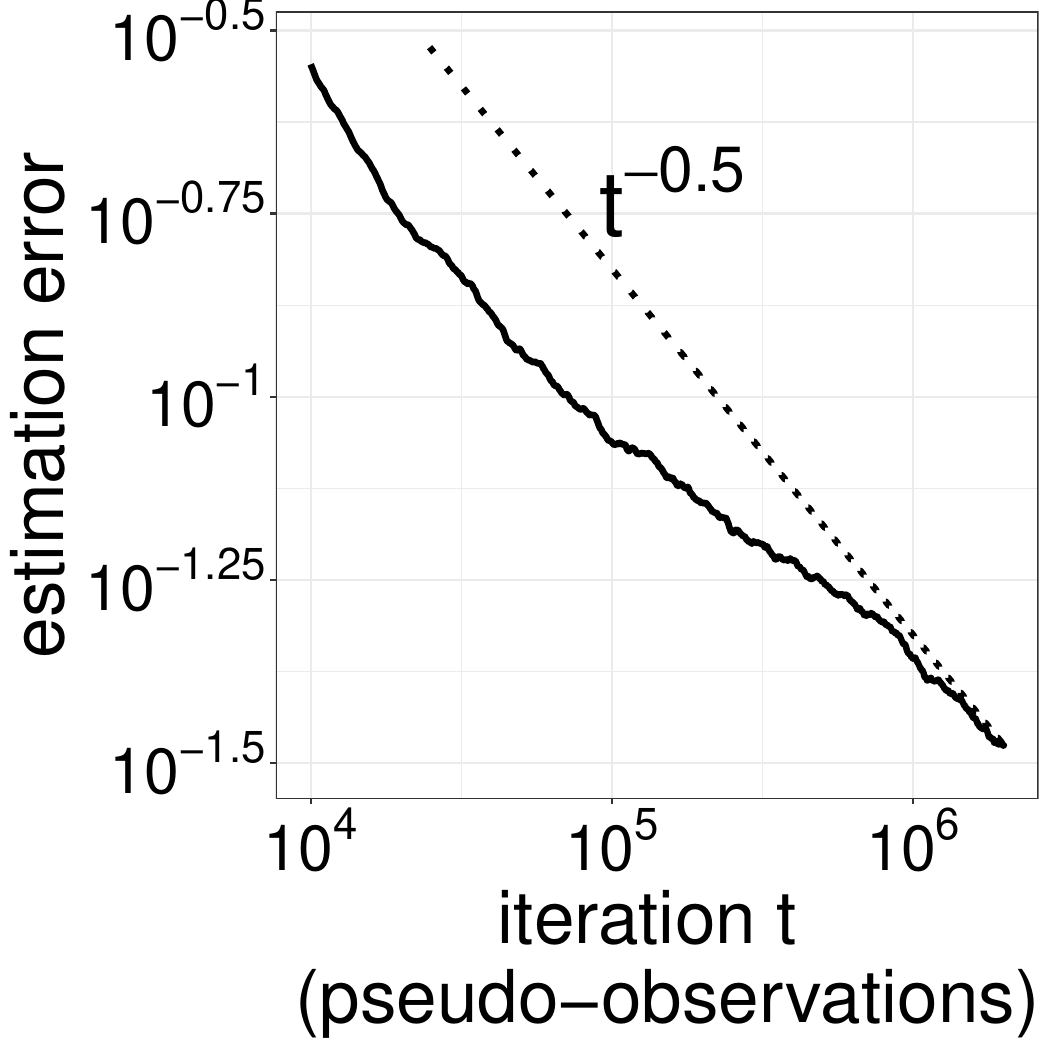}
       \caption{\label{fig:gandk_3}}
    \end{subfigure}
    
\caption{Example of Section \ref{sub:gandk}. Plot (\subref{fig:gandk_1}) compares the marginal distributions of $f_\theta(y)\eta(\dd y)$ and the empirical marginal distributions for $\theta=\hat{\theta}_{\mathrm{mle},n}$ (left) and for $\theta= \theta^{(-5.5,0)}$ (right). In Plot  (\subref{fig:gandk_2}) $T_1=10\,000$, $T_2=50\,000$ while $(\alpha,\csigma)=(0.8,10)$ (white boxplots), $(\alpha,\csigma)=(0.5,10)$ (fifth boxplot) and $(\alpha,\csigma)=(0.5,1)$  (sixth  boxplot).  In Plot (\subref{fig:gandk_3})  the dotted line  is as in Figures \ref{fig:CQR_5}-\ref{fig:mixture_boxplot_K2}.\label{fig:gank}}
\end{figure}

In Figure \ref{fig:gandk_2} we summarize the  100 values obtained for $\|\bar{\theta}^N_{T'}-\hat{\theta}_{\mathrm{mle},{n}}\|_{\infty}$ and for $\|\tilde{\theta}^N_{T'}-\hat{\theta}_{\mathrm{mle},{n}}\|_{\infty}$, with $T'\in\{10^4,T\}$. We remark that after only $T'=10^4$ iterations  the median  estimation error  is only approximatively 0.13 for the two estimators,  and that increasing the number of iterations to $T=50\,000$ approximately divides this value by 2.75 for $\bar{\theta}_t^N$ and by 2.1 for $\tilde{\theta}_t^N$. In this figure we also remark that for 15 runs of G-PFSO  the estimation error of $\tilde{\theta}_{T}^N$  is between 5.10 and 5.51. A close look at the values obtained for $\tilde{\theta}_{T}^N$ reveals that these large estimation errors arise because G-PFSO occasionally converges to an element of the set
$$
\Theta_{\mathrm{loc},\star}=\big\{\theta^{(v_1,v_2)},\,\,(v_1,v_2)\in\{5.5,-5.5,0\}^2\setminus\{(0,0)\}\big\},
$$ 
where for $(v_1,v_2)\in\R^2$ the notation $\theta^{(v_1,v_2)}$ is used for a $\theta=(a_1, a_2,b_1, b_2,g_1, g_2,k_1, k_2,\rho)\in\Theta$ such that $(g_1,g_2)= (v_1,v_2)$. By contrast, under  $\hat{\theta}_{\mathrm{mle},n}$ both $g_1$ and $g_2$ are close to zero. As illustrated in the right panel of Figure \ref{fig:gandk_1}, when $v_j\approx\pm 5.5$ the probability density function $f_{\theta^{(v_1,v_2)}}$ tries to capture the large spike at zero that is present in the  empirical distribution of the observations $\{\tilde{y}_{i,j}\}_{i=1}^n$, and thus  $\Theta_{\mathrm{loc},\star}$ is a natural set of    local maxima of the log-likelihood function. Due to the large distance $\mathrm{d}(\hat{\theta}_{\mathrm{mle},{n}},\Theta_{\mathrm{loc},\star})=\min_{\theta\in\Theta_{\mathrm{loc},\star}}\|\theta-\hat{\theta}_{\mathrm{mle},{n}}\|_\infty\approx 5.5$ between  the maximum likelihood estimator and this set, with only $N=500$ particles and a fast learning rate $h_t=t^{-0.8}$ it is difficult for G-PFSO to escape from one  element of $\Theta_{\mathrm{loc},\star}$ and to reach  a small neighbourhood of $\hat{\theta}_{\mathrm{mle},{n}}$. This  explains why, in 15 of the 100 runs of Algorithm \ref{algo:pi_tilde_online}, the two estimators  are still stuck around an element of $\Theta_{\mathrm{loc},\star}$ after   50\,000 iterations. Unreported results obtained by maximising $l_n(\theta)$ with a quasi-Newton algorithm suggest that the log-likelihood function has  stationary points that are not in $\Theta_{\mathrm{loc},\star}$, and thus that $l_n(\theta)$ may have   local maxima other than those belonging to this set.

We now repeat the above experiment with $\alpha=0.5$. The 100 resulting values of $\|\bar{\theta}_{T}^N-\hat{\theta}_{\mathrm{mle},{n}}\|_\infty$, summarized in Figure \ref{fig:gandk_2}, fifth boxplot, are all smaller than 0.17,  showing that reducing $\alpha$  improves the ability of G-PFSO to quickly find the highest mode of the objective function, as discussed  in Section \ref{sub:h_t}. On the other hand, when $\alpha=0.8$ the estimate $\tilde{\pi}_{T}^N$ of $\tilde{\pi}_{T}$ computed by Algorithm \ref{algo:pi_tilde_online} can be more concentrated around $\hat{\theta}_{\mathrm{mle},n}$  than when $\alpha=0.5$ and, for this reason, we see in  Figure \ref{fig:gandk_2} that $\|\bar{\theta}_{T}^N-\hat{\theta}_{\mathrm{mle},{n}}\|_\infty$ can be much smaller for the former value of $\alpha$ than for the latter. In Section \ref{sub:Multimodal} we saw that  for a given choice of $\alpha$  reducing $\csigma$ can   reduce the estimation error of G-PFSO. This point is illustrated  further in the sixth boxplot in Figure \ref{fig:gandk_2}, which summarizes the values of $\|\bar{\theta}_{T}^N-\hat{\theta}_{\mathrm{mle},{n}}\|_\infty$  obtained in 100 runs of Algorithm \ref{algo:pi_tilde_online} with $(\alpha,\csigma)=(0.5,1)$. We however observe that, for reasons explained in Section \ref{sub:Multimodal}, reducing $\csigma$ from 10 to 1 increases the probability of having  a `large' estimation error, since for $\csigma=1$ we have   $\|\bar{\theta}_{T}^N-\hat{\theta}_{\mathrm{mle},{n}}\|_\infty>0.24$ in 4 runs of the algorithm while, as   mentioned above, for the same value  of $\alpha$ this event never happens in Figure \ref{fig:gandk_2}  when $\csigma=10$.

Finally, in Figure \ref{fig:gandk_3} we report the evolution of $\|\bar{\theta}^N_{t}-\hat{\theta}_{\mathrm{mle},{n}}\|$ as the number of iterations $t$ increases, averaged over 5  runs of G-PFSO with $(\alpha,\csigma)=(0.5,1)$. As in the previous three examples,  we observe that  for  $h_t=t^{-1/2}$ the estimator $\bar{\theta}^N_{t}$ converges to the target parameter value at the optimal $t^{-1/2}$ rate.

\section{Future work and practical recommendations\label{sec:conclusion}}

The full theoretical justification for G-PFSO  is still in progress  but the already obtained result provides an important preliminary step towards a precise analysis of this algorithm.  Notably,  future work should aim at validating, or not, the $t^{-1/2}$ convergence rate observed  for the estimator $\bar{\theta}_t^N$, with $N\in\mathbb{N}$ fixed.  In particular, we conjecture that  for $h_t=t^{-1/2}$ and $N$ large enough, $\bar{\theta}_t^N$ has this convergence behaviour  with probability at least $p_N>0$, where   $p_N\rightarrow 1$ as $N\rightarrow\infty$. In addition,  following the discussion in  Section \ref{sub:extension}  we conjecture that  under the additional constraint that $\nu\leq 2$  we have $\|\bar{\theta}_t^N-\theta_\star\|=\bigO_\P(t^{-1/2})$, if $N$ is sufficiently large.

Iterative optimization methods are usually run until a given stopping criterion,   guaranteeing some control on the estimation error, is fulfilled. Ideally, a similar approach should be followed when G-PFSO is applied to observations $(Y_t)_{t\geq 1}$ that are sampled by the user, as in Section \ref{sub:gandk}. Unfortunately, defining a good stopping criterion for global stochastic optimization is known to be a hard task and ``most stochastic global optimization users just let their algorithm run until some time limit is exhausted'' \citep{Schoen2009}. Future research should aim at designing a   better strategy for stopping G-PFSO.

G-PFSO requires the user to specify a few ingredients, and we end this paper by proposing some default choices for them. For reasons given in Section \ref{sub:h_t}  our recommended default choice for the learning rate is $h_t=t^{-1/2}$. Depending on the estimation problem at hand, when the number of degrees of freedom  $\nu$ of the  Student's $t$-distributions appearing in \eqref{eq:mu_def}  is small, and notably when $\nu\leq 2$,   that is when the Student's $t$-distributions have an infinite variance,   G-PFSO may, as  in Section \ref{sub:Multimodal}, or may not, as in Section \ref{sub:mixture}, work better than when $\nu$ is large. Since in all the challenging problems of   Section \ref{sec:num}  we observe that G-PFSO performs well  when using Student's $t$-kernels  with thin tails,  we  recommend    by default to choose a  value for $\nu$ which is not too small, e.g.\ to let $\nu\geq 5$.  Finally, the covariance matrix $\Sigma$ of the Markov kernels  used to generate the new particles is another important ingredient of G-PFSO. Unfortunately, the `optimal' choice for this parameter is problem-dependent, as it should depend on the landscape of the  function  to be optimized. Focussing on the case where $\Sigma=\csigma I_d$, to escape more easily from a local mode $\csigma$ should be large if the objective function has a lot of local optima, or if its modes are far apart, and small otherwise to improve the concentration of the particle system around $\theta_\star$, and thus to reduce the estimation error. Since  G-PFSO is designed to address difficult optimization tasks, where finding the global optimum is the main challenge, it is  sensible to choose by default a value for $\csigma$ which is not too small, e.g.\ to let $\csigma\geq 1$.

\bibliographystyle{apalike}
\bibliography{complete}

\appendix

\section{Appendix: Proofs}\label{ap-proofs}
\subsection{Roadmap}

In Section \ref{S-sec:pi} we introduce a more general definition of the sequence $(\tilde{\pi}_t)_{t\geq 1}$, which notably does not assume that $h_t>0$ is positive for all $t$. Then, Theorem  \ref{t-thm:online} follows from Theorem \ref{thm:online}  and from Lemma \ref{lem:lower_bound_1}, which are derived for this  more general definition of  $(\tilde{\pi}_t)_{t\geq 1}$ and proven in Section \ref{p-thm:online} and in Section \ref{p-lem:lower_bound_1}, respectively. Proposition \ref{t-prop:h_t} is a direct consequence of the slightly more general Proposition \ref{prop:h_t}, proven in Section \ref{p-prop:h_t}, while Proposition \ref{prop:assumption} is proven in Section \ref{p-prop:assumption}. A complement to Section \ref{sub:extension} of the paper is given in Section \ref{psec:extensions} and  additional information for the numerical experiments is provided in Section \ref{sec:additional_info}.

\subsection{A more general definition of the sequence  \texorpdfstring{$(\tilde{\pi}_t)_{t\geq 1}$}{Lg}}\label{S-sec:pi}

The definition of the  sequence $(\tilde{\pi}_t)_{t\geq 1}$  requires to specify  $\nu\in(0,\infty)$ and three sequences, namely a sequence $(h_t)_{t\geq 0}$ in $[0,\infty)$,  a  strictly increasing sequence  $(t_p)_{p\geq 0}$   in  $\mathbb{N}_0:=\{0\}\cup\mathbb{N}$ and a sequence $(\Sigma_t)_{t\geq 0}$  of $d\times d$ covariance matrices verifying
\begin{align}\label{S-eq:Sigma_t}
\sup_{t\geq 0}\big(\|\Sigma_t\|\vee  \|\Sigma^{-1}_t\|\big)<\infty.
 \end{align}
Next,  given a  distribution $\tilde{\pi}_0\in\mathcal{P}_L(\Theta)$, and defining
\begin{align}\label{S-eq:mu_def}
\mu_t(\dd\theta)=
\begin{cases}
\delta_{\{0\}}, &h_t=0\\
\mathcal{N}_d(0,h_t^2\Sigma_t), & h_t>0\text{ and } t\not\in (t_p)_{p\geq 0}\\
t_{d,\nu}(0,h_t^2\Sigma_t), & h_t>0\text{ and } t \in (t_p)_{p\geq 0}
\end{cases},\quad \forall t\geq 0,
\end{align}
we let
\begin{align}\label{S-eq:pi_tilde}
 \tilde{\pi}_t(\dd\theta)=\frac{ f_\theta(Y_t)(\mu_{t-1}*\tilde{\pi}_{t-1})\vert_\Theta(\dd\theta)}{\int_{\R^d} f_\theta(Y_t)(\mu_{t-1}*\tilde{\pi}_{t-1})\vert_\Theta(\dd\theta)}\in \mathcal{P}_L(\Theta) ,\quad  t\geq 1.
\end{align}

\subsection{Theoretical results for \texorpdfstring{$(\tilde{\pi}_t)_{t\geq 1}$}{Lg} as defined in Section \ref{S-sec:pi}}\label{S-sec:results}
 
\begin{theorem}\label{thm:online}
Let $(h_t)_{t\geq 0}$ and $(t_p)_{p\geq 0}$  be such that
\begin{enumerate}
\item\label{thm1.1}  $h_{t_p}>0$ for all $p\geq 0$,
\item\label{thm1.2} $(t_{p+1}-t_p) \rightarrow\infty$ and  $(t_p-t_{p-1}) \sum_{s=t_{p-1}+1}^{t_p-1}h_s^2 \rightarrow 0$,
\item\label{thm1.3}  $\limsup_{p\rightarrow\infty}\frac{t_{p+1}-t_{p}}{ t_{p}-t_{p-1}}<\infty$ and $\log(h_{t_{p-1}})/(t_p-t_{p-1})\rightarrow 0$.
\end{enumerate}
Moreover, assume  that
\begin{align}\label{eq:compact}
\exists (\beta, c)\in (1,\infty)^2\quad\text{such that}\quad \P\big(\tilde{\pi}_{t_p} (V_{c/h_{t_{p}}^{\beta}})\geq \beta^{-1}\big)\rightarrow 0.
\end{align}
Then, under Assumptions \ref{m_star}-\ref{new1}, $\tilde{\pi}_{t}\Rightarrow\delta_{\{\theta_\star\}}$ in $\P$-probability.
\end{theorem}

Condition \eqref{eq:compact} of Theorem \ref{thm:online} holds when the parameter space $ \Theta$ is bounded, since in this case we have $\tilde{\pi}_t(\Theta^c)=0$ for all $t\geq 0$, $\P$-a.s. Lemma \ref{lem:lower_bound_1} below provides sufficient conditions for \eqref{eq:compact} to hold when $\Theta$ is unbounded.  We recall that for two strictly  positive sequences $(a_t)_{t\geq 1}$ and $(b_t)_{t\geq 1}$ the notation $a_t=\Theta(b_t)$ means that $\liminf_{t\rightarrow\infty} (a_t/b_t+b_t/a_t)>0$.

\begin{lemma}\label{lem:lower_bound_1}
Assume Assumptions \ref{m_star}, \ref{taylor}, \ref{new1} and  \ref{new2}. For every   $C\in(0,\infty)$ let
$$
\zeta(C)=
\begin{cases}
\E\big[ \sup_{\theta\in V_{C}}\log (\tilde{f}_{\theta}/f_{\theta_\star}) \big],&\text{if Assumption \ref{new2}.\ref{A61} holds}\\
 \log \big(\sup_{\theta\in V_C}\E[(\tilde{f}_{\theta}/f_{\theta_\star})]\big),&\text{if Assumption \ref{new2}.\ref{A62} holds}\\
 \log \big(\sup_{\theta\in V_C}\E[ \tilde{f}_{\theta} ]\big), &\text{if Assumption \ref{new2}.\ref{A63} holds}
\end{cases}
$$
and let $ k_\star\in\{1/2\}\cup \mathbb{N}  $ be as in Assumption \ref{new2}. Let $(h_t)_{t\geq 0}$ and $(t_p)_{p\geq 0}$ be such that Conditions \ref{thm1.1}-\ref{thm1.2} of Theorem \ref{thm:online} hold and such that
\begin{enumerate}
\item\label{thm2.1} $h_{t_p}<h_{t_{p-1}}$ for all $p\geq 0$,
\item\label{thm2.2} $h_{t_p}=\Theta(t_p^{-\alpha})$ for some $\alpha>0$,
\item\label{thm2.3} There exists a  constant      $\beta_\star\in(0,\infty)$ such that
\begin{align}\label{eq:thm2}
 \big|\zeta(h_{t_{p}}^{-\beta_\star})\big|^{-2k_\star}\sum_{i=1}^p (t_i-t_{i-1})^{-k_\star\ind_{\mathbb{N}}(k_\star)}\rightarrow 0.
\end{align}
\end{enumerate}

Then, Condition \eqref{eq:compact} of Theorem \ref{thm:online} holds.
\end{lemma}

\begin{remark}
We note that the conclusions of  Theorem \ref{thm:online} and of Lemma \ref{lem:lower_bound_1} remain valid when  $(\Sigma_t)_{t\geq 0}$ is a   random  sequence of covariance matrices, provided that for all $t>t_2$,    $\Sigma_t$ is independent of $(Y_s)_{s>t_{p_t-2}}$,  with $p_t=\sup\{p\in\mathbb{N}_0:\, t_p<t\}$.
\end{remark}

The following proposition provides sufficient conditions on $(h_t)_{t\geq 0}$ and on $(t_p)_{p\geq 0}$ to ensure that Conditions \ref{thm1.1}-\ref{thm1.3} of Theorem \ref{thm:online} and Conditions \ref{thm2.1}-\ref{thm2.3} of Lemma \ref{lem:lower_bound_1} hold.

\begin{proposition}\label{prop:h_t}

Let $C\in[1,\infty)$, $\alpha\in(0,\infty)$,  $\varrho\in(0,\alpha\wedge 1)$ and $c\in(0,1)$ be some constants. Let $(t_p)_{p\geq 0}$ be defined by
\begin{align}\label{eq:tp_formula}
t_0\in\mathbb{N}_0\quad C_{p-1}\in [c t^\varrho_{p-1},  t_{p-1}^\varrho/c],\quad t_p= t_{p-1}+\lceil C_{p-1} \log(t_{p-1})\vee C\rceil,\quad   p\geq 1
\end{align}
and let $(h_t)_{t\geq 0}$ be such that  $h_{t_p}=\Theta(t_p^{-\alpha})$, such that $h_{t_p}<h_{t_{p-1}}$ for all $p\geq 1$  and such that, for all $t\geq 1$, either $h_{t-1}>h_t$ or $h_{t-1}=0$. Then, the sequences $(h_t)_{t\geq 0}$ and $(t_p)_{p\geq 0}$ verify Conditions \ref{thm1.1}-\ref{thm1.3} of Theorem \ref{thm:online} and Conditions \ref{thm2.1}-\ref{thm2.2} of Lemma \ref{lem:lower_bound_1}. In addition,   the sequences $(h_t)_{t\geq 0}$ and $(t_p)_{p\geq 0}$ also verify  Condition \ref{thm2.3} of  Lemma \ref{lem:lower_bound_1} if  Assumption \ref{new2} holds  for a $k_\star>(1+\varrho)/\varrho$.
\end{proposition}

\subsection{Proof of Proposition  \ref{prop:assumption}\label{p-prop:assumption}}
\begin{proof}
The proof of the second part of the proposition is    similar to that of the first part  and, to save place, only this latter is given below.

For every $C\in\R_{>0}$ let  $A_{C}=\{(\mu,\sigma^2)\in\R^2:\, |\sigma^2|< C,\,|\mu|<C\}$. We first show that if $C$ is large enough then  Assumption \ref{new} holds for $A_\star=A_{C_\star}$.

Let $B_{1,C}=\{(\mu,\sigma^2)\in\Theta:\, \sigma^2\geq C\}$ and  $B_{2,C}=\{(\mu,\sigma^2)\in\Theta:\,|\mu|\geq C\}$ so that $A^c_C\cap\Theta=B_{1,C}\cup B_{2,C}$. Then,
\begin{align}\label{eq:p_a1}
\E[\sup_{\theta \not\in A_C}\log (\tilde{f}_{\theta})]\leq \E[\sup_{\theta \in B_{1,C}}\log (f_{\theta})]+\E[\sup_{\theta \in B_{2,C}}\log (f_{\theta})]
\end{align}
where
\begin{align}\label{eq:p_a2}
\E[\sup_{\theta \in B_{1,C}}\log (\tilde{f}_{\theta})]\leq -\frac{1}{2}\log( 2\pi C).
\end{align}

To proceed further let $\theta\in B_{2,C}$, $y\in\R$ and note that
\begin{align*}
f_\theta(y)&= f_\theta(y)\ind(|y|\geq C/2)+ f_\theta(y)\ind(|y|< C/2)\\
&\leq \frac{1}{\sqrt{2\pi\underline{\sigma}^2}}\ind(|y|\geq C/2)+\frac{1}{\sqrt{2\pi\sigma^2}}e^{-\frac{(y-C)^2}{2\sigma^2}}\ind(|y|< C/2)\\
&\leq \frac{1}{\sqrt{2\pi\underline{\sigma}^2}}\ind(|y|\geq C/2)+\frac{1}{\sqrt{2\pi\sigma^2}}e^{-\frac{C^2}{8\sigma^2}}\ind(|y|< C/2)\\
&\leq \frac{1}{\sqrt{2\pi\underline{\sigma}^2}}\ind(|y|\geq C/2)+\frac{1}{\sqrt{\pi C^2/2}}e^{-\frac{1}{2}}\ind(|y|< C/2).
\end{align*}
Therefore,
\begin{align*}
\E[\sup_{\theta \in B_{2,C}}\log (f_{\theta})]&\leq -\frac{1}{2}\log(2\pi\underline{\sigma}^2)\P(|Y_1|\geq  C/2)-\frac{1}{2}\log(\pi C^2/2)\P(|Y_1|< C/2)
\end{align*}
which, together with \eqref{eq:p_a1} and  \eqref{eq:p_a2}, shows that
\begin{align*}
\E[\sup_{\theta \not\in A_C}\log (\tilde{f}_{\theta})]&\leq -\frac{1}{2}\log( 2\pi C)-\frac{1}{2}\log(2\pi\underline{\sigma}^2)\P(|Y_1|\geq  C/2)-\frac{1}{2}\log(\pi C^2/2)\P(|Y_1|< C/2).
\end{align*}
Since the r.h.s.\ converges to $-\infty$ as $C\rightarrow\infty$ it follows that Assumption \ref{new}.\ref{A41}) holds for $A_\star=A_{C_\star}$, for a sufficiently  large constant $C_\star\in\R_{>0}$.

To show the second part of Assumption \ref{new} let $\tilde{A}_\star$ be an arbitrary compact set that contains a neighbourhood of $A_{C_\star}$. Note that $\tilde{A}_\star\cap\Theta$ is compact and that the mapping $\theta\mapsto f_\theta(y)$ is continuous on $\tilde{A}_\star\cap\Theta$ for all $y\in\R$. Then, since  for all $(\tilde{\theta},\theta)\in\Theta^2$ we have
\begin{align*}
\log\big( (f_{\tilde{\theta}}/f_\theta)(y)\big)=\frac{1}{2}\log( \sigma^2/\tilde{\sigma}^2)-\frac{1}{2}\Big(\frac{\sigma^2-\tilde{\sigma}^2}{\sigma^2\tilde{\sigma^2}}y^2+\frac{\tilde{\mu}^2}{\tilde{\sigma}^2}-\frac{\mu^2}{\sigma^2}-2y\frac{\tilde{\mu}\sigma^2-\mu\tilde{\sigma}^2}{\sigma^2\tilde{\sigma^2}}\Big)
\end{align*}
it follows that the second part of Assumption \ref{new} holds since $\E[Y_1^4]<\infty$ by assumption. This concludes to show that Assumption \ref{new} holds

To show that Assumption \ref{new1}.\ref{A53}   holds   it suffices to remark that $f_\theta(y)\leq (2\pi\underline{\sigma}^2)^{-1/2}$ for all $\theta\in\Theta$ and $y\in\setY$.

We now show that Assumption \ref{new2} holds. To this aim let $\theta_\star=(\mu_\star,\sigma_\star^2)\in\Theta$ and note that, for all $\sigma^2\in\R_{>0}$ and $\delta\neq 0$, we have
\begin{equation}\label{eq:CC}
\begin{split}
\E\Big[\exp\Big(-\frac{(Y_1-(\mu_\star+\delta))^2}{2\sigma^2}
\Big)\Big]&\leq  e^{- \delta^2/(8\sigma^2)}+ \P\big(|Y_1-\mu_\star|\geq |\delta|/2 \big)\\
&\leq   e^{- \delta^2/(8\sigma^2)} +\frac{2 \E[|Y_1-\mu_\star|]}{|\delta|}.
\end{split}
\end{equation}
where the last inequality uses Markov's inequality.

Let $\theta_C\in V_C$,  
$$
\epsilon_1=\frac{ |\sigma_C^2-\sigma_\star^2|}{C},\quad \epsilon_2=\frac{ |\mu_C-\mu_\star|}{C} 
$$
and remark that $\epsilon_1\vee \epsilon_2\geq 1/\sqrt{2}$. Assume first that  $\epsilon_2\geq  1/\sqrt{2}$. Then,  using \eqref{eq:CC} with $\delta=\mu_C-\mu_\star$, we have
\begin{equation}\label{eq:CC1}
\begin{split}
\E[f_{\theta_C}(Y_1)]&\leq \frac{1}{\sqrt{2\pi\underline{\sigma}^2}}\bigg(\exp\Big(- \frac{\epsilon_2^2 C^2}{8\sigma_C^2}\Big) +\frac{2 \E[|Y_1-\mu_\star|]}{\epsilon_2 C}\bigg)\\
&\leq \frac{1}{\sqrt{2\pi\underline{\sigma}^2}}\bigg(\exp\Big(- \frac{ C^2}{16 \underline{\sigma}^2}\Big) +\frac{2^{3/2} \E[|Y_1-\mu_\star|]}{ C}\bigg).
\end{split}
\end{equation}
Assume now that $\epsilon_2< 1/\sqrt{2}$, so that $\epsilon_1\geq 1/\sqrt{2}$.  Note that, for $C> 2^{3/2}\sigma^2_\star$ we have
$$
\sigma_C^2=\big|\sigma^2_\star-\sigma^2_C-\sigma^2_\star\big|\geq \epsilon_1 C -\sigma^2_\star\geq C/\sqrt{2}-\sigma^2_\star\geq \sigma^2_\star>0. 
$$
Therefore, for $C> 2^{3/2}\sigma^2_\star$ we have
\begin{equation}\label{eq:CC2}
\begin{split}
\E[f_{\theta_C}(Y_1)]&\leq \Big(2\pi(\sigma^2_\star-\sigma^2_C-\sigma^2_\star)\Big)^{-1/2}\leq   \Big(2\pi(C/\sqrt{2}-\sigma_\star^2)\Big)^{-1/2}. 
\end{split}
\end{equation}
Then, using \eqref{eq:CC1}-\eqref{eq:CC2}, it follows that there exists a constant $C'\in\R_{>0}$ such that, for   $C>0$ large enough we have
\begin{align*}
\sup_{\theta\in V_C}\E[f_{\theta}(Y_1)]&\leq \frac{1}{\sqrt{2\pi\underline{\sigma}^2}}\bigg(\exp\Big(- \frac{ C^2}{16 \underline{\sigma}^2}\Big) +\frac{2^{3/2} \E[|Y_1-\mu_\star|]}{ C}\bigg)+\Big(2\pi(C/\sqrt{2}-\sigma_\star^2)\Big)^{-1/2}\\
&\leq  \frac{C'}{C^{1/2}}
\end{align*}
showing that 
\begin{align*}
\limsup_{C\rightarrow\infty}\frac{\log\big(\sup_{\theta\in V_C}\E[f_{\theta}(Y_1)]\big)}{\log (C)}\leq \limsup_{C\rightarrow\infty}\frac{\log(C')-\frac{1}{2}\log(C)}{\log (C)}=-\frac{1}{2}<0.
\end{align*}
Hence the second part of Assumption \ref{new2}.\ref{A63} holds. In addition, since
$$
|\log(f_{\theta_\star}(y))|^{p}\leq 2^{p-1}\Big|\frac{1}{2}\log(2\pi\sigma_\star^2)\Big|^{p}+\frac{|y-\mu_\star|^{2p}}{2(\sigma_\star^2)^p},\quad\forall p\in\mathbb{N},
$$
it follows that the first part of Assumption \ref{new2}.\ref{A63} holds for every $k_\star\in\mathbb{N}$ since, by assumption, $\E[e^{c |Y_1|}]<\infty$ for some constant $c>0$. The proof is complete.
\end{proof}

\subsection{Proofs of Theorem \ref{thm:online}\label{p-thm:online}}

\subsubsection{Additional notation and conventions}

  Below we use  the convention that empty sums equal zero and empty products equal one, and let $(\tilde{U}_t)_{t\geq 0}$ be   a sequence of independent random variables such that
  $$
  \tilde{U}_t\sim 
  \begin{cases}
  \mathcal{N}_d(0, \Sigma_t), &t\not\in (t_p)_{p\geq 0}\\
  t_{d,\nu}(0, \Sigma_t), & t \in (t_p)_{p\geq 0}
  \end{cases},\quad\forall t\geq 0.
  $$ 
  For all $t\geq 0$ we let $U_t=h_t\tilde{U}_t$; notice that $U_t\sim\mu_t$.

Next, for every integers $0\leq k< t$ we define
$$
u_{k:t}=(u_{k},\dots, u_{t}),\quad [u_{k:t}]=  \max_{k\leq  s<t}\big\|\sum_{i=s}^{t-1}u_i\big\|,\quad  (u_{k},\dots, u_{t})\in\R^{d(t-k+1)}
$$
and  let
$$
\Theta_{\epsilon,k:t}=\big\{u_{k:t}\in \R^{d(t-k+1)}:\,[u_{(k+1):t}]<\epsilon\big\},\quad\epsilon>0.
$$

Lastly, for every $t\geq 0$ we let $\F_t$ be   the $\sigma$-algebra generated by $(Y_1,\dots, Y_t)$ (with the convention $\F_0=\emptyset$) and, for   every $0\leq k<t $ and $A\in\mathcal{B}(\R^d)$, we let
$$
\pi'_{k,t}(A)=\frac{\int_{A} (\mu_{k}*\tilde{\pi}_{k})\big(\theta-\sum_{s=k+1}^{t-1}U_s\big)\prod_{s=k+1}^{t}  \tilde{f}_{\theta-\sum_{i=s}^{t-1}U_i}(Y_s)\dd\theta}{\int_{\Theta}\E\big[ (\mu_{k}*\tilde{\pi}_{k})\big(\theta-\sum_{s=k+1}^{t-1}U_s\big)\prod_{s=k+1}^{t}  \tilde{f}_{\theta-\sum_{i=s}^{t-1}U_i}(Y_s)\big|\F_{t}\big]\dd\theta}.
$$

\subsubsection{Preliminary results}\label{sec:p-prelim}

The following result \citep[see][Proposition 6.2 and above comment, page 124]{Ghosal2017} provides the following necessary and sufficient condition for checking that a sequence $(\nu_t)_{t \geq 1}$ of probability measures on  $\Theta$ (implicitly indexed by random variables) is such that $\nu_t \Rightarrow \delta_{\{\theta_\star\}}$ in $\P$-probability.

\begin{proposition}\label{prop:conv_proba}
$\nu_t \Rightarrow \delta_{\{\theta_\star\}}$ in $\P$-probability if and only if $\E[\nu_t(V_\delta)]\rightarrow 0$ for all $\delta>0$.
\end{proposition}

\begin{lemma}\label{lemma:new}
Assume Assumption \ref{new} and let $A_\star\in \mathcal{B}(\R^d)$ and $\tilde{A}_\star\in \mathcal{B}(\R^d)$ be as in  Assumption \ref{new}. Then, there exists a set $A'_\star\in  \mathcal{B}(\R^d)$, with $A'_\star\subsetneq \tilde{A}_\star$, that contains a  neighbourhood of  $A_\star$ and such that, for every sequence $(\gamma'_t)_{t\geq 1}$ in $\R_{\geq  0}$  such that $\gamma'_t\rightarrow 0$  and every sequence $(s_t)_{t\geq 1}$ in $\mathbb{N}_0$ such that  $\inf_{t\geq 1}(t-s_t)\geq 1$ and such that $(t-s_t)\rightarrow\infty$,   there exists a sequence $(\delta_t)_{t\geq 1}$ in $\R_{>  0}$ such that $\delta_t\rightarrow 0$ and such that
$$
\P\Big(\sup_{(u_{s_t:t},\theta)\in \Theta_{\gamma_t',s_t:t}\times (A'_\star\cap\Theta)}\prod_{s=s_t+1}^t  (\tilde{f}_{\theta-\sum_{i=s}^{t-1} u_i}/f_{\theta}) (Y_s)<e^{(t-s_t)\delta_{t}}\Big)\rightarrow 1.
$$
\end{lemma}
The proof of this result is given in Section \ref{p-lemma:new}.

The following results rewrite  the probability measure $\tilde{\pi}_t$ in a more convenient  way.

\begin{lemma}\label{lemma:pi}
With $\P$-probability one we have, for all $t\geq 0$ and all $A\in\mathcal{B}(\R^d)$,
$$
\tilde{\pi}_t(A)=\frac{\int_A \E\big[ (\mu_0*\tilde{\pi}_0) \big(\theta-\sum_{s=1}^{t-1}U_s\big)\prod_{s=1}^{t} \tilde{f}_{\theta -\sum_{i=s}^{t-1}U_i}(Y_s)\big|\F_t\big]\dd\theta}{\int_{\Theta} \E\big[ (\mu_0*\tilde{\pi}_0) \big(\theta-\sum_{s=1}^{t-1}U_s\big)\prod_{s=1}^{t} \tilde{f}_{\theta -\sum_{i=s}^{t-1}U_i}(Y_s)\big|\F_t\big]\dd\theta}.
$$
\end{lemma}
The proof of this result is given in Section \ref{p-lemma:pi}.

The next result builds on \citet[][Lemma 8.1]{MR1790007} and will be used to control the denominator of $\tilde{\pi}_t(\theta)$.
\begin{lemma}\label{lemma:denom}
Assume Assumptions \ref{m_star}-\ref{taylor} and let  $\delta_\star>0$ be as in these two assumptions. Then,  there exists a constant $\tilde{C}_{\star}\in(0,\infty)$ such that,  for every $\epsilon\geq 0$, sequence $(s_t)_{t\geq 1}$  in $\mathbb{N}_0$ with $\inf_{t\geq 1}(t-s_t)\geq 1$, every constants $\delta\geq \tilde{\delta} >0$ such that $\delta+\tilde{\delta}<\delta_\star$  and  every   probability measure $\eta\in\mathcal{P}(\R^d)$  we have, for all  $t\geq 1$ and with $C^\eta_{\delta, \tilde{\delta}}=\inf_{v\in B_{\tilde{\delta}}(0)}\eta(B_\delta(\theta_\star-v))$,
\begin{align*}
\P\bigg(&\int_{\Theta}\E\Big[\eta\big(\theta-\sum_{s=s_t+1}^{t-1}U_s\big)\prod_{s=s_t+1}^{t} (\tilde{f}_{\theta-\sum_{i=s}^{t-1}U_i}/f_{\theta_\star})(Y_s)\big|\F_{t}\Big]  \dd\theta \leq  \frac{\P(U_{s_t:t}\in  \Theta_{\tilde{\delta},s_t:t}) }{e^{(t-s_t)(2 (\tilde{C}_\star \delta)^2+\epsilon)}}C^\eta_{\delta, \tilde{\delta}}\bigg)\\
& \leq \big((t-s_t)\big((\tilde{C}_\star\delta)+(\tilde{C}_\star\delta)^{-1}\epsilon\big)^2\big)^{-1}.
\end{align*}
\end{lemma}
The proof of this result is given in Section   \ref{p-lemma:denom}.

The next result  will be used to control the numerator  of $\tilde{\pi}_t(\theta)$.
\begin{lemma}\label{lemma:test}
Assume  Assumptions \ref{test}-\ref{new} and let $(\gamma_t)_{t\geq 1}$ be a sequence in $[0,\infty)$ such that $\gamma_t\rightarrow 0$ and $(s_t)_{t\geq 1}$ be a sequence in $\mathbb{N}_0$ such that $\inf_{t\geq 1}(t-s_t)\geq 1$ and   $(t-s_t)\rightarrow\infty$. Then, for every $\epsilon>0$ there exist a constant  $ \tilde{D}_\star \in\R_{>0}$ and a sequence of measurable functions $(\phi_t)_{t\geq 1}$, $\phi_t:\setY^t\rightarrow\{0,1\}$, such that $\E[\phi_t(Y_{1:t})]\rightarrow 0$ and, for $t$ large enough,
$$
\sup_{(\theta,u_{s_t:t})\in V_\epsilon\times \Theta_{\gamma_t,s_t:t}}\E\Big[(1-\phi_t(Y_{1:t}))\prod_{s=s_t+1}^t(\tilde{f}_{\theta-\sum_{i=s}^{t-1}u_i}/f_{\theta_\star})(Y_s)\big|\,\F_{s_t}\Big]\leq e^{-(t-s_t) \tilde{D}_\star}.
$$
\end{lemma}
The proof of this result is given in Section    \ref{p-lemma:test}.

The next lemma builds  on  \citet[][Theorem 3.1]{Kleijn2012}

\begin{lemma}\label{lemma:part_1}
Assume Assumptions \ref{m_star}-\ref{new}. Let $\epsilon>0$, $(\gamma_t)_{t\geq 1}$ be a sequence in $[0,\infty)$ such that $\gamma_t\rightarrow 0$ and $(s_t)_{t\geq 1}$ be a sequence in $\mathbb{N}_0$ such that $\inf_{t\geq 1}(t-s_t)\geq 1$ and   $(t-s_t)\rightarrow\infty$. Then, there exist   constants $(C_1,C_2)\in\R_{>0}^2$ such that, for every constants $\delta\geq \tilde{\delta}>0$ such that $\delta+\tilde{\delta}<\delta_\star$ (with $\delta_\star>0$ as in Lemma \ref{lemma:denom}), there exists  a  sequence of measurable functions $(\phi'_t)_{t\geq 1}$, $\phi'_t:\setY^t\rightarrow\{0,1\}$, such that $\E[\phi'_t(Y_{1:t})]\rightarrow 0$ and such that, for $t$ large enough,
\begin{align*}
 \E\big[(1-\phi'_t(Y_{1:t}))\ind_{ \Theta_{\gamma_t,s_t:t}}&(U_{s_t:t}) \pi'_{s_t,t}(V_{\epsilon})|\F_{s_t}\big]\\
&\leq  \frac{   e^{-(t-s_t)( C_1^{-1}-C_2\delta^2)}}{\P(U_{s_t:t}\in \Theta_{\tilde{\delta},s_t:t})\, \inf_{v\in B_{\tilde{\delta}}(0)}(\mu_{s_t}*\tilde{\pi}_{s_t})(B_{\delta}(\theta_\star-v))},\quad\P-a.s.
\end{align*}
\end{lemma}
The proof of this result is given in Section \ref{p-lemma:part_1}.

\subsubsection{Proof of the theorem}

Theorem \ref{thm:online} is a direct consequence of Proposition \ref{prop:conv_proba} and of the following three lemmas.

\begin{lemma}\label{lemma:online1}
Consider the set-up of Theorem \ref{thm:online}. Then, $\E[\tilde{\pi}_{t_p}(V_{\epsilon })]\rightarrow 0$ for all $\epsilon>0$.
\end{lemma}
See Section \ref{p-lemma:online1} for the proof.

\begin{lemma}\label{lemma:online2}

Consider the set-up of Theorem \ref{thm:online} and assume that the conclusion of Lemma \ref{lemma:online1} holds. Let $(v_p)_{p\geq 1}$ be a sequence in $\mathbb{N}$ such that  $t_{p-1}\leq v_p<t_p$ for all $p\geq 1$ and such that $(v_p-t_{p-1})\rightarrow\infty$, and  let $(\tau_k)_{k\geq 1}$ be a strictly increasing sequence in $\mathbb{N}$ such that $(\tau_k)_{k\geq 1}=\{t\in\mathbb{N}:\,\exists p\geq 1,\, v_p\leq t<  t_p\}$. Then, $ \E[\tilde{\pi}_{\tau_k}(V_{\epsilon})]\rightarrow 0$ for all $\epsilon>0$.
\end{lemma}
See Section \ref{p-lemma:online2} for the proof.


\begin{lemma}\label{lemma:online3}

Consider the set-up of Theorem \ref{thm:online}  and assume that the conclusion  of Lemma \ref{lemma:online1}  holds.  Then, there exists a sequence $(v_p)_{p\geq 1}$ verifying the conditions of Lemma \ref{lemma:online2}     such that,  with $(\tau'_q)_{q\geq 1}$ the strictly increasing sequence in $\mathbb{N}$ verifying
$$
(\tau'_q)_{q\geq 1}=\{t\in\mathbb{N}:\,\exists p\geq 1,\, t_{p-1}<t\leq v_p\},
$$
we have $\E[\tilde{\pi}_{\tau'_q}(V_{\epsilon})]\rightarrow 0$ for all $\epsilon>0$.
\end{lemma}
See Section \ref{p-lemma:online3} for the proof.

\subsubsection{Proof of Lemma \ref{lemma:online1}\label{p-lemma:online1}}
\begin{proof}

Below $C\in(0,\infty)$ is a constant whose value can change from one expression to another.

We first remark that, by Doob's martingale inequality and under the assumptions on $(\mu_t)_{t\geq 0}$,
$$
\limsup_{p\rightarrow\infty}\P\big([U_{(t_{p-1}+1):t_p}]\geq \gamma\big)\leq C \gamma^{-2}\limsup_{p\rightarrow\infty} \sum_{s=t_{p-1}+1}^{t_{p}-1}h_s^2=0,\quad\forall \gamma>0
$$
showing that there exists a sequence $(\gamma_t)_{t\geq 1}$ in $\R_{>0}$ such that
\begin{align}\label{eq:gamma_t}
\gamma_t\rightarrow 0, \quad \P\big([U_{(t_{p-1}+1):t_p}]\geq \gamma_{t_p}\big)\rightarrow 0.
\end{align}

Let $(\gamma_t)_{t\geq 1}$ be as in \eqref{eq:gamma_t} and $(s_t)_{t\geq 1}$ be a sequence in $\mathbb{N}_0$ such that $\inf_{t\geq 1}(t-s_t)\geq 1$,  $(t-s_{t})\rightarrow\infty$ and    $s_{t_p}=t_{p-1}$ for every $p\geq 1$. Remark that such a sequence $(s_t)_{t\geq 1}$ exists under the assumptions of the lemma.

To proceed further let $(C_1,C_2)\in\R_{>0}^2$ be  as in Lemma \ref{lemma:part_1}. Without loss of generality we assume below that $2\sqrt{1/(2C_1 C_2)}<\delta_\star$, with $\delta_\star>0$ as in Lemma \ref{lemma:denom}. Let $\tilde{\delta}=\delta=\sqrt{1/(2C_1 C_2)} $,    $(\phi'_t)_{t\geq 1}$  be   as in    Lemma \ref{lemma:part_1} and, for every $t\geq 1$, let $\tilde{\phi}_{t}(Y_{1:t})$ be such that $\tilde{\phi}_t(Y_{1:t})=1$ whenever $\tilde{\pi}_{t} (V_{c/h^\beta_{t}})\geq \beta^{-1}$  and such that $\tilde{\phi}_t(Y_{1:t})=0$ otherwise,  with $(c,\beta)\in (1,\infty)$ as in \eqref{eq:compact}. Notice that $\E[\tilde{\phi}_{t_p}(Y_{1:t_p})]\rightarrow 0$ by \eqref{eq:compact} while  $\E[\phi'_{t_p}(Y_{1:t_p})]\rightarrow 0$ by Lemma \ref{lemma:part_1}.

Therefore, using  Lemma \ref{lemma:pi}, Tonellli's theorem and the shorthand $ \Theta_{t_p}= \Theta_{\gamma_{t_p},t_{p-1}:t_p}$,
\begin{equation}\label{eq:inquality1}
\begin{split}
 &\limsup_{p\rightarrow\infty}  \E[\tilde{\pi}_{t_p}(V_{\epsilon})]\\
& \leq \limsup_{p\rightarrow\infty}\E[(1-\phi'_{t_p}(Y_{1:t_p}))(1-\tilde{\phi}_{t_{p-1}}(Y_{1:t_{p-1}}))\tilde{\pi}_{t_p}(V_{\epsilon})]\\
&\leq \limsup_{p\rightarrow\infty}\E\big[(1-\phi'_{t_p}(Y_{1:t_p}))(1-\tilde{\phi}_{t_{p-1}}(Y_{1:t_{p-1}}))\ind_{ \Theta_{t_p}}(U_{t_{p-1}:t_p}) \pi'_{t_{p-1},t_p}(V_{\epsilon})\big]\\
&+\limsup_{p\rightarrow\infty}\E\Big[(1-\phi'_{t_p}(Y_{1:t_p}))(1-\tilde{\phi}_{t_{p-1}}(Y_{1:t_{p-1}}))\ind_{ \Theta^c_{t_p}}(U_{t_{p-1}:t_p}) \pi'_{t_{p-1},t_p}(V_{\epsilon_{t_p}}) \big]
\end{split}
\end{equation}
where, by Lemma \ref{lemma:part_1} and for $p$ large enough we have,  $\P$-a.s.,
\begin{equation}\label{eq:inquality_int1}
\begin{split}
\E\big[(1-\phi'_{t_p}&(Y_{1:t_p}))  \ind_{ \Theta_{t_p}}(U_{t_{p-1}:t_p}) \pi'_{t_{p-1},t_p}(V_{\epsilon_{t_p}})|\F_{t_{p-1}}\big]\\
&\leq   \frac{e^{-(t_p-t_{p-1}) C^{-1}}}{\P(U_{t_{p-1}:t_p}\in\Theta_{\delta,t_{p-1}:t_p})\, \inf_{v\in B_{\delta}(0)}(\mu_{t_{p-1}}*\tilde{\pi}_{t_{p-1}})(B_{\delta}(\theta_\star-v))}.
\end{split}
\end{equation}

To proceed further let  $v\in B_{\delta}(0)$, $f_{t_{p-1}}$ be the density of $\mu_{t_{p-1}}$ and remark that for all $p\geq 1$ we have (using Tonelli's theorem for the second  equality)
\begin{equation}\label{eq:lower_t}
\begin{split}
(\mu_{t_{p-1}}*\tilde{\pi}_{t_{p-1}})&(B_{\delta}(\theta_\star-v))\\
&=\int_{B_{\delta}(\theta_\star-v)} \int_{\Theta} f_{t_{p-1}}(\theta-u)\tilde{\pi}_{t_{p-1}}(\dd u) \dd\theta\\
&=\int_{\Theta} \int_{B_{\delta}(\theta_\star-v)}  f_{t_{p-1}}(\theta-u)\dd \theta \,\tilde{\pi}_{t_{p-1}}(\dd u)\\
&\geq \tilde{\pi}_{t_{p-1}}\big(B_{c h_{t_{p-1}}^{-\beta}}(\theta_\star)\big)\inf_{u\in B_{c h_{t_{p-1}}^{-\beta}}(\theta_\star)}\mu_{t_{p-1}}\big(B_{\delta}(\theta_\star-v+u)\big).
\end{split}
\end{equation}
Recall that $\mu_{t_{p-1}}(\dd u)$ is the $t_{d,\nu}(0, h_{t_{p-1}}^2\Sigma_{t_{p-1}})$. Then, under the assumptions on $(\Sigma_t)_{t\geq 0}$, and  using \eqref{eq:lower_t}, it is easily checked  that, with the shorthand  $\nu_1=\beta(\nu+d)+\nu$,
\begin{equation}\label{eq:lower_C}
 \P\Big( \inf_{v\in B_{\delta}(0)}(\mu_{t_{p-1}}*\tilde{\pi}_{t_{p-1}})(B_{\delta}(\theta_\star-v))\geq C^{-1}   h_{t_{p-1}}^{\nu_1} \big|\tilde{\phi}_{t_{p-1}}(Y_{1:t_{p-1}})=0 \Big)=1,\quad\forall p\geq 1.
\end{equation}
Consequently, using \eqref{eq:inquality_int1} and for $p$ large enough,  we have
\begin{equation}\label{eq:inquality2}
\begin{split}
\E\big[(1-\phi'_{t_p}(Y_{1:t_p}))(1- \tilde{\phi}_{t_{p-1}} (Y_{1:t_{p-1}}))\ind_{ \Theta_{t_p}}& (U_{t_{p-1}:t_p}) \pi'_{t_{p-1},t_p}(V_{\epsilon})\big]\\
&\leq C h_{t_{p-1}}^{-\nu_1}\frac{e^{-(t_p-t_{p-1}) C^{-1}}}{\P(U_{t_{p-1}:t_p}\in\Theta_{\delta,t_{p-1}:t_p})}.
\end{split}
\end{equation}

To proceed further remark that, by Lemma \ref{lemma:denom} and using \eqref{eq:lower_C},  we can without loss of generality assume that  $(\phi'_t)_{t\geq 1}$ is such that, for all $p\geq 1$,
\begin{equation}\label{eq:psi1}
\begin{split}
&\P\bigg(\int_{\Theta}\E\Big[ (\mu_{t_{p-1}}*\tilde{\pi}_{t_{p-1}}) \Big(\theta-\sum_{s=t_{p-1}+1}^{t_p-1}U_s\Big)\prod_{s=t_{p-1}+1}^{t_p} (\tilde{f}_{\theta-\sum_{i=s}^{t_{p}-1}U_i}/f_{\theta_\star})(Y_s)\big|\F_{t_p}\Big]  \dd\theta \\
 &>  C^{-1}   h_{t_{p-1}}^{\nu_1}\P(U_{t_{p-1}:t_p}\in\Theta_{\delta,t_{p-1}:t_p})\,e^{-(t_p-t_{p-1}) C \delta^2}\big|\phi'_{t_{p}}(Y_{1:t_{p}})\vee \tilde{\phi}_{t_{p-1}}(Y_{1:t_{p-1}})=0\bigg)\\
 &=1
\end{split}
\end{equation}
while, by  the law of large numbers, we can also without loss of generality assume that  $(\phi'_t)_{t\geq 1}$ is such that
\begin{equation}\label{eq:psi2}
\P\Big(-\frac{1}{t-s_t}\sum_{s=s_t+1}^t\log( f_{\theta_\star}(Y_s))\leq 1-\E[\log (f_{\theta_\star})]\,\,\big| \phi'_t(Y_{1:t})=0\Big)=1,\quad\forall t\geq 1.
\end{equation}

We now show that, under Assumption \ref{new1}  we have, for $p$ large enough,
\begin{equation}\label{eq:inquality3}
\begin{split}
 \E\big[  (1-\phi'_{t_p} (Y_{1:t_p}))(1-\tilde{\phi}_{t_{p-1}}&(Y_{1:t_{p-1}})) \ind_{ \Theta^c_{t_p}} (U_{t_{p-1}:t_p}) \pi'_{t_{p-1},t_p}(V_{\epsilon} )\big]\\
 &\leq  \frac{Ce^{(t_p-t_{p-1})C}}{\P(U_{t_{p-1}:t_p}\in\Theta_{\delta,t_{p-1}:t_p})}h_{t_{p-1}}^{-\nu_1}\P\big(U_{t_{p-1}:t_p}\not\in  \Theta_{t_p}\big).
\end{split}
\end{equation}

Assume first that  Assumption \ref{new1}.\ref{A51} holds. In this case,  by the law of large numbers, we can without loss of generality assume that  $(\phi'_t)_{t\geq 1}$ is such that
\begin{equation}\label{eq:psi3}
\P\Big(\frac{1}{t-s_t}\sum_{s=s_t+1}^t\sup_{\theta\in\Theta}\log (f_{\theta}(Y_s))\leq 1+\E[\sup_{\theta\in\Theta}\log (f_\theta)] \,\big| \phi'_t(Y_{1:t})=0\Big)=1,\quad\forall t\geq 1
\end{equation}
in which case \eqref{eq:inquality3} directly follows from \eqref{eq:psi1}, \eqref{eq:psi2} and \eqref{eq:psi3}.

Assume now that Assumption  \ref{new1}.\ref{A52}  holds.   Then,  \eqref{eq:inquality3} holds since, for all $p\geq 1$,
\begin{equation*}
\begin{split}
 \E\big[ &(1-\phi'_{t_p} (Y_{1:t_p}))(1-\tilde{\phi}_{t_{p-1}}(Y_{1:t_{p-1}})) \ind_{ \Theta^c_{t_p}} (U_{t_{p-1}:t_p}) \pi'_{t_{p-1},t_p}(V_{\epsilon} )\big]\\
 &\leq  \frac{Ce^{(t_p-t_{p-1})C}}{\P(U_{t_{p-1}:t_p}\in\Theta_{\delta,t_{p-1}:t_p})}h_{t_{p-1}}^{-\nu_1} \E\Big[\ind_{ \Theta^c_{t_p}} (U_{t_{p-1}:t_p})\\
 &\qquad \times \int_{V_{\epsilon}} (\mu_{t_{p-1}}*\tilde{\pi}_{t_{p-1}}) \Big(\theta-\sum_{s=t_{p-1}+1}^{t_p-1}U_s\Big)\prod_{s=t_{p-1}+1}^{t_p}\E\big[  \tilde{f}_{\theta-\sum_{i=s}^{t_{p}-1}U_i}/f_{\theta_\star})(Y_s)\big| U_{t_{p-1}:t_p}\big] \dd\theta \Big]\\
 &\leq  \frac{Ce^{(t_p-t_{p-1})C}}{\P(U_{t_{p-1}:t_p}\in\Theta_{\delta,t_{p-1}:t_p})}h_{t_{p-1}}^{-\nu_1} \big(\sup_{\theta\in\Theta}\E\big[f_\theta/f_{\theta_\star} \big]\big)^{t_p-t_{p-1}}\P\big(U_{t_{p-1}:t_p}\not\in  \Theta_{t_p}\big)\\
 &\leq  \frac{C e^{(t_p-t_{p-1})C}}{\P(U_{t_{p-1}:t_p}\in\Theta_{\delta, t_{p-1}:t_p})}h_{t_{p-1}}^{-\nu_1}\P\big(U_{t_{p-1}:t_p}\not\in  \Theta_{t_p}\big)
\end{split}
\end{equation*}
where the last inequality holds since  $\sup_{\theta\in\Theta}\E[f_\theta/f_{\theta_\star}]<\infty$ under Assumption \ref{new1}.\ref{A52}.

Lastly, assume that  Assumption \ref{new1}.\ref{A53} holds. Then,  \eqref{eq:inquality3} holds since, for all $p\geq 1$,
\begin{equation*}
\begin{split}
 \E\big[ &(1-\phi'_{t_p} (Y_{1:t_p}))(1-\tilde{\phi}_{t_{p-1}}(Y_{1:t_{p-1}})) \ind_{ \Theta^c_{t_p}} (U_{t_{p-1}:t_p}) \pi'_{t_{p-1},t_p}(V_{\epsilon} )\big]\\
 &\leq  \frac{Ce^{(t_p-t_{p-1})C}}{\P(U_{t_{p-1}:t_p}\in\Theta_{\delta, t_{p-1}:t_p})}h_{t_{p-1}}^{-\nu_1} \E\Big[\ind_{ \Theta^c_{t_p}} (U_{t_{p-1}:t_p})\\
 &\times \int_{V_{\epsilon}} (\mu_{t_{p-1}}*\tilde{\pi}_{t_{p-1}}) \Big(\theta-\sum_{s=t_{p-1}+1}^{t_p-1}U_s\Big)\prod_{s=t_{p-1}+1}^{t_p} \E\big[ \tilde{f}_{\theta-\sum_{i=s}^{t_{p}-1}U_i} (Y_s)\big| U_{t_{p-1}:t_p}\big]\dd\theta \Big]\\
 &\leq  \frac{C e^{(t_p-t_{p-1})C}}{\P(U_{t_{p-1}:t_p}\in\Theta_{\delta,t_{p-1}:t_p})}h_{t_{p-1}}^{-\nu_1} \big(\sup_{\theta\in\Theta}\E\big[f_\theta\big]\big)^{t_p-t_{p-1}}\P\big(U_{t_{p-1}:t_p}\not\in  \Theta_{t_p}\big)\\
 &\leq  \frac{C e^{(t_p-t_{p-1})C}}{\P(U_{t_{p-1}:t_p}\in\Theta_{\delta,t_{p-1}:t_p})}h_{t_{p-1}}^{-\nu_1}\P\big(U_{t_{p-1}:t_p}\not\in  \Theta_{t_p}\big)
\end{split}
\end{equation*}
where the last inequality holds since $\sup_{\theta\in\Theta}\E[f_\theta]<\infty$ under Assumption \ref{new1}.\ref{A53}. This concludes to show that \eqref{eq:inquality3} holds under Assumption \ref{new1}.

Using  \eqref{eq:inquality1}, \eqref{eq:inquality2} and \eqref{eq:inquality3}, to conclude the proof it is enough to show that, as $p\rightarrow\infty$,
\begin{align}
&h_{t_{p-1}}^{-\nu_1}e^{-(t_p-t_{p-1})C^{-1}}\rightarrow 0\label{eq:lim1}\\
&h_{t_{p-1}}^{-\nu_1}e^{(t_p-t_{p-1})C}\P\big(U_{t_{p-1}:t_p}\not\in  \Theta_{\gamma_{t_p},:t_{p-1},t_p}\big)\rightarrow 0\label{eq:lim2}\\
&\liminf_{p\rightarrow\infty}\P(U_{t_{p-1}:t_p}\in\Theta_{\delta,t_{p-1}:t_p})>0.\label{eq:lim3}
\end{align}

By assumptions, $\log(1/h_{t_{p-1}})(t_p-t_{p-1})^{-1}\rightarrow 0$ and therefore \eqref{eq:lim1}   holds.

To establish \eqref{eq:lim2} note that, using Doob's martingale inequality and the fact that $(\tilde{U} _s)_{s=t_{p-1}+1}^{t_p-1}$ are independent Gaussian random variables, we have
\begin{align}\label{eq:Doob_gauss}
\P\big([U_{(t_{p-1}+1):t_p}]\geq \gamma \big)\leq C  \exp\Big(-\frac{\gamma^2}{C \sum_{s=t_{p-1}+1}^{t_p-1}h_s^2}\Big),\quad\forall \gamma>0,\quad\forall p\geq 1
\end{align}
so that
\begin{align*}
\P\big(U_{t_{p-1}: t_p}\not\in  \Theta_{\gamma_{t_p}, t_{p-1}:t_p}\big)=\P\big([U_{(t_{p-1}+1):t_p}]\geq\gamma_{t_p}\big)\leq C \exp\Big(-\frac{\gamma_{t_p}^2}{C\sum_{s=t_{p-1}+1}^{t_p-1}h_s^2}\Big),\quad\forall p> 1.
\end{align*}
Consequently, for all $p> 1$,
\begin{align*}
\log\Big(&h_{t_{p-1}}^{-\nu_1}\, e^{(t_p-t_{p-1})C}\, \P\big(U_{t_{p-1}: t_p}\not\in  \Theta_{\gamma_{t_p}, t_{p-1}:t_p}\big)\Big)\\
&\leq -\nu_1 \log (h_{t_{p-1}})+C(t_p-t_{p-1})+\log(C)-C^{-1}\gamma_{t_p}^2\Big(\sum_{s=t_{p-1}+1}^{t_p-1}h_s^2\Big)^{-1}\\
&=-(t_p-t_{p-1})\bigg(-\frac{ \log(C)}{t_p-t_{p-1}}-C+\frac{C^{-1}\gamma_{t_p}^2-\nu_1\log(1/h_{t_{p-1}})\sum_{s=t_{p-1}+1}^{t_p-1}h_s^2}{(t_p-t_{p-1})\sum_{s=t_{p-1}+1}^{t_p-1}h_s^2}\bigg)
\end{align*}
where, under   the assumptions of the lemma and by taking $(\gamma_t)_{t\geq 1}$ such that $\gamma_{t_p}\rightarrow 0$ sufficiently slowly
$$
\bigg(-\frac{ \log(C)}{t_p-t_{p-1}}-C+\frac{C^{-1}\gamma_{t_p}^2-\nu_1\log(1/h_{t_{p-1}})\sum_{s=t_{p-1}+1}^{t_p-1}h_s^2}{(t_p-t_{p-1})\sum_{s=t_{p-1}+1}^{t_p-1}h_s^2}\bigg)\rightarrow\infty
$$
so that   \eqref{eq:lim2} holds.

Lastly, to show \eqref{eq:lim3} it suffices to remark that, since $\gamma_p\rightarrow 0$,
$$
\liminf_{p\rightarrow\infty}\P(U_{t_{p-1}:t_p}\in\Theta_{\delta, t_{p-1}:t_p})\geq\liminf_{p\rightarrow\infty}  \P(U_{t_{p-1}:t_p}\in\Theta_{\gamma_p, t_{p-1}:t_p})=1
$$
where the equality holds since \eqref{eq:lim2} holds. The proof is complete.

\end{proof}

\subsubsection{Proof of Lemma \ref{lemma:online2}}\label{p-lemma:online2}

\begin{proof}

Below $C\in(0,\infty)$ is a constant whose value can change from one expression to another.

We first remark that under the assumptions of the lemma  there exists a $p_1\in\mathbb{N}$ such that we have both $t_p-t_{p-1}>1$ and $v_p-t_{p-1}>0$ for all $p\geq p_1$. To  simplify the presentation of the proof we assume  without loss of generality that   $p_1=1$ in what follows.

For every $k\geq 1$  let $p_k=\sup\{p\geq 0:\, t_p<\tau_k\}$ and let $(s_t)_{t\geq 1}$ be a sequence in $\mathbb{N}_0$ such that $\inf_{t\geq 1}(t-s_t)\geq 1$, $(t-s_t)\rightarrow\infty$ and  $s_{\tau_k}=t_{p_k}$ for every $k\geq 1$. Remark that such a sequence $(s_t)_{t\geq 1}$ exists since $(\tau_k-s_{\tau_k})\geq 1$ for all $k\geq 1$ by construction while, by assumption,
$$
\liminf_{k\rightarrow\infty}(\tau_k-s_{\tau_k})\geq  \liminf_{k\rightarrow\infty} (v_{p_k+1}-t_{p_k})=\infty.
$$
Remark now that, by Doob's martingale inequality and under the assumptions on $(\mu_t)_{t\geq 0}$  we have, for every $\gamma>0$,
\begin{align*}
\limsup_{k\rightarrow\infty}\P\big([U_{(s_{\tau_k}+1):\tau_k}]\geq  \gamma\big)\leq C\limsup_{k\rightarrow\infty}\gamma^{-2} \sum_{s=t_{p_k}+1}^{\tau_k-1}h^2_s \leq C\limsup_{k\rightarrow\infty} \sum_{s=t_{p_k}+1}^{t_{p_{k}+1}-1}h^2_s =0
\end{align*}
showing that there exists a sequence $(\gamma_t)_{t\geq 1}$ in $\R_{>0}$ such that
\begin{align}\label{eq:gamma_t2}
 \gamma_t\rightarrow 0,\quad \P\big([U_{(s_{\tau_k}+1):\tau_k}]\geq  \gamma_{\tau_k}\big)\rightarrow 0.
\end{align}

Let $(\gamma_t)_{t\geq 1}$ be as in \eqref{eq:gamma_t2} and $(C_1,C_2)\in(0,\infty)^2$ be  as in Lemma \ref{lemma:part_1}. Without loss of generality we assume below that $2\sqrt{1/(2C_1 C_2)}<\delta_\star$, with $\delta_\star>0$ as in Lemma \ref{lemma:denom}.

Let $\kappa=\sqrt{1/(12C_1 C_2)}$, $\delta= 3\kappa$, $\tilde{\delta}=\kappa$ and remark that, for every $t\geq 1$ and $v \in B_{\tilde{\delta}}(0)$  we have, with $\theta_{s_t}:\Omega\rightarrow\Theta$  such that $\theta_{s_t}\sim\tilde{\pi}_{s_t}(\dd\theta)$ under $\P$,
\begin{equation}\label{eq:lower_1}
\begin{split}
\P\big(\|\theta_{s_t}+U_{s_t}-\theta_\star+v\|&\geq  \delta |\,\theta_{s_t}\in B_{\kappa}(\theta_\star)\big)\\
&\leq \P\big(\|\theta_{s_t}-\theta_\star\|+\|U_{s_t}\|+\|v\|\geq  \delta |\,\theta_{s_t}\in B_{\kappa}(\theta_\star)\big)\\
&\leq \P\big( U_{s_t}\not\in B_{ \delta-\tilde{\delta}-\kappa}(0)\big)\\
&\leq  \P\big( U_{s_t}\not\in B_{ \kappa}(0)\big).
\end{split}
\end{equation}
Then, since  $h_{s_{t-1}}\rightarrow 0$, it is easily checked   using \eqref{eq:lower_1}   that, under the assumptions on $(\mu_t)_{t\geq 0}$   and with $C^{\mu_{s_t}*\tilde{\pi}_{s_t}}_{\delta, \tilde{\delta}}$   as defined in Lemma \ref{lemma:denom}, for $t$ large enough we have
\begin{align}\label{eq:Ct}
C^{\mu_{s_t}*\tilde{\pi}_{s_t}}_{\delta, \tilde{\delta}}\geq \P\big( U_{s_t} \in B_{\kappa}(0)\big)\tilde{\pi}_{s_t}(B_{\kappa}(\theta_\star))\geq C^{-1} \tilde{\pi}_{s_t}(B_{\kappa}(\theta_\star)),\quad\P-a.s.\
\end{align}

Let  $(\phi'_t)_{t\geq 1}$ be as in Lemma \ref{lemma:part_1}, $\beta\in(0,1)$ and, for every $t\geq 1$,   $\tilde{\phi}_t(Y_{1:t})$ be such that $\tilde{\phi}_t(Y_{1:t})=1$ whenever $\tilde{\pi}_{t} (V_{\kappa})\geq \beta^{-1}$  and such that $\tilde{\phi}_t(Y_{1:t})=0$ otherwise. Notice that $\E[\phi'_{\tau_k}(Y_{1:\tau_k})]\rightarrow 0$ by Lemma \ref{lemma:part_1} while $\E[\tilde{\phi}_{s_{\tau_k}}(Y_{1:s_{\tau_k}})]\rightarrow 0$ by  Lemma \ref{lemma:online1}.

Therefore, using the shorthand $\Theta_{\tau_k}=\Theta_{\gamma_{\tau_k}, s_{\tau_k}:\tau_k}$,
\begin{equation}\label{eq:inquality11}
\begin{split}
 &\limsup_{k\rightarrow\infty}  \E[\tilde{\pi}_{\tau_k}(V_{\epsilon})]\\
 &\leq \limsup_{k\rightarrow\infty}\E\big[(1-\phi'_{\tau_k}(Y_{1:\tau_k}))(1-\tilde{\phi}_{s_{\tau_k}}(Y_{1:s_{\tau_k}}))\ind_{\Theta_{\tau_k}}(U_{s_{\tau_k}:\tau_k})  \pi'_{s_{\tau_k},\tau_k}(V_{\epsilon}) \big]\\
&+\limsup_{k\rightarrow\infty}\E\big[(1-\phi'_{\tau_k}(Y_{1:\tau_k}))(1-\tilde{\phi}_{s_{\tau_k}}(Y_{1:s_{\tau_k}}))\ind_{\Theta^c_{\tau_k}}(U_{s_{\tau_k}:\tau_k})  \pi'_{s_{\tau_k},\tau_k}(V_{\epsilon})\big].
\end{split}
\end{equation}

We now remark that, using Lemma \ref{lemma:online1}, \eqref{eq:Ct} and Lemma \ref{lemma:denom}, and because $s_{\tau_k}=t_{p_k}$ for all $k\geq 1$, we can without loss of generality  assume that $(\phi'_t)_{t\geq 1}$ is such that, for all $k\geq 1$,
\begin{equation}\label{eq:psi11}
\begin{split}
\P\Big(\int_{\Theta}\E\Big[&(\mu_{t_{p_k}}*\tilde{\pi}_{t_{p_k}}) \big(\theta-\sum_{s=s_{\tau_k}+1}^{\tau_k-1}U_s\Big)\prod_{s=s_{\tau_k}+1}^{\tau_k} (\tilde{f}_{\theta-\sum_{i=s}^{\tau_{k}-1}U_i}/f_{\theta_\star})(Y_s)\big|\F_{\tau_k}\Big]  \dd\theta \\
 &>  C^{-1}  \P(U_{s_{\tau_k}:\tau_k}\in\Theta_{\tilde{\delta}, s_{\tau_k}:\tau_k})\,e^{-C(\tau_k-s_{\tau_k})  \delta^2}\big| \phi'_{\tau_k}(Y_{1:\tau_k})=0\Big)=1.
\end{split}
\end{equation}

Then, following the computations in the proof of Lemma \ref{lemma:online1}, with \eqref{eq:psi11} used in place of \eqref{eq:psi1}, we obtain for $k$ large enough
\begin{equation}\label{eq:inquality22}
\begin{split}
\E\big[(1-\phi'_{\tau_k}(Y_{1:\tau_k}))(1-\tilde{\phi}_{s_{\tau_k}} (Y_{1:s_{\tau_k}}))&\ind_{\Theta_{\tau_k}}(U_{s_{\tau_k}:\tau_k})  \pi'_{s_{\tau_k},\tau_k} (V_{\epsilon})\big]\\
&\leq   \frac{C e^{-(\tau_k-s_{\tau_k})C^{-1}}}{\P(U_{s_{\tau_k}:\tau_k}\in\Theta_{\tilde{\delta}, s_{\tau_k}:\tau_k}) }
\end{split}
\end{equation}
and
\begin{equation}\label{eq:inquality33}
\begin{split}
 \E\big[ (1-\phi'_{\tau_k} (Y_{1:\tau_k}))(1-\tilde{\phi}_{s_{\tau_k}} (Y_{1:s_{\tau_k}})) &\ind_{\Theta^c_{\tau_k}} (U_{s_{\tau_k}:\tau_k})  \pi'_{s_{\tau_k},\tau_k}(V_{\epsilon} )\big]\\
&\leq \frac{C e^{(\tau_k-s_{\tau_k})C}}{\P(U_{s_{\tau_k}:\tau_k}\in\Theta_{\tilde{\delta},s_{\tau_k}:\tau_k})} \P\big(U_{s_{\tau_k}:\tau_k}\not\in \Theta_{\tau_k}\big).
\end{split}
\end{equation}

Therefore, using  \eqref{eq:inquality11}, \eqref{eq:inquality22} and \eqref{eq:inquality33}, to conclude the proof it is enough to show that
\begin{align}
&e^{-(\tau_k-s_{\tau_k}) C^{-1}}\rightarrow 0\label{eq:lim11}\\
&e^{(\tau_k-s_{\tau_k})C}\P\big(U_{s_{\tau_k}:\tau_k}\not\in \Theta_{\gamma_{\tau_k}, s_{\tau_k}:\tau_k}\big)\rightarrow 0\label{eq:lim22}.
\end{align}
Since $\liminf_{k\rightarrow\infty}(\tau_k-s_{\tau_k})\geq  \liminf_{k\rightarrow\infty}(v_{p_k+1}-t_{p_k})= \infty$ it follows that \eqref{eq:lim11} holds. 

To show \eqref{eq:lim22} remark that
\begin{align*}
e^{(\tau_k-s_{\tau_k})C}\P\big(U_{s_{\tau_k}:\tau_k}\not\in  \Theta_{\gamma_{\tau_k}, s_{\tau_k}:\tau_k}\big)&=e^{(\tau_k-s_{\tau_k})C}\P\big([U_{(s_{\tau_k}+1):\tau_k}]\geq \gamma_{\tau_k} \big)\\
&\leq C e^{(t_{p_k+1}-t_{p_k})C}\exp\Big(-\frac{\gamma_{\tau_k}}{C\sum_{s=t_{p_k}+1}^{\tau_k-1}h_s^2}\Big)\\
&\leq C e^{(t_{p_k+1}-t_{p_k})C}\exp\Big(-\frac{\gamma_{\tau_k}}{C\sum_{s=t_{p_k}+1}^{t_{p_k+1}-1}h_s^2}\Big)
\end{align*}
where the first inequality uses \eqref{eq:Doob_gauss}. As shown in the proof of Lemma \ref{lemma:online1},  the term on the r.h.s.\ of the last inequality sign converges to 0 as $k\rightarrow\infty$ when $\gamma_{\tau_k}\rightarrow 0$ sufficiently slowly, and thus \eqref{eq:lim22} holds. The proof is complete.

\end{proof}

\subsubsection{Proof of  Lemma \ref{lemma:online3}\label{p-lemma:online3}}

Below $C\in(0,\infty)$ is a constant whose value can change from one expression to another.

For every $p\geq 1$  let
\begin{align*}
&\xi_p=1\wedge\Bigg(\frac{\log(h_{t_p}^{-\nu})^{1/2}\wedge (t_p-t_{p-1})^{1/2}}{\log(h_{t_p}^{-\nu})}\Bigg)\\
&v_p=t_{p-1}+ \lfloor   \log(h_{t_{p-1}}^{-\nu\xi_{p-1}})\rfloor \wedge (t_{p}-t_{p-1}-1)
\end{align*}
so that
\begin{align}\label{eq:xi_p}
 \xi_p\rightarrow 0,\quad \log(h_{t_{p-1}}^{-\nu\xi_{p-1}})\rightarrow\infty,\quad \frac{ \log(h_{t_{p-1}}^{-\nu\xi_{p-1}})}{t_p-t_{p-1}}\rightarrow 0
\end{align}
while $(v_p)_{p\geq 1}$ verifies the conditions of Lemma \ref{lemma:online2}.

 For every $q\geq 1$  let $p_q=\sup\{p\geq 0:\, t_p<\tau'_q\}$ and note that, using \eqref{eq:xi_p},

\begin{equation}\label{eq:st}
\begin{split}
\liminf_{q\rightarrow\infty}(\tau'_q-v_{p_q})&\geq \liminf_{q\rightarrow\infty}(t_{p_q}-v_{p_q})\\
&= \liminf_{p\rightarrow\infty}(t_p-v_p)\\
&\geq  \liminf_{p\rightarrow\infty} (t_p-t_{p-1})\Big(1-\frac{\log(h^{-\nu\xi_{p-1}}_{t_{p-1}})}{t_p-t_{p-1}}\Big)\\
&=\infty.
\end{split}
\end{equation}

Note also that, under the assumptions of the lemma, there exists a $p_1\in\mathbb{N}$ such that
\begin{equation}\label{eq:st1}
   \log(h_{t_{p-1}}^{-\nu})^{1/2}>1,\quad\forall p\geq p_1.
\end{equation}
To  simplify the presentation of the proof we assume without  loss of generality that   $p_1=1$ in what follows.

We now let $(s_t)_{t\geq 0}$  be a sequence in $\mathbb{N}_0$ such that $\inf_{t\geq 1}(t-s_t)\geq 1$, $(t-s_t)\rightarrow\infty$ and
$$
s_{\tau'_q}=\big(\tau'_q- \lfloor \log(h_{t_{p_q}}^{-\nu})^{1/2}\rfloor\big)\vee v_{p_q},\quad \forall q\geq 1.
$$
Notice that such a sequence   $(s_t)_{t\geq 0}$ exists by \eqref{eq:st} and because we are assuming that \eqref{eq:st1} holds with $p_1=1$. Note also that $(s_t)_{t\geq 1}$ is such that $v_{p_q}\leq s_{\tau'_q}\leq t_{p_q}$ for all $q\geq 1$.

We now show that
\begin{align}\label{eq:gamma_cond}
 \P\big([U_{(s_{\tau'_q}+1):\tau'_q}]\geq  \gamma\big)\rightarrow 0,\quad\forall \gamma>0.
\end{align}
To this aim let $\gamma>0$ and remark that 
\begin{equation}\label{eq:split_prob}
\begin{split}
\P\big([U_{(s_{\tau'_q}+1):\tau'_q}] \geq  \gamma\big)&\leq \P\big(\max_{s_{\tau'_q}<s< \tau'_q}\big\|\sum_{i=s}^{\tau'_q-1} U_i-U_{t_{p_q}}\ind(s\leq t_{p_q})\|\geq  \gamma/2\big)\\
&+\P\big( \|U_{t_{p_q}}\|\geq  \gamma/2\big).
\end{split}
\end{equation}
where, using the fact that 
\begin{align}\label{eq:bound_tail_stud}
\int_a^\infty \Big(1+\frac{x^2}{\nu}\Big)^{-\frac{\nu+1}{2}}\dd x\leq \nu^{\frac{\nu+1}{2}}\int_a^\infty x^{-(\nu+1)}\dd x=\nu^{\frac{\nu-1}{2}} a^{-\nu}\quad\forall a>0, 
\end{align}
we have $\P\big( \|U_{t_{p_q}}\|\geq  \gamma/2\big)\rightarrow 0$ since $h_t\rightarrow 0$. In addition, by Doob's martingale inequality and under the assumptions on $(\mu_t)_{t\geq 0}$,
\begin{align*}
&\limsup_{q\rightarrow\infty}\P\Big(\max_{s_{\tau'_q}<s<\tau'_q}\big\|\sum_{i=s}^{\tau'_q-1} U_i-U_{t_{p_q}}\ind(s\leq t_{p_q})\|\geq  \frac{\gamma}{2}\Big)\leq C\limsup_{q\rightarrow\infty}\frac{\sum_{s=t_{p_q-1}+1}^{t_{p_{q+1}}-1}h_s^2}{(\gamma/2)^2}=0.
\end{align*}
Hence, \eqref{eq:gamma_cond} holds showing that there exists a  sequence $(\gamma_t)_{t\geq 1}$ in $\R_{>0}$ such that
\begin{align}\label{eq:gamma_t22}
\gamma_t\rightarrow 0,\quad  \P\big([U_{(s_{\tau'_q}+1):\tau'_q}]\geq  \gamma_{\tau'_q}\big)\rightarrow 0.
\end{align}

Let $(\gamma_t)_{t\geq 1}$ be as in \eqref{eq:gamma_t22}, $(\tau_k)_{k\geq 1}$ be as defined in Lemma \ref{lemma:online2}, $(\tilde{\tau}_r)_{r\geq 1}$ be a strictly increasing sequence in $\mathbb{N}_0$ such that $(\tilde{\tau}_r)_{r\geq 1}= (\tau_k)_{k\geq 1}\cup(t_p)_{p\geq 0}$ and note that, by Lemmas \ref{lemma:online1}-\ref{lemma:online2} ,
\begin{align}\label{eq:conv_sub2}
\E[\tilde{\pi}_{\tilde{\tau}_r}(V_{\kappa})]\rightarrow 0,\quad\forall\kappa>0.
\end{align}

Remark also that, by construction,  $(s_{\tau'_q})_{q\geq 1}\subset (\tilde{\tau}_r)_{r\geq 1}$ so that we can   now follow the computations in the proof of Lemma \ref{lemma:online2}. As in this latter let  $\tilde{\delta}=\kappa$ for some sufficiently small $\kappa>0$ (see the proof of Lemma \ref{lemma:online2} for the expression of $\kappa$). Then,  as shown  in the proof of Lemma \ref{lemma:online2} we have,
\begin{align*}
\limsup_{q\rightarrow\infty}\E[\tilde{\pi}_{\tau'_q}(V_{\epsilon})]&\leq \limsup_{q\rightarrow\infty}     \frac{C\,e^{-(\tau'_q-s_{\tau'_q}) C^{-1}}}{\P(U_{s_{\tau'_q}:\tau'_q}\in\Theta_{\tilde{\delta}, s_{\tau'_q}:\tau'_q}) }\\
&+\limsup_{q\rightarrow\infty}  \frac{C\, e^{(\tau'_q-s_{\tau'_q})C}}{\P(U_{s_{\tau'_q}:\tau'_q}\in\Theta_{\tilde{\delta},s_{\tau'_q}:\tau'_q})} \P\big(U_{s_{\tau'_q}:\tau'_q}\not\in \Theta_{\gamma_{\tau'_q},s_{\tau'_q}:\tau'_q}\big)
\end{align*}
so that to conclude the proof it is enough   to show that
\begin{align}
&e^{(\tau'_q-s_{\tau'_q})C^{-1}}\rightarrow 0\label{eq:lim111}\\
&e^{(\tau'_q-s_{\tau'_q})C}\P\big(U_{(s_{\tau'_q}+1):\tau'_q}\not\in \Theta^{\gamma_{\tau'_q}}_{s_{\tau'_q},\tau'_q}\big)\rightarrow 0\label{eq:lim222}.
\end{align}
Since $(\tau'_q-s_{\tau'_q})\rightarrow \infty$ it follows that \eqref{eq:lim111}  holds.  

To show \eqref{eq:lim222} remark  that
\begin{equation}\label{eq:decomp_final}
\begin{split}
e^{(\tau'_q-s_{\tau'_q})C}\P\big(U_{s_{\tau'_q}:\tau'_q}&\not\in \Theta_{\gamma_{\tau'_q},s_{\tau'_q}:\tau'_q}\big)\\
&=e^{(\tau'_q-s_{\tau'_q})C}\P\big([U_{(s_{\tau'_q}+1):\tau'_q}]\geq \gamma_{\tau'_q}\big) \\
&\leq  e^{(\tau'_q-s_{\tau'_q})C} \P\big(\max_{s_{\tau'_q}<s< \tau'_q}\big\|\sum_{i=s}^{\tau'_q-1} U_i-U_{t_{p_q}}\ind(s\leq t_{p_q})\|\geq  \gamma_{\tau'_q}/2\big)\\
&+e^{(\tau'_q-s_{\tau'_q})C}\P\big( \|U_{t_{p_q}}\|\geq  \gamma_{\tau'_q}/2\big).
\end{split}
\end{equation}
In addiiton,
\begin{align}\label{eq:in_st}
\tau'_q-s_{\tau'_q}\leq  \log(h_{t_{p_q}}^{-\nu\xi_{p_q}}) \leq 2(t_{p_q+1}-t_{p_{q}}),
\end{align}
where the second inequality holds  by \eqref{eq:xi_p} and for $q$ large enough.

Therefore, using \eqref{eq:Doob_gauss},
\begin{equation}\label{eq:decomp_final1}
\begin{split}
\limsup_{q\rightarrow\infty} &\,e^{(\tau'_q-s_{\tau'_q})C}\P\big(\max_{s_{\tau'_q}<s< \tau'_q}\big\|\sum_{i=s}^{\tau'_q-1} U_i-U_{t_{p_q}}\ind(s\leq t_{p_q})\|\geq  \gamma_{\tau'_q}/2\big)\\
&\leq C \limsup_{q\rightarrow\infty}  e^{(t_{p_q+1}-t_{p_{q}})C}\exp\bigg(-\frac{\gamma_{\tau'_q}}{C\sum_{s=t_{p_q-1}+1}^{t_{p_q+1}-1}h_s^2}\bigg)\\
&=0
\end{split}
\end{equation}
where the  equality holds   assuming without loss of generality that $\gamma_{\tau'_q}\rightarrow 0$ sufficiently slowly and uses the fact that, by assumption,
$$
\limsup_{p\rightarrow\infty}\frac{t_{p+2}-t_{p+1}}{ t_{p+1}-t_{p}}<\infty.
$$
Using \eqref{eq:bound_tail_stud} we have
\begin{align*}
\P\big( \|U_{t_{p}}\|\geq  \gamma\big) \leq   C\Big(\frac{h_{t_p}}{ \gamma}\Big)^\nu,\quad\forall p\geq 1,\quad\forall \gamma>0.
\end{align*}
 Remark also that, using the first inequality in \eqref{eq:in_st} and recalling that $\xi_p\rightarrow 0$,
 \begin{align*}
\limsup_{p\rightarrow\infty}  h_{t_{p_q}}^\nu e^{(\tau'_p-s_{\tau'_p})C}&\leq\limsup_{p\rightarrow\infty}  h_{t_{p_q}}^{\nu(1-C \xi_{p_q})}=0
 \end{align*}
and therefore, assuming without loss of generality that $\gamma_{\tau'_q}\rightarrow 0$ sufficiently slowly,
 $$
 \limsup_{p\rightarrow\infty} e^{(\tau'_q-s_{\tau'_q})C} \P\big( \|U_{t_{p_q}}\|\geq   \gamma_{\tau'_q}/2\big)=0.
 $$
 Together with \eqref{eq:decomp_final} and \eqref{eq:decomp_final1} this last result shows that \eqref{eq:lim222} holds. The proof is complete.

\subsection{Proof of Lemma \ref{lem:lower_bound_1}\label{p-lem:lower_bound_1}}

\subsubsection{Preliminary results}

We first show the following simple result.

\begin{lemma}\label{lemma:simple}
Let $(t_p)_{p\geq 0}$ be a subsequence of $\mathbb{N}_0$ and $g:\setY\rightarrow\R$ be a measurable function such that $\E[|g(Y_1)]|<\infty$. Then, as $p\rightarrow\infty$,
$$
\P\Big(\max_{0<i\leq p}\frac{1}{t_i-t_{i-1}}\Big|\sum_{s=t_{i-1}+1}^{t_i}\big(g(Y_s)-\E[g(Y_1)]\big)\Big|\geq  t^\delta_{p}\Big)\rightarrow 0,\quad\forall \delta>1.
$$

\end{lemma}
\begin{proof}
Let $\delta>1$ so that, using Markov's inequality for the second inequality,
\begin{align*}
\limsup_{p\rightarrow\infty}\P\Big(&\max_{0< i\leq p}\frac{1}{t_i-t_{i-1}}\Big|\sum_{s=t_{i-1}+1}^{t_i}\big(g(Y_s)-\E[g(Y_1)]\big)\Big|\geq  t^\delta_{p}\Big)\\
&\leq \limsup_{p\rightarrow\infty}\sum_{i=1}^p\P\Big(  \frac{1}{t_i-t_{i-1}}\Big|\sum_{s=t_{i-1}+1}^{t_i}\big(g(Y_s)-\E[g(Y_1)]\big)\Big|\geq  t^\delta_{p}\Big)\\
&\leq \limsup_{p\rightarrow\infty}\frac{2\E[|g(Y_1)|]}{t^{\delta-1}_{p}}\\
&=0.
\end{align*}
\end{proof}

We also recall the following result \citep[][Theorem 1.2]{ferger2014}.
\begin{lemma}\label{lemma:mean_p}
Let  $(X_i)_{i\geq 1}$ be a sequence of i.i,d.\ random variables such that $\E[X_1]=0$ and $\E[X_1^{2p}]<\infty$ for some $p\in\mathbb{N}$. Then $\E\big[(\sum_{i=1}^n X_i)^{2p}\big]=\bigO(n^p)$.
\end{lemma}

\subsubsection{Proof of the lemma}
\begin{proof}

Let $\beta>\max( \beta_\star,1/\alpha)$ and $D_p= h_{t_{p}}^{-\beta}$ for all $p\geq 1$.

We first establish the result of the theorem under Assumption \ref{new2}.\ref{A61}, and start by showing that
\begin{align}\label{eq:conv_max}
 \P\Big(\max_{0< i\leq p} e^{-(t_i-t_{i-1})\zeta(D_p)/2}\prod_{s=t_{i-1}+1}^{t_i}\sup_{\theta_s\in V_{D_p}}(\tilde{f}_{\theta_{s}}/f_{\theta_\star})(Y_s)< 1\Big)\rightarrow 1.
\end{align}

Let $p_1\in\mathbb{N}$ be such that  $\zeta(D_p)<0$ for all $p\geq p_1$; notice that such a $p_1$ exists under Assumption \ref{new2}.\ref{A61}. Then, for every $p\geq p_1$ we have
\begin{equation}\label{eq:conv_max_1}
\begin{split}
\P\Big(& \max_{0< i\leq p} e^{-(t_i-t_{i-1})\zeta(D_p)/2}\prod_{s=t_{i-1}+1}^{t_i}\sup_{\theta_s\in V_{D_p}}(\tilde{f}_{\theta_{s}}/f_{\theta_\star})(Y_s)\geq  1\Big)\\
&=\P\Big(\max_{0< i\leq p} \sum_{s=t_{i-1}+1}^{t_i}\Big\{\sup_{\theta_s\in V_{D_p}}\log \Big(\frac{\tilde{f}_{\theta_{s}}}{f_{\theta_\star}}(Y_s)\Big)-\zeta(D_p)+\frac{\zeta(D_p)}{2}\Big\}\geq  0\Big)\\
&\leq \sum_{i=1}^p\P\Big( \sum_{s=t_{i-1}+1}^{t_i}\Big(\sup_{\theta_s\in V_{D_p}}\log \Big(\frac{\tilde{f}_{\theta_{s}}}{f_{\theta_\star}}(Y_s)\Big)- \zeta(D_p )\Big)> (t_i-t_{i-1})|\zeta(D_p)|/2\Big)\\
&=\sum_{i=1}^p\P\Big(\frac{1}{t_i-t_{i-1}} \sum_{s=t_{i-1}+1}^{t_i}\Big(\sup_{\theta_s\in V_{D_p}}\log \Big(\frac{\tilde{f}_{\theta_{s}}}{f_{\theta_\star}}(Y_s)\Big)- \zeta(D_p )\Big)> |\zeta(D_p)|/2\Big).
\end{split}
\end{equation}
Using Markov's inequality,  Assumption  \ref{new2}.\ref{A61} and Lemma \ref{lemma:mean_p},   there exists a  constant $C\in(0,\infty)$ such that, for all $i\in 1:p$,
\begin{equation}\label{eq:conv_max_2}
\begin{split}
\P\Big(&\frac{1}{t_i-t_{i-1}} \sum_{s=t_{i-1}+1}^{t_i}\Big(\sup_{\theta_s\in V_{D_p}}\log \Big(\frac{\tilde{f}_{\theta_{s}}}{f_{\theta_\star}}(Y_s)\Big)- \zeta(D_p )\Big)> |\zeta(D_p)|/2\Big)\\
&\leq 2^{2k_\star}\frac{\E\Big[\Big|\frac{1}{t_i-t_{i-1}} \sum_{s=t_{i-1}+1}^{t_i}\big(\sup_{\theta_s\in V_{D_p}}\log\big( (\tilde{f}_{\theta_{s}}/f_{\theta_\star})(Y_s)\big)- \zeta(D_p) \big)\Big|^{2k_\star}\Big]}{ \zeta(D_p)^{2 k_\star}}\\
&\leq  \frac{C }{ (t_i-t_{i-1})^{k_\star\ind_{\mathbb{N}}(k_\star)} |\zeta(D_p)|^{2k_\star}}.
\end{split}
\end{equation}
By assumptions, $|\zeta(D_p)|^{-2 k_\star}\sum_{i=1}^p (t_i-t_{i-1})^{-k_\star\ind_{\mathbb{N}}(k_\star)}\rightarrow 0$ which, together with  \eqref{eq:conv_max_1} and \eqref{eq:conv_max_2}, implies \eqref{eq:conv_max}.

To proceed further let $\tilde{C}_{\star}\in(0,\infty)$ be as in Lemma \ref{lemma:denom},  $\delta_p=t_p^{-\gamma}$ for some $\gamma>\max(\alpha,1/2)$, $\delta>0$ and, for all $i\geq 1$ and $s\in t_{i-1}:(t_i-1)$, let $V_{i,s}=\sum_{j=s}^{t_i-1}U_j$.

Under Assumptions \ref{m_star}-\ref{taylor} and   by Lemma \ref{lemma:denom},  for $p$ large enough (i.e.\ for $\delta_p$ small enough) and  all $i\in 1:p$, we have
\begin{equation*}
\begin{split}
g_{i}(Y_{1:t_i})&:=\int_{\Theta} \E\big[ (\mu_{t_{i-1}}*\tilde{\pi}_{t_{i-1}})(\theta-V_{i,t_{i-1}+1})\prod_{s=t_{i-1}+1}^{t_i}(\tilde{f}_{\theta-V_{i,s}}/f_{\theta_\star})(Y_s)\big|\F_{t_i}\big]\dd\theta\\
&\leq    \frac{\P( U_{ t_{i-1}:t_i} \in\Theta_{\delta_p, t_{i-1}:t_i})}{  e^{(t_i-t_{i-1})(2(\delta_p\tilde{C_\star})^2+\delta)}} \inf_{v\in B_{\delta_p}(0)}(\mu_{t_{i-1}}*\tilde{\pi}_{t_{i-1}})(B_{\delta_p}(\theta_\star-v))\\
&=:\tilde{g}_{i,p}(Y_{1:t_i})
\end{split}
\end{equation*}
with probability at most $p_{i,p}:=\big( (t_i-t_{i-1})\big\{ (\tilde{C}_{\star} \delta_p) +(\tilde{C}_{\star} \delta_p)^{-1} \delta \big\}^{2}\big)^{-1}$.

Then, noting that $\sum_{i=1}^p   (t_i-t_{i-1})^{-1}\leq t_p$, it follows that
\begin{equation}\label{eq:conv_den2}
\begin{split}
\limsup_{p\rightarrow\infty}\P\Big(\min_{1\leq i\leq p}\frac{g_{i}(Y_{1:t_i})}{\tilde{g}_{i,p}(Y_{1:t_i})} \leq 1\Big)& \leq \limsup_{p\rightarrow\infty} \sum_{i=1}^p p_{i,p}\\
&\leq \limsup_{p\rightarrow\infty} \frac{\delta_p^{2}\, \tilde{C}_{\star}^2}{\delta^2}\sum_{i=1}^p (t_i-t_{i-1})^{-1}\\
& \leq \frac{ \tilde{C}_{\star}^2}{\delta^2} \limsup_{p\rightarrow\infty} t_p^{-2\gamma} t_p\\
&=0
\end{split}
\end{equation}
where the last equality holds since $\gamma>1/2$.

Therefore, by \eqref{eq:conv_max}  and  \eqref{eq:conv_den2},   there exists a sequence $(\setY_p)_{p\geq 1}$, with $\setY_p\subset\sigY^{\otimes t_p}$ for all $p\geq 1$,  such that  $\P(Y_{1:t_p}\in\setY_p)\rightarrow 1$ and such that,  for all $p\geq 1$,
\begin{align*}
 \max_{0< i\leq p} e^{-(t_i-t_{i-1})\zeta(D_p)/2}\prod_{s=t_{i-1}+1}^{t_i}\sup_{\theta_s\in V_{D_p}}\frac{\tilde{f}_{\theta_{s}}}{f_{\theta_\star}}(y_s)< 1,\quad \min_{1\leq i\leq p}\frac{g_{i}(y_{1:t_i})}{\tilde{g}_{i,p}(y_{1:t_i})}>1,\quad \forall y_{1:t_p}\in\setY_p.
\end{align*}

To proceed further let  $\epsilon\in(0,1)$ and note that, since $\P(Y_{1:t_p}\in\setY_p)\rightarrow 1$,
\begin{align}\label{eq:res_1}
\limsup_{p\rightarrow\infty}\P(\tilde{\pi}_{t_p}(V_{2D_p})\geq \epsilon)\leq  \limsup_{p\rightarrow\infty}\P\big(\{\tilde{\pi}_{t_p}(V_{2D_p})\geq \epsilon\}\cap \{Y_{1:t_p}\in\setY_p\}\big)
\end{align}
so that to prove the   theorem under Assumption \ref{new2}.\ref{A61} it remains to show that
$$
 \limsup_{p\rightarrow\infty}\P\big(\{\tilde{\pi}_{t_p}(V_{2D_p})\geq \epsilon\}\cap \{Y_{1:t_p}\in\setY_p\}\big)=0.
$$

To this aim let $p_0\in\mathbb{N}$ and, for every $p>p_0$, let $(\setY_{i,p})_{i=0}^p$ be a sequence in $\sigY^{\otimes t_p}$ such that $\setY_{0,p}=\setY^{t_p}$,  $\setY_{p,p}=\setY_p$ and such that, for every $i\in 1:p$,  $\setY_{i,p}\subset \cap_{j=0}^{i-1}\setY_{j,p}$ and
\begin{align*}
 \max_{0< j\leq i} e^{-(t_j-t_{j-1})\zeta(D_p)/2}\prod_{s=t_{j-1}+1}^{t_j}\sup_{\theta_s\in V_{D_p}}\frac{\tilde{f}_{\theta_{s}}}{f_{\theta_\star}}(y_s)< 1,\quad \min_{1\leq j\leq i}\frac{g_{j}(y_{1:t_j})}{\tilde{g}_{j,p}(y_{1:t_j})}>1,\quad \forall y_{1:t_p}\in\setY_{i,p}.
\end{align*}
Let $A_{i,p}=\{\tilde{\pi}_{t_i}(V_{2D_p})\geq \epsilon\}\cap\{Y_{1:t_p}\in Y_{i,p}\}$ and $\tilde{A}_{i,p}=\{\tilde{\pi}_{t_i}(V_{2D_p})\ind_{\setY_{i,p}}(Y_{1:t_i})\geq \epsilon\}$ for all $p>p_0$ and all $i\in 1:p$.

Then, for every $p>p_0$ we have
\begin{equation}\label{eq:p_split1}
\begin{split}
\P\big(\{\tilde{\pi}_{t_p}(V_{2D_p})\geq \epsilon\} \cap   \{Y_{1:t_p}\in\setY_p\}\big)&=\P(A_{p,p})\\
&\leq \P(A_{p,p}| A^c_{p-1,p})+\P(A_{p-1,p})\\
&\leq \sum_{i=p_0+1}^p \P(A_{i,p}| A^c_{i-1,p})+\P(A_{p_0,p})\\
&=  \sum_{i=p_0+1}^p \P(\tilde{A}_{i,p}| A^c_{i-1,p})+\P(A_{p_0,p})
\end{split}
\end{equation}
and we now study $\P(\tilde{A}_{i,p}| A^c_{i-1,p})$ for all  $i\in (p_0+1):p$.

Let $i\in (p_0+1):p$ and
\begin{align*}
&X_{i,p}^{(1)}= \int_{V_{2D_p}} \E\Big[\ind([U_{(t_{i-1}+1):t_i}]<D_p) (\mu_{t_{i-1}}*\tilde{\pi}_{t_{i-1}})(\theta-V_{i,t_{i-1}+1})\prod_{s=t_{i-1}+1}^{t_i} \frac{\tilde{f}_{\theta-V_{i,s}}}{f_{\theta_\star}}(Y_s)\big|\F_{t_i}\Big]\dd\theta\\
&X_{i,p}^{(2)}= \int_{V_{2D_p}} \E\Big[\ind([U_{(t_{i-1}+1):t_i}]\geq D_p) (\mu_{t_{i-1}}*\tilde{\pi}_{t_{i-1}})(\theta-V_{i,t_{i-1}+1})\prod_{s=t_{i-1}+1}^{t_i} \frac{\tilde{f}_{\theta-V_{i,s}}}{f_{\theta_\star}}(Y_s)\big|\F_{t_i}\Big]\dd\theta
\end{align*}
so that, by Lemma \ref{lemma:pi},
$$
\tilde{\pi}_{t_i}(V_{2D_p})=\frac{X_{i,p}^{(1)}+X_{i,p}^{(2)}}{g_i(Y_{1:t_i})},\quad\P-a.s.
$$
Then, by Markov's inequality and using the definition of $\setY_{i,p}$, we have
\begin{align*}
\P(\tilde{A}_{i,p}| A^c_{i-1,p})&\leq \epsilon^{-1}\E\bigg[\frac{(X_{i,p}^{(1)}+X_{i,p}^{(2)})\ind_{\setY_{i,p}}(Y_{1:t_p})}{g_i(Y_{1:t_i})}\big|A^c_{i-1,p}\bigg]\\
& \leq\epsilon^{-1}\E\bigg[\frac{(X_{i,p}^{(1)}+X_{i,p}^{(2)})\ind_{\setY_{i,p}}(Y_{1:t_p})}{\tilde{g}_{i,p}(Y_{1:t_i})}\big|A^c_{i-1,p}\bigg]\\
&=\epsilon^{-1}\E\bigg[\frac{(X_{i,p}^{(1)}+X_{i,p}^{(2)})\ind_{\setY_{i,p}}(Y_{1:t_p})}{\tilde{g}_{i,p}(Y_{1:t_i})}\big|\, \tilde{\pi}_{t_{i-1}}(V_{2D_p})<\epsilon\bigg]
\end{align*}
where the equality uses the definition of $ A_{i-1,p}$ and the fact that $Y_{i,p}\subset Y_{i-1,p}$.

Following similar computations  as   in the proof of Lemma \ref{lemma:online1} (see \eqref{eq:lower_C}), it is readily checked  that there exists a constant $C_1\in(0,\infty)$ such that
\begin{align*}
\P\Big(\inf_{v\in B_{\delta_p}(0)}(\mu_{t_{i-1}}*\tilde{\pi}_{t_{i-1}})(B_{\delta_p}(\theta_\star-v))  \geq C_1^{-1} \delta_p^d h_{t_{i-1}}^\nu h_{t_{p-1}}^{\beta(\nu+d)}\big|\, \tilde{\pi}_{t_{i-1}}(V_{2D_p})<\epsilon\Big)=1
\end{align*}
and thus
\begin{equation}\label{eq:p_split2}
\begin{split}
 \P(\tilde{A}_{i,p}| A^c_{i-1,p})&\leq \frac{C_1  \delta_p^{-d} h_{t_{i-1}}^{-\nu} h_{t_{p-1}}^{-\beta(\nu+d)}e^{(t_i-t_{i-1})(2(\delta_p\tilde{C_\star})^2+\delta)}}{\epsilon\,\P( U_{ t_{i-1}:t_i} \in\Theta_{\delta_p, t_{i-1}:t_i})}\\
&\qquad\times \E\big[ (X_{i,p}^{(1)}+X_{i,p}^{(2)}) \ind_{\setY_{i,p}}(Y_{1:t_i})|\, \tilde{\pi}_{t_{i-1}}(V_{2D_p})<\epsilon\big].
\end{split}
\end{equation}
Next, noting that $X_{i,p}^{(1)}\leq \prod_{s=t_{i-1}+1}^{t_i}\sup_{\theta_s\in V_{D_p}}(\tilde{f}_{\theta_s}/f_{\theta_\star})(Y_s)$, $\P$-a.s., it follows, by the definition of  $\setY_{i,p}$, that
\begin{align}\label{eq:p_split3}
\E[X_{i,p}^{(1)}\ind_{\setY_{i,p}}(Y_{1:t_i})|\, \tilde{\pi}_{t_{i-1}}(V_{2D_p})<\epsilon ]\leq e^{(t_i-t_{i-1})\zeta(D_p)/2}.
\end{align}

We now show that under  Assumption \ref{new1} there exists a constant $C_2\in(0,\infty)$  such that
\begin{equation}\label{eq:p_split4}
\begin{split}
 \E[X_{i,p}^{(2)}\ind_{\setY_{i,p}}(Y_{1:t_i})| \, \tilde{\pi}_{t_{i-1}}(V_{2D_p})<\epsilon]\leq  C_2 e^{-C^{-1}_2(t_i-t_{i-1})D_p^2}.
\end{split}
\end{equation}

As preliminary computations to establish \eqref{eq:p_split4}  remark that, under the assumptions on $(\mu_t)_{t\geq 0}$ and for some constant  $C_3\in(0,\infty)$   we have, by \eqref{eq:Doob_gauss},
\begin{equation}\label{eq:p_split5}
\begin{split}
\P\big([U_{(t_{i-1}+1):t_i}]\geq D_p\big)&\leq C_3\exp\Big(-\frac{D_p^2}{C_3\sum_{j=t_{i-1}+1}^{t_i-1}h_j^2}\Big)\\
&=C_3\exp\Big(-\frac{(t_i-t_{i-1})D_p^2}{C_3(t_i-t_{i-1})\sum_{j=t_{i-1}+1}^{t_i-1}h_j^2}\Big)\\
&\leq C_4\exp\big(-C_4^{-1}(t_i-t_{i-1})D_p^2\big)
\end{split}
\end{equation}
where the last inequality holds  for some constant $C_4\in(0,\infty)$ since, by assumption,   $(t_p-t_{p-1})\sum_{i=t_{p-1}+1}^{t_p-1}h_i^2\rightarrow 0$.

Assume first Assumption \ref{new1}.\ref{A51}  and recall that $\beta$ is such that $\beta\alpha>1$. Then, by Lemma \ref{lemma:simple}, for $p_0$ large enough, we can without loss of generality assume that, for all $p'>p_0$ and  $j\in (p_0+1):p'$, the set  $\setY_{j,p'}$ is such that
\begin{align*}
\sum_{s=t_{j-1}+1}^{t_j}\sup_{\theta\in\Theta}\log\big((\tilde{f}_\theta/f_{\theta_\star})(y_s)\big) \leq  (t_j-t_{j-1}) D_{p'},\quad\forall y_{1:t_{p'}}\in\setY_{j,p'}.
\end{align*}
Then, using \eqref{eq:p_split5},
\begin{equation*}
\begin{split}
 \E[X_{i,p}^{(2)}\ind_{\setY_{i,p}}(Y_{1:t_i})| \, \tilde{\pi}_{t_{i-1}}(V_{2D_p})<\epsilon]&\leq \P\big([U_{(t_{i-1}+1):t_i}]\geq D_p\big)e^{ (t_i-t_{i-1}) D_p}\\
& \leq C_4 e^{-(t_i-t_{i-1})(C_4^{-1}D^2_p- D_p)}
\end{split}
\end{equation*}
which establishes \eqref{eq:p_split4}.

Assume now Assumption  \ref{new1}.\ref{A52}  and let $c_\star=\log\big(\sup_{\theta\in\Theta}\E[f_\theta/f_{\theta_\star}]\big)<\infty$. Then, using \eqref{eq:p_split5} we have, $\P$-a.s.
\begin{equation*}
\begin{split}
 \E[X_{i,p}^{(2)}|\F_{t_{i-1}}]&\leq \P\big([U_{(t_{i-1}+1):t_i}]\geq D_p\big)\big(\sup_{\theta\in\Theta}\E[f_{\theta}/f_{\theta_\star}]\big)^{t_i-t_{i-1}}\\
&= C_4 e^{-(t_i-t_{i-1})(C_4^{-1}D_p^2-c_\star)}
\end{split}
\end{equation*}
which shows \eqref{eq:p_split4}.

Lastly, assume  Assumption \ref{new1}.\ref{A53}. Let $c_\star=\log(\sup_{\theta\in\Theta}\E[ f_{\theta}])<\infty$ and note that, by Lemma \ref{lemma:simple}, for $p_0$ large enough we can without loss of generality assume that, for all $p'>p_0$ and  $j\in (p_0+1):p'$, the set  $\setY_{j,p'}$ is such that
\begin{align*}
\sum_{s=t_{j-1}+1}^{t_j}\log (f_{\theta_\star}(y_s))\geq  - (t_j-t_{j-1})D_{p'},\quad\forall y_{1:t_{p'}}\in\setY_{j,p'}.
\end{align*}
Then, using \eqref{eq:p_split5}  we have, $\P$-a.s.
\begin{equation*}
\begin{split}
 \E[X_{i,p}^{(2)}|\F_{t_{i-1}}]&\leq e^{D_p(t_i-t_{i-1})}\P\big([U_{(t_{i-1}+1):t_i}]\geq D_p\big)\big(\sup_{\theta\in\Theta}\E[ f_{\theta}]\big)^{t_i-t_{i-1}}\\
&= C_4 e^{-(t_i-t_{i-1})(C_4^{-1}D_p^2-D_p-c_\star)}
\end{split}
\end{equation*}
showing \eqref{eq:p_split4}. This concludes to show that \eqref{eq:p_split4} holds under Assumption \ref{new1}.

Then, using  \eqref{eq:p_split2}-\eqref{eq:p_split4}, we have
\begin{equation}\label{eq:p_split6}
\begin{split}
\P(\tilde{A}_{i,p}| A^c_{i-1,p})&\leq C_1  \delta_p^{-d} h_{t_{i-1}}^{-\nu} h_{t_{p-1}}^{-\beta(\nu+d)}\frac{ e^{(t_i-t_{i-1})(2(\delta_p\tilde{C}_\star)^2+\delta+\zeta(D_p)/2)}}{\epsilon\,\P( U_{ t_{i-1}:t_i} \in\Theta_{\delta_p, t_{i-1}:t_i})}\\
&+C_1 C_2 \delta_p^{-d} h_{t_{i-1}}^{-\nu} h_{t_{p-1}}^{-\beta(\nu+d)}\frac{e^{(t_i-t_{i-1})(2(\delta_p\tilde{C}_\star)^2+\delta-C_2^{-1}D_p^2)}}{\epsilon\,\P( U_{ t_{i-1}:t_i} \in\Theta_{\delta_p, t_{i-1}:t_i})}.
\end{split}
\end{equation}

We now find a lower bound for $\P( U_{ t_{i-1}:t_i} \in\Theta_{\delta_p, t_{i-1}:t_i})$. To this aim remark that, using the shorthand $\tilde{\delta}_p:=d^{-1/2}\delta_p$,
\begin{equation}\label{eq:int_p}
\begin{split}
\P(& U_{ t_{i-1}:t_i} \in\Theta_{\delta_p, t_{i-1}:t_i})\\
&=\P\Big(\big\|\sum_{j=s}^{t_i-1}U_s\big\|<\delta_p,\,\forall s\in  (t_{i-1}+1):(t_i-1)\Big)\\
&\geq \P\Big(\big\|\sum_{j=s}^{t_i-1}U_s\big\|_\infty<\tilde{\delta}_p,\,\forall s\in (t_{i-1}+1):(t_i-2)\Big)\\
&=\P\big(\big\|U_{t_{i}-1}\big\|_\infty<\tilde{\delta}_p\big) \prod_{s=t_{i-1}+1}^{t_{i}-2} \P\Big(\big\| \sum_{j=s}^{t_i-1}U_s\big\|_\infty<\tilde{\delta}_p\,\Big|\,\big\|\sum_{j=s+1}^{t_i-1}U_s\big\|_\infty<\tilde{\delta}_p\Big).
\end{split}
\end{equation}

Recall now that $\delta_p=t_p^{-\gamma}$ with $\gamma>\alpha$,  implying that $\delta_p/h_{t_p}\rightarrow 0$. Hence, under the assumptions on $(\mu_t)_{t\geq 0}$  there exists a constant $c_1>0$ such that, for every  $s\in 1:p$ such that $h_s\neq 0$ and  every $ v \in\R^d$ such that $\|v \|_\infty<\tilde{\delta}_p$, we have
\begin{align*}
 \P\Big(\big\| \sum_{j=s}^{t_i-1}U_s\big\|_\infty<\tilde{\delta}_p\,\Big|\,\big\|\sum_{j=s+1}^{t_i-1}U_s\big\|_\infty=v\Big)&=\P( \|U_s+v\|_\infty<\tilde{\delta}_p)\\
&=\P(-\tilde{\delta}_p\leq U_{s}+v <\tilde{\delta}_p)\\
&=\P\Big(-\frac{\tilde{\delta}_p+v}{h_s}\leq \tilde{U}_s <\frac{\delta_p-v}{h_s}\Big)\\
&\geq c_1\,(\tilde{\delta}_p/h_s)^d.
\end{align*}
Consequently, for every $s\in (t_{i-1}+1):(t_i-2)$ such that $h_s\neq 0$ we have
\begin{equation}\label{eq:int_p1}
\begin{split}
 \P\Big(\big\| \sum_{j=s}^{t_i-1}U_s\big\|_\infty<\tilde{\delta}_p\,\Big|\,\big\|\sum_{j=s+1}^{t_i-1}U_s\big\|_\infty<\tilde{\delta}_p\Big) &\geq  c_1\,(2\tilde{\delta}^2_p/h_s)^d.
\end{split}
\end{equation}
The above computations   also show that of $h_{t_i-1}>0$ then $\P\big(\|U_{t_{i}-1}\|_\infty<\tilde{\delta}_p\big)\geq c'_1\,(\tilde{\delta}_p/h_{t_i-1})^d$ for some constant $c_1'>0$ which, together  with \eqref{eq:int_p}-\eqref{eq:int_p1}, shows that for some constant $c_2>0$,
\begin{equation}\label{eq:p_split7}
\P( U_{ t_{i-1}:t_i} \in\Theta_{\delta_p, t_{i-1}:t_i}) \geq (c_2 \delta^{2d}_p)^{t_i-t_{i-1}-2}\prod_{s=t_{i-1}+1}^{t_i-1} \big(h_s^{-d}\wedge 1/c_2\big)\geq  \delta_p^{2d(t_i-t_{i-1})}
\end{equation}
where the second inequality assumes without loss of generality that $p_0$ is such that  $h^d_s<c_2$ for all $s>t_{p_0}$.

Combining \eqref{eq:p_split6} and \eqref{eq:p_split7}, and recalling that   $\delta_p\rightarrow 0$, that $|\zeta(D_p)|\rightarrow\infty$ and that $h_{t_{p}}<h_{t_{p-1}}$, it follows that for   $p_0$ large enough we have, for all $p>p_0$ and all $i\in(p_0+1):p$,
\begin{equation}\label{eq:p_split8}
\begin{split}
\P(\tilde{A}_{i,p} | A^c_{i-1,p})&\leq C_1\epsilon^{-1}  h_{t_{i-1}}^{-\nu} h_{t_{p}}^{-\beta(\nu+d)}  e^{-(t_i-t_{i-1})(|\zeta(D_p)|/4-3d\log(\delta_p))} \\
&+C_1 C_2\epsilon^{-1}   h_{t_{i-1}}^{-\nu} h_{t_{p}}^{-\beta(\nu+d)} e^{-(t_i-t_{i-1})(C_2^{-1}D_p^2/2-3d\log(\delta_p))}.
\end{split}
\end{equation}

To proceed further we first note that, under Assumption \ref{new2} and the assumptions on $(h_t)_{t\geq 0}$, there exist    constants $c'_\zeta>c_\zeta>0$ such that, for $p_0$ is large enough,
\begin{align}\label{eq:zeta}
|\zeta(D_p)|\geq c'_\zeta \log(D_p) =-c'_\zeta\,\beta \log(h_{t_p})\geq c_\zeta\,\beta\alpha \log(t_p).
\end{align}
Therefore, recalling that $\delta_p=t_p^{-\gamma}$ with $\gamma> \max(1/2,\alpha)$, for $p_0$ large enough we have, for all $p>p_0$ and all $i\in(p_0+1):p$,
\begin{align*}
  -(t_i-t_{i-1})\big( &|\zeta(D_p)|/4-3d\log(1/\delta_p)\big)\leq   -\log(t_p)(t_i-t_{i-1})\big(c_\zeta \beta\alpha/4 -3d\gamma ).
\end{align*}
Next, let $\tilde{c}>4\alpha\nu$ be sufficiently large so  that
$$
\beta':=\frac{4(\tilde{c}+3d\gamma)}{c_\zeta \alpha}\geq \max(\beta_\star,1/\alpha).
$$
Then, the above computations show that, for every $\beta\geq \beta'$ and  $p_0$ large enough we have, for all $p>p_0$ and all $i\in(p_0+1):p$,
\begin{align}\label{eq:beta_1}
-(t_i-t_{i-1})\big( |\zeta(D_p)|/4-3d\log(1/\delta_p)\big)\leq -\tilde{c}\log(t_p)(t_i-t_{i-1}),\,\,\forall i\in (p_0+1):p.
\end{align}

We now take $p_0$ sufficiently large so that
$$
t_i-t_{i-1}>\frac{\alpha\beta(\nu+d)+1}{\tilde{c}}+\frac{1}{2},\quad\forall i>p_0.
$$
Then, using \eqref{eq:beta_1} and under the assumptions on $(h_t)_{t\geq 0}$, and for $p_0$ large enough, there exists a constant $C_3\in(0,\infty)$ such that,  for all $p>p_0$,
\begin{align}\label{eq:beta_2}
h_{t_{p}}^{-\beta(\nu+d)}  e^{-(t_i-t_{i-1})(|\zeta(D_p)|/4-3d\log(1/\delta_p))}\leq C_3t_p^{-\tilde{c}/2-1},\quad  \forall i\in (p_0+1):p.
\end{align}

On the other hand, if $p_0$ is large enough then,  using \eqref{eq:zeta}, $C_2^{-1} D^2_p>c_\zeta\,\beta\alpha\log(t_p)/2$ for all $p>p_0$ and therefore the above computations imply that, for $p_0$ large enough and all $p>p_0$,
\begin{equation}\label{eq:beta_3}
\begin{split}
 h_{t_{p}}^{-\beta(\nu+d)}  e^{-(t_i-t_{i-1})(C_2^{-1}D_p^2/2-3d\log(\delta_p))}&\leq C_3 t_p^{-\tilde{c}/2-1},\quad  \forall i\in (p_0+1):p.
\end{split}
\end{equation}

Then, by combining \eqref{eq:p_split8}, \eqref{eq:beta_2} and \eqref{eq:beta_3}, and recalling that $\tilde{c}>4\alpha\nu$, it follows that, for $p_0$ large enough and some constant $C_3'\in(0,\infty)$, for all $p>p_0$ we have
\begin{align*}
 \sum_{i=p_0+1}^p \P(\tilde{A}_{i,p}| A^c_{i-1,p})&\leq  C_3 C_1(1+C_2)\epsilon^{-1} t_p^{-\tilde{c}/2-1}\sum_{i=p_0+1}^p h_{t_{i-1}}^{-\nu}\\
&\leq C_3' \epsilon^{-1} t_p^{-\tilde{c}/2+\nu\alpha}\\
&\leq C_3'\epsilon^{-1} t_p^{-\nu\alpha}
\end{align*}
and therefore
\begin{align}\label{eq:lim_sum}
\limsup_{p\rightarrow\infty}\sum_{i=p_0+1}^p \P(\tilde{A}_{i,p}| A^c_{i-1,p})&\leq \limsup_{p\rightarrow\infty} C_3' \epsilon^{-1} t_p^{-\nu\alpha}=0.
\end{align}

We now show that $\P(A_{p_0,p})\rightarrow 0$. To this aim remark that
$$
0\leq \ind_{V_{2D_p}}(\theta)\E[\tilde{\pi}_{t_{p_0}}(\theta)]\leq  \E[\tilde{\pi}_{t_{p_0}}(\theta)],\quad\forall \theta\in\R^d,\quad\forall p\geq 1
$$
and that, using Tonelli's  theorem,
$$
\int_{\R^d}\E[\tilde{\pi}_{t_{p_0}}(\theta)]\dd\theta=\E[\tilde{\pi}_{t_{p_0}}(\R^d)]=1.
$$
Therefore, by the reverse Fatou lemma we have (and using Tonelli's theorem for the first equality)
\begin{align*}
\limsup_{p\rightarrow\infty}\E[\tilde{\pi}_{t_{p_0}}(V_{2D_p})]&=\limsup_{p\rightarrow\infty}\int_{\R^d}\ind_{V_{2D_p}}(\theta)\E[\tilde{\pi}_{t_{p_0}}(\theta)]\dd\theta\\
&\leq \int_{\R^d}\limsup_{p\rightarrow\infty}\Big(\ind_{V_{2D_p}}(\theta)\E[\tilde{\pi}_{t_{p_0}}(\theta)]\Big)\dd\theta\\
&=0
\end{align*}
and thus, using Markov's inequality,
\begin{equation}\label{eq:lim_sum2}
\begin{split}
\limsup_{p\rightarrow\infty}\P(A_{p_0,p}  )&\leq \limsup_{p\rightarrow\infty}\P(\tilde{\pi}_{t_{p_0}}(V_{2D_p})\geq\epsilon)\leq \limsup_{p\rightarrow\infty}\frac{\E[\tilde{\pi}_{t_{p_0}}(V_{2D_p})]}{\epsilon}=0.
\end{split}
\end{equation}
Then, combining  \eqref{eq:res_1}, \eqref{eq:p_split1},\eqref{eq:lim_sum} and \eqref{eq:lim_sum2} proves    the theorem under Assumption \ref{new2}\ref{A61}.

We now prove the theorem under Assumption \ref{new2}.\ref{A63}.  Let $p$ be such that $\zeta(D_p)<0$ and remark first that, for some constant $C\in(0,\infty)$,
\begin{align*}
\P\Big(&\min_{1\leq i\leq p}   e^{ -(t_i-t_{i-1})\zeta(D_p)/2}\prod_{s=t_{i-1}+1}^{t_i}  f_{\theta_\star}(Y_s)\leq 1 \Big)\\
&=\P\Big(\min_{1\leq i\leq p}   \sum_{s=t_{i-1}+1}^{t_i}  \big(\log (f_{\theta_\star}(Y_s))-\zeta(D_p)/2\big)\leq 0\Big)\\
&=\P\Big(-\min_{1\leq i\leq p}   \sum_{s=t_{i-1}+1}^{t_i}  \big(\log( f_{\theta_\star}(Y_s))-\zeta(D_p)/2\big)\geq  0\Big)\\
&=\P\Big( \max_{1\leq i\leq p}   \sum_{s=t_{i-1}+1}^{t_i}  \big(\zeta(D_p)/2-\log (f_{\theta_\star}(Y_s))\big)\geq  0\Big)\\
&\leq \sum_{i=1}^p\P\Big(    \frac{1}{t_i-t_{i-1}}\sum_{s=t_{i-1}+1}^{t_i} \big(\E[\log (f_{\theta_\star}) ]- \log (f_{\theta_\star}(Y_s))\big)\geq  |\zeta(D_p)|/2+\E[\log (f_{\theta_\star}) ]\Big)\\
&\leq   \frac{C}{ (t_i-t_{i-1})^{k_\star\ind_{\mathbb{N}}(k_\star)} |\zeta(D_p)|^{2 k_\star}}
\end{align*}
where the second inequality uses Assumption \ref{new2}.\ref{A63}, Lemma \ref{lemma:mean_p} and Markov's inequality.

Therefore, under the assumptions of the theorem, there exists a sequence $(\setY_p)_{p\geq 1}$, with $\setY_p\subset\sigY^{\otimes t_p}$ for all $p\geq 1$,   such that $\P(Y_{1:t_p}\in\setY_p)\rightarrow 1$ and  such that
\begin{align*}
\min_{1\leq i\leq p} e^{ -(t_i-t_{i-1})\zeta(D_p)/2}\prod_{t_{i-1}+1}^{t_i}  f_{\theta_\star}(y_s)> 1,\quad \min_{1\leq i\leq p}\frac{g_{i}(y_{1:t_i})}{\tilde{g}_{i,p}(y_{1:t_i})}>1,\quad \forall y_{1:t_p}\in\setY_p,\quad\forall p\geq 1.
\end{align*}

Then, using the computations done to prove the theorem under Assumption \ref{new2}.\ref{A61},   to prove the theorem under \ref{new2}.\ref{A63} we only need to show that,  for $p_0\in\mathbb{N}$ large enough and every  $p>p_0$,
\begin{align}\label{eq:To_show_den22}
\E[X_{i,p}^{(1)}\ind_{\setY_{i,p}}(Y_{1:t_i})|  \, \tilde{\pi}_{t_{i-1}}(V_{2D_p})<\epsilon ]\leq e^{(t_i-t_{i-1})\zeta(D_p)/2},\quad \forall i\in (p_0+1):p,
\end{align}
where,  for every $p>p_0$, $(\setY_{i,p})_{i=0}^p$ is a sequence in $\sigY^{\otimes t_p}$ such that $\setY_{0,p}=\setY^{t_p}$,  $\setY_{p,p}=\setY_p$ and such that, for every $i\in 1:p$, we have  $\setY_{i,p}\subset \cap_{j=0}^{i-1}\setY_{j,p}$ and
\begin{align*}
\min_{1\leq j\leq i} e^{-(t_j-t_{j-1})\zeta(D_p)/2}\prod_{t_{j-1}+1}^{t_j}  f_{\theta_\star}(y_s)\geq 1,\quad \min_{1\leq j\leq i}\frac{g_{j}(y_{1:t_j})}{\tilde{g}_{j,p}(y_{1:t_j})}>1,\quad \forall y_{1:t_p}\in\setY_{i,p}.
\end{align*}

Using the   definitions of $\zeta(D_p)$ under Assumption \ref{new2}.\ref{A63} and the above definition of  $Y_{i,p}$  under Assumption \ref{new2}.\ref{A63}, we have  for $p_0$ large enough and all $p>p_0$,
\begin{equation*}
\begin{split}
 \E[X_{i,p}^{(1)}\ind_{\setY_{i,p}}(Y_{1:t_i})|  \, \tilde{\pi}_{t_{i-1}}(V_{2D_p})<\epsilon  ]&\leq e^{-\zeta(D_p)/2(t_j-t_{j-1})} \big(\sup_{\theta\in V_{D_p}}\E[\tilde{f}_\theta]\big)^{t_i-t_{i-1}}\\
&=  e^{(t_i-t_{i-1})\zeta(D_p)/2}.
\end{split}
\end{equation*}
 This shows \eqref{eq:To_show_den22} and the proof of the theorem under Assumption \ref{new2}.\ref{A63} is complete.

Lastly, we prove the theorem under Assumption \ref{new2}.\ref{A62}, where $\zeta(D_p)= \log(\sup_{\theta\in V_{D_p}}\E [\tilde{f}_{\theta}/f_{\theta_\star}])$. Following the computations of the proof of the theorem under Assumption \ref{new2}.\ref{A61}, to prove the theorem under Assumption \ref{new2}.\ref{A62} we only need to show that,     $p_0\in\mathbb{N}$ large enough we have,   for all $p>p_0$,
\begin{align}\label{eq:To_show_den}
\E[X_{i,p}^{(1)}\ind_{\setY_{i,p}}(Y_{1:t_i})|\, \tilde{\pi}_{t_{i-1}}(V_{2D_p})<\epsilon   ]\leq e^{(t_i-t_{i-1})\zeta(D_p)/2},\quad \forall i\in (p_0+1):p
\end{align}
where,  for every $p>p_0$, $(\setY_{i,p})_{i=0}^p$ is a sequence in $\sigY^{\otimes t_p}$ such that $\setY_{0,p}=\setY^{t_p}$,  $\setY_{p,p}=\setY_p$ and such that, for every $i\in 1:p$,  $\setY_{i,p}\subset \cap_{j=0}^{i-1}\setY_{j,p}$ and
\begin{align*}
\min_{1\leq j\leq i}\frac{g_{j}(y_{1:t_j})}{\tilde{g}_{j,p}(y_{1:t_j})}>1,\quad \forall y_{1:t_p}\in\setY_{ji,p}.
\end{align*}
Using the  definition of $\zeta(D_p)$ and the above definition of $Y_{i,p}$ under Assumption \ref{new2}\ref{A62} we have
\begin{equation*}
\begin{split}
 \E[X_{i,p}^{(1)}\ind_{\setY_{i,p}}(Y_{1:t_i})|A^c_{i-1,p} ]&\leq  \big(\sup_{\theta\in V_{D_p}}\E\big[\tilde{f}_{\theta}/f_{\theta_\star}\big]\big)^{t_i-t_{i-1}}=  e^{(t_i-t_{i-1})\zeta(D_p)}.
\end{split}
\end{equation*}
This shows \eqref{eq:To_show_den} and the proof of the theorem under Assumption \ref{new2}.\ref{A62} is complete.
\end{proof}

\subsection{Proof of Propositions \ref{prop:h_t}\label{p-prop:h_t}}

\begin{proof}

Under the assumptions of the proposition there exist  constants $C_1,C_2\in(0,\infty)$ such that,    for  $p$ large enough,
\begin{align*}
(t_p-t_{p-1})  \sum_{s=t_{p-1}+1}^{t_p-1}h_s^2&\leq (t_p-t_{p-1})^2h^2_{t_{p-1}}\\
&\leq  C_1 (t_p-t_{p-1})^2 t_{p-1}^{-2\alpha}\\
&\leq  C_2  \log(t_{p_{-1}})^2\big) t_{p-1}^{-2(\alpha-\varrho)}.
\end{align*}
Therefore, since $\varrho\in(0,\alpha)$, we have $
(t_p-t_{p-1})  \sum_{s=t_{p-1}+1}^{t_p-1}h_s^2\rightarrow 0$
as required. Together with the fact that$(t_{p+1}-t_p)\rightarrow\infty$, this shows that Condition \ref{thm1.2} of Theorem \ref{thm:online} holds.

To show that Condition \ref{thm1.3} of Theorem \ref{thm:online} holds  as well remark first that $(t_p-t_{p-1})^{-1} \log(1/h_{t_{p-1}})\rightarrow 0$ as required. In addition,  with $c\in(0,1)$ as in the statement of the proposition, we have for $p$ large enough
\begin{align}\label{eq:bound_varrpho}
\frac{t_{p+2}-t_{p+1}}{t_{p+1}-t_{p}}&\leq  \frac{ C_{p+1} \log(t_{p+1})}{C_{p} \log(t_{p})-1}\leq  \frac{ t_{p+1}^\varrho \log(t_{p+1}) }{c^2 t_{p}^\varrho \log(t_{p})-c}\leq  \frac{2}{c^2}\frac{t^\varrho_{p+1}\log(t_{p+1})}{t^\varrho_{t_p}\log(t_{p})}
\end{align}
where, for $p$ large enough,
\begin{align}\label{eq:bound_varrpho2}
\frac{\log(t_{p+1})}{\log(t_{p})}\leq 1+\frac{t_{p+1}-t_p}{t_p\log(t_p)}\leq  1+\frac{t_{p+1}}{t_p}.
\end{align}
Next, recalling that $\varrho\in(0,1)$, for $p$ large enough we have
\begin{align*}
\frac{t_{p+1}}{t_p}\leq 1+\frac{t_p^{\varrho} \log(t_p)}{c\,t_p}\leq 1+(2/c)
\end{align*}
with, together with \eqref{eq:bound_varrpho} and \eqref{eq:bound_varrpho2},  shows that  $\limsup_{p\rightarrow\infty}(t_{p+2}-t_{p+1})/(t_{p+1}-t_{p})<\infty$. This concludes to show that   Condition \ref{thm1.3} of Theorem \ref{thm:online} holds.  

To show the second part of the proposition  note that there exists a constant $c'>0$ such that    $(t_p-t_{p-1})\geq  c' t_{p-1}^{\varrho} $ for all $p\geq 1$. By assumption,  $k_\star>1/\varrho+1$, and thus
\begin{align*}
\limsup_{p\rightarrow\infty}\sum_{i=1}^p (t_i-t_{i-1})^{-k_\star}\leq (1/c') \limsup_{p\rightarrow\infty}\sum_{i=1}^p   t_{i-1}^{-k_\star\varrho} \leq (1/c') \limsup_{p\rightarrow\infty} \sum_{i=1}^\infty   t_{i }^{-(1+\varrho)}<\infty.
\end{align*}
The result follows.
\end{proof}

\subsection{Proof of the preliminary results given in Section \ref{sec:p-prelim}}

\subsubsection{Proof of Lemma \ref{lemma:new}\label{p-lemma:new}}

We first show the following simple lemma.
\begin{lemma}\label{lemma:simple0}
Let $(s_t)_{t\geq 1}$ be a sequence in $\mathbb{N}_0$ such that $\inf_{t\geq 1}(t-s_t)\geq 1$ and $(t-s_t)\rightarrow\infty$, $(X_t)_{t\geq 1}$ be a sequence of independent real-valued random variables such that $\E[X_n]\rightarrow m$ for some $m\in\R$ and such that $\sup_{s\geq 1}\var(X_s)<\infty$. 
Then, we have $(t-s_t)^{-1}\sum_{s=s_t+1}^t X_s\rightarrow m$ in $\P$-probability.
\end{lemma}
\begin{proof}
For every $t\geq 1$ we have
\begin{align}\label{eq:sum}
\frac{1}{t-s_t}\sum_{s=s_t+1}^t X_s-m=\frac{1}{t-s_t}\sum_{s=s_t+1}^t( X_s-\E[X_s])+\Big(\frac{1}{t-s_t}\sum_{s=s_t+1}^t \E[X_s]-m\big).
\end{align}
Let $\epsilon>0$ and $s_\epsilon\in\mathbb{N}$ be such that $|\E[X_s]-m|\leq \epsilon$ for all $s\geq s_\epsilon$. Then, for all $t$ such that $s_t\geq s_\epsilon$ we have
\begin{align*}
\Big|\frac{1}{t-s_t}\sum_{s=s_t+1}^t \E[X_s]-m\Big|\leq \frac{1}{t-s_t}\sum_{s=s_t+1}^t| \E[X_s]-m|\leq \epsilon
\end{align*}
showing that $(t-s_t)^{-1}\sum_{s=s_t+1}^t \E[X_s]\rightarrow m$. Hence, by \eqref{eq:sum}, to complete the proof it remains to show that $
(t-s_t)^{-1}\sum_{s=s_t+1}^t( X_s-\E[X_s])\rightarrow 0$ in $\P$-probability. Using Markov's inequality, for every $\epsilon>0$ we have
\begin{align*}
\limsup_{t\rightarrow\infty}\P\Big(\big|\frac{1}{t-s_t}\sum_{s=s_t+1}^t( X_s-\E[X_s])\big|\geq \epsilon\Big)&\leq \limsup_{t\rightarrow\infty}\frac{\sum_{s=s_t+1}^t\var(X_s)}{\epsilon^2(t-s_t)^2}\\
&\leq \limsup_{t\rightarrow\infty}\frac{\sup_{s\geq 1}\var(X_s)}{\epsilon^2 (t-s_t)}\\
&=0
\end{align*}
and the proof is complete.
\end{proof}

\textit{Proof of Lemma \ref{lemma:new}.} Let $A_\star$ and $\tilde{A}_\star$ be as in Assumption \ref{new} and $A'_\star\in\mathcal{B}(\R^d)$ be such that $A'_\star$ contains a neighbourhood of $A_\star$ and such that $\tilde{A}_\star$ contains a neighbourhood of $A'_\star$. Let   $B'_\star=A'_\star\cap \Theta$, $\tilde{B}_\star=\tilde{A}_\star\cap \Theta$,  and remark  that we can without loss of generality assume that the sequence $(\gamma_t')_{t\geq 1}$ is non-increasing.  Below we denote by $\bar{B}$  the closure of the set $B\subset\R^d$.

Let  $(\delta_t)_{t\geq 1}$ be a sequence in $\R_{>0}$ and let $W_s=\overline{B_{\gamma'_s}(0)}$ for all $s\geq 1$. Then, for all $t\geq 1$ we have
\begin{align*}
\P\Big(\sup_{(u_{s_t:t},\theta)\in \Theta_{\gamma_t',s_t:t}\times B'_\star}  \prod_{s=s_t+1}^t  & (\tilde{f}_{\theta-\sum_{i=s}^{t-1} u_i}/f_{\theta})(Y_s)\geq e^{(t-s_t)\delta_{t}} \Big)\\
&\leq \P\Big(\sup_{\theta\in B'_\star}\prod_{s=s_t+1}^t \sup_{v_s\in \overline{B_{\gamma'_t}(0)}}  (\tilde{f}_{\theta+v_s}/f_{\theta}) (Y_s)\geq e^{(t-s_t)\delta_{t}}\Big)\\
&\leq \P\Big( \prod_{s=s_t+1}^t \sup_{(\theta,w_s)\in B'_\star\times\mathsf{W}_s}  (\tilde{f}_{\theta+w_s}/f_{\theta}) (Y_s)\geq e^{(t-s_t)\delta_{t}}\Big)
\end{align*}
where 
the last inequality uses the fact that the sequence $(\gamma'_t)_{t\geq 1}$ is non-increasing. Consequently,  to prove the lemma it is enough to show that
\begin{align}\label{eq:toShow_a4}
 \frac{1}{t-s_t}\sum_{s=s_{t}+1}^{t}\sup_{(\theta,w_s)\in B'_\star\times \mathsf{W}_s}\log \big((\tilde{f}_{\theta+w_s}/f_{\theta})(Y_s)\big)\rightarrow0,\quad\text{in $\P$-probability}.
\end{align}

To establish \eqref{eq:toShow_a4} we define, for every $s\geq 1$,
\begin{align*}
X_s=\sup_{(\theta,w_s)\in B'_\star\times  \mathsf{W}_s}\log \big((\tilde{f}_{\theta+w_s}/f_{\theta})(Y_s)\big),\quad m^{(1)}_s=\E[X_s],\quad m_s^{(2)}=\E[X_s^2]
\end{align*}
and show below that there exists a $\underline{t}\in\mathbb{N}$ such that
\begin{align}\label{eq:Kol}
 \lim_{t\rightarrow\infty}m^{(1)}_{ \underline{t}+t}\rightarrow 0,\quad \sup_{s\geq s_{\underline{t}}}m_s^{(2)}\leq m_{s_{\underline{t}}}^{(2)}<\infty.
\end{align}
Then, \eqref{eq:toShow_a4} will follow by Lemma \ref{lemma:simple0}.

To show \eqref{eq:Kol} remark first that  for all $s\geq 1$  we have $m^{(1)}_s=\E[\tilde{X}_s]$,   where
$$
  \tilde{X}_s=\sup_{(\theta,w_s)\in B'_\star\times  \mathsf{W}_s}\log \big((\tilde{f}_{\theta+w_s}/f_{\theta})(Y_1)\big).
$$
Next, let $\tilde{\Omega}\in\F$ be such that $\P(\tilde{\Omega})=1$ and such that  the mapping $\theta\mapsto    f_\theta(Y_1(\omega))$ is    continuous on the compact set $\tilde{B}_\star$, for all $\omega\in\tilde{\Omega}$, notice that such a set $\tilde{\Omega}$ exists under Assumption \ref{new}. Let $\underline{s}\in\mathbb{N}$ be such that   $\theta+w\in  \tilde{B}_\star$  for all $(\theta,w)\in B'_\star\times W_{\underline{s}}$ such that  $\theta+w\in\Theta$. Then, recalling that  $\gamma'_{s+1}\leq \gamma'_{s}$ for all $s\geq 1$, it follows that 
$$
  \tilde{X}_s\leq \sup_{(\theta,\theta')\in \tilde{B}^2_\star :\,\|\theta-\theta'\|\leq \gamma'_s }\log \big((f_{\theta'}/f_{\theta})(Y_1)\big),\quad\forall s\geq \underline{s}.
$$
Then, by Weierstrass's theorem we have, for all $\omega\in \tilde{\Omega}$ and $s\geq\underline{s}$ we have
\begin{equation*}
\begin{split}
\tilde{X}_s(\omega)&\leq  \log\Big( \big(f_{h_{\gamma'_s}(\omega)}/f_{g_{\gamma'_s}(\omega)}\big) \big(Y_1(\omega)\big)\Big)
\end{split}
\end{equation*}
for some (measurable)   functions $h_{\gamma'_s}: \tilde{\Omega} \rightarrow  \tilde{B}_\star$ and $g_{\gamma'_s}: \tilde{\Omega}\rightarrow  \tilde{B}_\star$ such that we have $\|h_{\gamma'_s}(\omega)-h_{\gamma'_s}(\omega)\|\leq\gamma'_s$ for all $\omega\in \tilde{\Omega}$.

 By the maximum   theorem, we can assume that, for all  $s\geq  \underline{s}$ and every    $\omega\in\tilde{\Omega}$ the mappings $\gamma\mapsto h_{\gamma}(\omega)$ and $\gamma\mapsto g_{\gamma}(\omega)$ are continuous on $[0,\gamma'_{\underline{s}}]$.   Therefore,  since $h_{0}(\omega)=g_{0}(\omega)$ for all $\omega\in\tilde{\Omega}$, we have
\begin{equation}\label{eq:lim}
\begin{split}
0\leq \limsup_{s\rightarrow\infty}\tilde{X}_{\underline{s}+s}(\omega)&\leq \limsup_{s\rightarrow \infty}\log\Big( \big(f_{h_{\gamma'_s}(\omega)}/f_{g_{\gamma'_s}(\omega))}\big) \big(Y_1(\omega)\big)\Big)\\
&=\log\Big( \big(f_{h_{0}(\omega)}/f_{g_{0}(\omega)}\big) \big(Y_1(\omega)\big)\Big)\\
&=0,\qquad\qquad  \forall \omega\in\tilde{\Omega}.
\end{split}
\end{equation}

To proceed further remark that   since $W_s\subseteq W_{\underline{s}}$ for all $s\geq  \underline{s}$,   it follows that for all $s\geq \underline{s}$  we have $\tilde{X}_s\leq \tilde{X}_{\underline{s}}$, $\P$-a.s. Let $\tilde{\delta}>0$ be as in \eqref{new}. Then, under this latter   assumption, and taking $\underline{s}\in\mathbb{N}$ sufficiently large so that $\gamma'_{\underline{s}}\leq \tilde{\delta}$, we have $\E[\tilde{X}_{\underline{s}}]<\infty$. Then, by the dominated convergence theorem, $\lim_{s\rightarrow\infty} m^{(1)}_{\underline{s}+s}= 0$,  showing the first part of \eqref{eq:Kol}. To show the second part of \eqref{eq:Kol} it suffices to remark that $m^{(2)}_{s}\leq m^{(2)}_{\underline{s}}$ for all $s\geq \underline{s}$ where, under \eqref{new}, $m^{(2)}_{\underline{s}}<\infty$.\hfill$\square$

\subsubsection{Proof of Lemma \ref{lemma:pi}\label{p-lemma:pi}}

\begin{proof}
Let $t \geq 0$ and $A\in\mathcal{B}(\R^d)$. Then, if  $\mu_t\neq \delta_{\{0\}}$  we have, $\P$-a.s.,
 \begin{align*}
(\mu_{t}*\tilde{\pi}_{t})(A)
&=\int_{\R^d\times \R^d} \ind_A(\theta) \tilde{\pi}_{t} (\theta-u_{t}) \mu_{t}(\dd u_{t}) \dd\theta =\int_{A}  \E\big[\tilde{\pi}_{t} (\theta-U_{t}) |\F_t\big]\dd\theta
\end{align*}
while, if $\mu_t= \delta_{\{0\}}$ we have, $\P$-a.s.,
 \begin{align*}
(\mu_{t}*\tilde{\pi}_{t})(A)
=\int_{\R^d\times \R^d} \ind_A(\theta+u_t)   \mu_{t}(\dd u_{t})\tilde{\pi}_t(\theta)\dd\theta  =\int_{A}    \tilde{\pi}_t(\theta)\dd\theta.
\end{align*}
Recall that $\P(\cap_{t\geq 0}\Omega_t)=1$ if $\P(\Omega_t)=1$ for all $t\geq 0$ and that two probability measures $\nu_1,\nu_2\in\mathcal{P}(\R^d)$ are such that $\nu_1=\nu_2$ if $\nu_1(E_i)=\nu_2(E_i)$ for all $i\geq 1$, with $(E_i)_{i\geq 1}$ a dense subset of  $\R^d$ such that $E_i\in \mathcal{B}(\R^d)$ for all $i\geq 1$.

Therefore, the above computations imply that
\begin{align}\label{eq:density}
(\mu_{t}*\tilde{\pi}_{t})(A)=\int_A \E\big[\tilde{\pi}_{t} (\theta-U_{t}) |\F_t\big]\dd\theta,\quad\forall t\geq 0,\quad\forall A\in\mathcal{B}(\R^d),\quad\P-a.s.
\end{align}
We now prove the result of the lemma by induction on $t\geq 1$.

The result trivially holds for $t=1$ and we now assume  that it holds for some $t\geq 1$. Then, $\P$-a.s.,
\begin{align*}
 \tilde{\pi}_{t+1}(\dd\theta)&\propto   \tilde{f}_\theta(Y_{t+1})(\mu_t*\tilde{\pi}_t)(\dd\theta)\\
&= \tilde{f}_\theta(Y_{t+1})\E\big[\tilde{\pi}_{t} (\theta-U_{t}) |\F_t\big]\dd\theta\\
&\propto \tilde{f}_\theta(Y_{t+1})\E\Big[  (\mu_0*\tilde{\pi}_0) \big(\theta-\sum_{s=1}^{t}U_s\big)\prod_{s=1}^{t} \tilde{f}_{\theta-\sum_{i=s}^{t}U_i}(Y_s)\big|\F_{t}\Big]\dd\theta\\
&= \E\Big[ (\mu_0*\tilde{\pi}_0) \big(\theta-\sum_{s=1}^{t}U_s\big)\prod_{s=1}^{t+1} \tilde{f}_{\theta-\sum_{i=s}^{t}U_i}(Y_s)\big|\F_{t+1}\Big]\dd\theta
\end{align*}
where the first equality uses \eqref{eq:density} and the second line the inductive hypothesis. The proof is complete.
\end{proof}

\subsubsection{Proof of Lemma \ref{lemma:denom}}\label{p-lemma:denom}
\begin{proof}

Let $t\geq 1$, $p_t=\P(U_{s_t:t}\in  \Theta_{\tilde{\delta},s_t:t})$, $C_\star\in\R_{>0}$ be as in Assumption \ref{taylor}, $\tilde{C}_\star=2 (\E[m^2_\star]+C_\star)^{1/2}$ and note that,  under  Assumptions \ref{m_star}-\ref{taylor}  and for all $\theta\in B_{\delta_{\star}}(\theta_\star)$,
\begin{equation}\label{eq:Bound_KL_1}
\max\bigg(-\E\big[\log(\tilde{f}_{\theta}/f_{\theta_\star})\big], \E\big[\big(\log(\tilde{f}_{\theta}/f_{\theta_\star})\big)^2\big]\Big)\leq  \big(\E[m_\star^2]+C_\star\big)\|\theta-\theta_\star\|^2.
\end{equation}

Remark now that if $p_t=0$ then the result of the lemma trivially holds and henceforth we therefore assume that $p_t>0$. To simplify the notation let $\E_{Y_1}[\cdot]$ denote  expectations w.r.t.\ the distribution of $Y_1$, $\E^{\tilde{\delta}}_{t,\mu}[\cdot]$ denote   expectations  w.r.t.\ the restriction of $\otimes_{s=s_t}^t\mu_t$ to the set $\Theta_{\tilde{\delta},s_t:t}$, and let $V_t:=(U_s)_{s=s_t+1}^t$.

For every $u_{(s_t+1):t}\in\Theta^{t-s_t}$ let $\tilde{\eta}(\dd\theta, u_{s_t:t})$ be the   probability measure on $B_\delta(\theta_\star)$ with density function $\tilde{\eta}(\cdot, u_{(s_t+1):t})$ defined by
$$
\tilde{\eta}(\theta, u_{(s_t+1):t})=\frac{\eta\big(\theta-\sum_{s=s_t+1}^{t-1}u_s\big)}{\eta\big(B_\delta(\theta_\star-\sum_{s=s_t+1}^{t-1}u_s)\big)},\quad\theta\in B_\delta(\theta_\star).
$$
Then, using the shorthand $a_t=(t-s_t)(2(\tilde{C}_\star \delta)^2+\epsilon)$, we have
\begin{equation}\label{eq:p1}
\begin{split}
\tilde{\setY}_t&:=\bigg\{y_{s_t:t}:
\int_{\Theta}\E\Big[\eta\big(\theta-\sum_{s=s_t+1}^{t-1}U_s\big)\prod_{s=s_t+1}^{t} (\tilde{f}_{\theta-\sum_{i=s}^{t-1}U_i}/f_{\theta_\star})(y_s)\Big]
\dd\theta \leq p_t\, C^\eta_{\delta,\tilde{\delta}}\, e^{-a_t}\bigg\}\\
&=\bigg\{y_{s_t:t}:
\E\Big[\int_{\Theta}\eta\big(\theta-\sum_{s=s_t+1}^{t-1}U_s\big)\prod_{s=s_t+1}^{t} (\tilde{f}_{\theta-\sum_{i=s}^{t-1}U_i}/f_{\theta_\star})(y_s)
\dd\theta\Big] \leq  p_t\, C^\eta_{\delta,\tilde{\delta}}\, e^{-a_t}\bigg\}\\
&\subset \bigg\{y_{s_t:t}:\,\E^{\tilde{\delta}}_{t,\mu}\Big[\int_{B_{\delta}(\theta_\star)}\eta\big(\theta-\sum_{s=s_t+1}^{t-1}U_s\big)\prod_{s=s_t+1}^{t} (\tilde{f}_{\theta-\sum_{i=s}^{t-1}U_i} / f_{\theta_\star})(y_s)
\dd\theta\Big]\leq  C^\eta_{\delta,\tilde{\delta}}\, e^{-a_t}\bigg\}\\
&\subset  \bigg\{y_{s_t:t}:\,\E^{\tilde{\delta}}_{t,\mu}\Big[\int_{B_{\delta}(\theta_\star)}\prod_{s=s_t+1}^{t} (\tilde{f}_{\theta-\sum_{i=s}^{t-1}U_i}/f_{\theta_\star})(y_s)\tilde{\eta}(\dd\theta,V_t)\Big] \leq      e^{-a_t}\bigg\}\\
&\subset \bigg\{y_{s_t:t}: \E^{\tilde{\delta}}_{t,\mu}\Big[\sum_{s=s_t+1}^t \int_{B_{\delta}(\theta_\star)} \log\big((\tilde{f}_{\theta-\sum_{i=s}^{t-1}U_i}/f_{\theta_\star})(y_s)\big)\,\tilde{\eta}(\dd\theta, V_t)\Big]\leq  -a_t\bigg\}
\end{split}
\end{equation}
 where the equality uses Tonelli's theorem, the second inclusion uses the definition of $ C^\eta_{\delta,\tilde{\delta}}$ and the last inclusion uses twice   Jensen's inequality.

To simplify the notation in what follows we define, for every $t\geq s\geq 0$,
\begin{align}\label{eq:g}
g_\theta(u_{s:t},y)=\log\big( (\tilde{f}_{ \theta-\sum_{i=s}^{t-1}u_i}/f_{\theta_\star})(y)\Big),\quad  (\theta,u_{s:t})\in\Theta^{t-s+2},\,\,y\in\setY.
\end{align}

Remark now that, by \eqref{eq:Bound_KL_1} and using the inequality $\|a+b\|^2\leq 2(\|a\|^2+\|b\|^2)$ for all $a,b\in\R^d$,    
\begin{equation}\label{eq:Bound_KL}
\begin{split}
-&\sum_{s=s_t+1}^t\E[g_\theta(u_{s:t},Y_1)]\leq  (t-s_t)(\tilde{C}_\star \delta)^2,\quad\forall  \theta \in B_{\delta}(\theta_\star),\quad\forall u_{s_t:t}\in\Theta_{\tilde{\delta},s_t:t}\\
&\sum_{s=s_t+1}^t \E[g_\theta(u_{s:t},Y_1)^2]\leq (t-s_t)(\tilde{C}_\star \delta)^2,\quad\forall \theta \in B_{\delta}(\theta_\star),\quad\forall u_{s_t:t}\in\Theta_{\tilde{\delta},s_t:t}.
\end{split}
\end{equation}

Therefore, using  \eqref{eq:p1} and   \eqref{eq:Bound_KL}, we have
\begin{equation}\label{eq:p2}
\begin{split}
 &\P (Y_{s_t:t}\in \tilde{\setY}_t)\\
 &\leq \P \Big(\E^{\tilde{\delta}}_{t,\mu}\Big[\sum_{s=s_t+1}^t\int_{B_{\delta}(\theta_\star)} g_{\theta}(U_{s:t},Y_s)\tilde{\eta}(\dd\theta, V_t)\Big] \leq  -a_t\Big)\\
&= \P\Big(\E^{\tilde{\delta}}_{t,\mu}\Big[\frac{1}{t-s_t}\sum_{s=s_t+1}^t\Big\{\int_{B_{\delta}(\theta_\star)} g_\theta(U_{s:t},Y_s)-\int_{B_{\delta}(\theta_\star)}\E_{Y_1}\big[g_\theta(U_{s:t},Y_1)\big]\Big\}\tilde{\eta}(\dd\theta, V_t)\\
& \,\,\,\,\,\qquad +\frac{1}{t-s_t}\sum_{s=s_t+1}^t\int_{B_{\delta(\theta_\star)}} \E_{Y_1}\big[g_\theta(U_{s:t},Y_1)\big]\tilde{\eta}(\dd\theta, V_t)\Big]\leq -2(\tilde{C}_\star \delta)^2-\epsilon \Big)\\
&\leq \P\Big(\E^{\tilde{\delta}}_{t,\mu}\bigg[\frac{1}{t-s_t}\sum_{s=s_t+1}^t\Big\{\int_{B_{\delta}(\theta_\star)}g_\theta(U_{s:t},Y_s)-\int_{B_{\delta}(\theta_\star)}\E_{Y_1}\big[g_\theta(U_{s:t},Y_1)\big]\bigg\}\tilde{\eta}(\dd\theta, V_t)\Big]\\
& \,\,\,\,\,\qquad  \leq -(\tilde{C}_\star \delta)^2-\epsilon \Big).
\end{split}
\end{equation}

We now show that
\begin{equation}\label{eq:Fub3}
\begin{split}
\E^{\tilde{\delta}}_{t,\mu}\Big[ \int_{B_{\delta}(\theta_\star)}\E_{Y_1}\big[ g_\theta(U_{s:t},Y_1) \big]&\tilde{\eta}(\dd\theta, V_t)\Big] \\
&=\E_{Y_1}\Big[ \E^{\tilde{\delta}}_{t,\mu}\Big[  \int_{B_{\delta}(\theta_\star)}g_\theta(U_{s:t},Y_1)\tilde{\eta}(\dd\theta, V_t)\Big]\Big].
\end{split}
\end{equation}
By \eqref{eq:Bound_KL},  there exists a  constant $C_t\in[0,\infty)$ such, that for every $u_{s_t:t}\in \Theta_{\tilde{\delta},s_t:t}$,
\begin{equation}\label{eq:Fub0}
\begin{split}
 \int_{B_{\delta}(\theta_\star)}\E_{Y_1}\big[|g_\theta(u_{s:t},Y_1)|\big]\tilde{\eta}(\dd\theta, u_{s_t:t})&\leq  \int_{B_{\delta}(\theta_\star)}\E_{Y_1}\big[g_\theta(u_{s:t},Y_1)^2\big]^{\frac{1}{2}}\tilde{\eta}(\dd\theta, u_{s_t:t})\\
 &\leq C_t
\end{split}
\end{equation}
and thus by, Fubini-Tonelli's theorem,
\begin{equation}\label{eq:Fub1}
\begin{split}
\E^{\tilde{\delta}}_{t,\mu}\Big[\int_{B_{\delta}(\theta_\star)}\E_{Y_1}\big[ g_\theta(U_{s:t},Y_1) \big]&\tilde{\eta}(\dd\theta, V_t)\Big]\\
&= \E^{\tilde{\delta}}_{t,\mu}\Big[\E_{Y_1}\Big[ \int_{B_{\delta}(\theta_\star)} g_\theta(U_{s:t},Y_1) \tilde{\eta}(\dd\theta, V_t)\Big]\Big].
\end{split}
\end{equation}
Using \eqref{eq:Fub0}, we also have
\begin{align*}
\E^{\tilde{\delta}}_{t,\mu}\Big[ \E_{Y_1}\Big[\Big|\int_{B_{\delta}(\theta_\star)}g_\theta(U_{s:t},Y_1)\tilde{\eta}(\dd\theta, V_t)\Big|\Big]\Big] &\leq \E^{\tilde{\delta}}_{t,\mu} \Big[\E_{Y_1}\Big[ \int_{B_{\delta}(\theta_\star)}|g_\theta(U_{s:t},Y_1)|\tilde{\eta}(\dd\theta,V_t)\Big]\Big] \\
&= \E^{\tilde{\delta}}_{t,\mu}\Big[ \int_{B_{\delta}(\theta_\star)}\E_{Y_1}\big[|g_\theta(U_{s:t},Y_1)|\big]\tilde{\eta}(\dd\theta, V_t)\Big]\Big]\\
&\leq C_t,
\end{align*}
where the equality uses Tonelli's theorem. By   Fubini-Tonelli's theorem we therefore have
\begin{equation}\label{eq:Fub2}
\begin{split}
\E^{\tilde{\delta}}_{t,\mu}\Big[ \E_{Y_1}\Big[ \int_{B_{\delta}(\theta_\star)}g_\theta(U_{s:t},Y_1)&\tilde{\eta}(\dd\theta,V_t) \Big]\Big]\\
&=\E_{Y_1} \Big[\E^{\tilde{\delta}}_{t,\mu}\Big[ \int_{B_{\delta}(\theta_\star)}g_\theta(U_{s:t},Y_1)\tilde{\eta}(\dd\theta, V_t)\Big]\Big]
\end{split}
\end{equation}
and \eqref{eq:Fub3} follows from \eqref{eq:Fub1} and \eqref{eq:Fub2}.

Consequently, using \eqref{eq:p2} and \eqref{eq:Fub3}, we have
\begin{equation}\label{eq:moments0}
\begin{split}
 \P (Y_{s_t:t}\in \tilde{\setY}_t)&\leq \P\Big( \frac{1}{t-s_t}\sum_{s=s_t+1}^t\bigg\{\E^{\tilde{\delta}}_{t,\mu}\Big[ \int_{B_{\delta}(\theta_\star)}g_\theta(U_{s:t},Y_s)\tilde{\eta}(\dd\theta, V_t)\Big]\\
&\,\,\,\,\,\qquad- \E_{Y_1}\Big[\E^{\tilde{\delta}}_{t,\mu}\Big[\int_{B_{\delta}(\theta_\star)} g_\theta(U_{s:t},Y_1)\tilde{\eta}(\dd\theta, V_t)\Big]\Big]\bigg\} \leq  -(\tilde{C}_\star \delta)^2-\epsilon \Big)\\
&\leq \P\Big(\Big|\frac{1}{t-s_t}\sum_{s=s_t+1}^t\bigg\{\E^{\tilde{\delta}}_{t,\mu}\Big[ \int_{B_{\delta}(\theta_\star)}g_\theta(U_{s:t},Y_s)\tilde{\eta}(\dd\theta, V_t)\Big]\\
&\,\,\,\,\,\qquad- \E_{Y_1}\Big[\E^{\tilde{\delta}}_{t,\mu}\Big[\int_{B_{\delta}(\theta_\star)} g_\theta(U_{s:t},Y_1)\tilde{\eta}(\dd\theta, V_t)\Big]\Big]\bigg\}\Big|\geq  (\tilde{C}_\star \delta)^2+\epsilon\Big)
\end{split}
\end{equation}
and we finally upper bound the last term using Markov's inequality.

To this aim remark   that
\begin{equation}\label{eq:moments}
\begin{split}
\sum_{s=s_t+1}^t\E_{Y_1}\Big[  \E^{\tilde{\delta}}_{t,\mu}\Big[ \int_{B_{\delta}(\theta_\star)}&g_\theta(U_{s:t},Y_1)\tilde{\eta}(\dd\theta, V_t)\Big]^2\Big]\\
&\leq \sum_{s=s_t+1}^t\E_{Y_1}\Big[ \E^{\tilde{\delta}}_{t,\mu}\Big[\Big(\int_{B_{\delta}(\theta_\star)}g_\theta(U_{s:t},Y_1)\tilde{\eta}(\dd\theta,V_t)\Big)^2\Big]\Big]\\
&\leq \sum_{s=s_t+1}^t\E_{Y_1}\Big[ \E^{\tilde{\delta}}_{t,\mu}\Big[\int_{B_{\delta}(\theta_\star)}g_\theta(U_{s:t},Y_1)^2\tilde{\eta}( \theta, V_t)\dd\theta \Big]\Big]\\
&= \sum_{s=s_t+1}^t\E^{\tilde{\delta}}_{t,\mu}\Big[ \E_{Y_1}\Big[\int_{B_{\delta}(\theta_\star)}g_\theta(U_{s:t},Y_1)^2\tilde{\eta}(\dd\theta, V_t) \Big]\Big]\\
&= \sum_{s=s_t+1}^t\E^{\tilde{\delta}}_{t,\mu}\Big[ \int_{B_{\delta}(\theta_\star)}\E_{Y_1}\big[g_\theta(U_{s:t},Y_1)^2\big]\tilde{\eta}( \dd\theta, V_t)\Big]\\
&\leq (t-s_t)(\tilde{C}_\star \delta)^2
\end{split}
\end{equation}
where   the last inequality uses \eqref{eq:Bound_KL}, the first two inequalities use Jensen's inequality and the two equalities hold by Tonelli's theorem.

Therefore, using \eqref{eq:moments0}, \eqref{eq:moments} and Markov's inequality,
\begin{align*}
 \P (Y_{s_t:t}\in \tilde{\setY}_t)& \leq \frac{\frac{1}{(t-s_t)^2}\sum_{s=s_t+1}^t\var\Big(\E^{\tilde{\delta}}_{t,\mu}\big[ \int_{B_{\delta}(\theta_\star)}g_\theta(U_{s:t},Y_s)\tilde{\eta}(\dd\theta, V_t)\big]\Big)}{((\tilde{C}_\star\delta)^2+\epsilon)^2}\\
&\leq \frac{\frac{1}{(t-s_t)^2}\sum_{s=s_t+1}^t\E_{Y_1} \Big[ \E^{\tilde{\delta}}_{t,\mu}\big[ \int_{B_{\delta}(\theta_\star)}g_\theta(U_{s:t},Y_1)\tilde{\eta}(\dd\theta, V_t)\big]^2\Big]}{((\tilde{C}_\star\delta)^2+\epsilon)^2}\\
&\leq \frac{1}{(t-s_t) ((\tilde{C}_\star\delta)+(\tilde{C}_\star\delta)^{-1}\epsilon)^2}.
\end{align*}
The proof is complete.

\end{proof}

\subsubsection{Proof of Lemma \ref{lemma:test}\label{p-lemma:test}}

We first recall the following result due to \citet[][Lemma 3.3]{Kleijn2012}.

\begin{lemma}\label{lemma:test_K}
Assume Assumption \ref{test}. Then, for every  compact set $W\in\mathcal{B}(\Theta)$ such that $\theta_\star\in W$ and every $\epsilon>0$ there exist a sequence of measurable functions $(\psi_t)_{t\geq 1}$, with $\psi_t:\setY^t\rightarrow\{0,1\}$, and a  constant  $ D_\star\in(0,\infty) $ such that $\E[\psi_t(Y_{1:t})]\rightarrow 0$ and such that, for $t$ large enough,
$$
\sup_{\theta\in V_\epsilon\cap\Theta}\E\Big[\big(1-\psi_t(Y_{1:t})\big)\prod_{s=1}^t (f_{\theta}/f_{\theta_\star})(Y_s)\Big]\leq e^{-tD_\star}.
$$
\end{lemma}

\noindent \textit{Proof of Lemma \ref{lemma:test}.} Let $A_\star$ and $\tilde{A}_\star$ be as in Assumption \ref{new}, $A'_\star\subsetneq\tilde{A}_\star$ be as in Lemma \ref{lemma:new} and  $W=\tilde{A}_\star\cap \Theta$. Remark that  $W$ is a compact set under Assumption \ref{new} and that $A'_\star\cap\Theta\subset W$. Let   $(\psi_t)_{t\geq 1}$  be as in Lemma \ref{lemma:test_K} and $(\delta_t)_{t\geq 1}$ be as in Lemma \ref{lemma:new} for the sequence $(\gamma'_t)_{t\geq 1}$ defined by $\gamma'_t=2\gamma_t$, $\forall t\geq 1$. Without loss of generality we assume below that $\epsilon>0$ is such that  $B_{\epsilon}(\theta_\star)\subset A'_\star$. Indeed, since $\theta_\star\in A_\star$ and  $A'_\star$ contains an open set that contains $A_\star$, it follows that  $B_{\delta}(\theta_\star)\subset A'_\star$ for $\delta>0$ small enough. Let $t_1\in\mathbb{N}$ be such that  for all $t\geq t_1$ we have $\theta'-\sum_{i=s}^{t-1} u'_i\not\in  A_\star$ for all $\theta'\not\in A'_\star$ and all $u'_{s_t:t}\in \Theta_{\gamma_t,s_t:t}$. Notice that such a $t_1\in\mathbb{N}$ exists since $A_\star'$ contains a neighbourhood of $A_\star$.

We first show the lemma assuming   Assumption \ref{new}.\ref{A41}). To this aim, for every, $t\geq 1$  we define $\setY_t= \setY^{(1)}_t\cap \setY^{(2)}_t$  where
\begin{equation*}
\begin{split}
&\setY^{(1)}_t=\Big\{y_{1:t}\in\setY^t:\,\sup_{(u_{s_t:t},\,\theta)\in  \Theta_{\gamma_t,s_t:t}\times (A'_\star\cap\Theta)}\prod_{s=s_t+1}^t  (\tilde{f}_{\theta-\sum_{i=s}^{t-1} u_i}/f_{\theta}) (y_s)<e^{(t-s_t)\delta_{t}}\Big\}\\
&\setY^{(2)}_t=\Big\{y_{1:t}\in\setY^t:\,\sum_{s=s_t+1}^t\sup_{\theta\not\in A_\star}\log \big(\tilde{f}_{\theta}(y_s)\big)<\sum_{s=s_t+1}^t \log \big(f_{\theta_\star}(y_s)\big)-c_{\mathrm{A4}} \Big\}
\end{split}
\end{equation*}
where $c_{\mathrm{A4}}>0$ is such that $\E[\sup_{\theta\not\in A_\star}\log (\tilde{f}_{\theta})]< \E[ \log (f_{\theta_\star})]-2c_{\mathrm{A4}}$. Notice that such a constant $c_{\mathrm{A4}}>0$ exists under     Assumption \ref{new}.\ref{A41}).

For  every $t\geq 1$, let  $\phi_t(Y_{1:t})=\ind_{\setY_t^c}(Y_{1:t})+\psi_{t-s_t}(Y_{(s_t+1):t})\ind_{\setY_t}(Y_{1:t})$ and remark that  under Assumption \ref{new}.\ref{A41}), and by Lemmas \ref{lemma:new} and \ref{lemma:test_K},    $ \E[\phi_t(Y_{1:t})]\rightarrow 0$, as required.

To show the second part of the lemma let  $\theta\in V_\epsilon$, $u_{s_t:t}\in  \Theta_{\gamma_t,s_t:t}$. Remark that $1-\phi_t(Y_{1:t})=\big(1-\psi_{t-s_t}(Y_{(s_t+1):t})\big)\ind_{\setY_t}(Y_{1:t})$ for all $t\geq 1$, and assume first that   $\theta\in (A'_\star\cap\Theta)$. Then,
\begin{equation}\label{eq:test_int_1}
\begin{split}
 \E\Big[(1 -\phi_t(Y_{1:t})) &\prod_{s=s_t+1}^t (\tilde{f}_{\theta-\sum_{i=s}^{t-1}u_i}/f_{\theta_\star})(Y_s)\big|\,\F_{s_t}\Big]\\
&\leq e^{(t-s_t)\delta_{t}}\E\Big[\big(1-\psi_{t-s_t}(Y_{(s_t+1):t})\big)\prod_{s=s_t+1}^t  (f_{\theta}/f_{\theta_\star})(Y_s)\Big]\\
&\leq  e^{-(t-s_t)(D_\star   - \delta_{t})}\\
&\leq  e^{-(t-s_t)\frac{D_\star}{2}}
\end{split}
\end{equation}
where the first inequality uses the definition of $\setY^{(1)}_t$, the second inequality uses  Lemma \ref{lemma:test_K} and holds for $t$ large enough,  while the last inequality   holds for $t$ sufficiently large since $\delta_t\rightarrow 0$. Notice that if $\theta\not\in\Theta$ we have $\E\big[(1 -\phi_t(Y_{1:t})) \prod_{s=s_t+1}^t (\tilde{f}_{\theta-\sum_{i=s}^{t-1}u_i}/f_{\theta_\star})(Y_s)\big|\,\F_{s_t}\big]=0$ and thus \eqref{eq:test_int_1} also holds if $\theta\in (A'_\star\cap\Theta^c)$.

Assume now that $ \theta \not\in A'_\star$. Then,
\begin{equation}\label{eq:test_int_2}
\begin{split}
\E\Big[(1-\phi_t(Y_{1:t})) \prod_{s=s_t+1}^t  \frac{\tilde{f}_{\theta-\sum_{i=s}^{t-1}u_i}}{f_{\theta_\star}}(Y_s)\big|\,\F_{s_t}\Big]&\leq \E\Big[\ind_{\setY^{(2)}_t}(Y_{1:t})\prod_{s=s_t+1}^t \sup_{\theta \not\in A_\star}  \frac{\tilde{f}_{\theta}}{f_{\theta_\star}}(Y_s)\Big]\\
&\leq  e^{- (t-s_t) c_{\mathrm{A4}}}
\end{split}
\end{equation}
where the first inequality holds for $t\geq t_1$. Together with \eqref{eq:test_int_1}, \eqref{eq:test_int_2}  shows that the result of the lemma holds under    Assumption \ref{new}.\ref{A41}) with   $\tilde{D}_\star= D_\star/2\vee c_{\mathrm{A4}}$.

We now show the result of the lemma   under  Assumption  \ref{new}.\ref{A42}) and  under Assumption \ref{new}.\ref{A43}). To do so remark that, using the above computations, we only need to find a set $\setY_t^{(2)}\in\sigY^{\otimes t}$ such that $\P(Y_{1:t}\in\setY_t^{(2)})\rightarrow 1$ and  such that there exists a constant $c_{\mathrm{A4}}>0$ for which, for $t$ large enough,
\begin{align}\label{eq:c_stareq}
\sup_{\theta\not\in A'_\star}\E\big[(1-\phi_t(Y_{1:t}))\prod_{s=s_t+1}^t  \frac{\tilde{f}_{\theta-\sum_{i=s}^{t-1}u_i}}{f_{\theta_\star}}(Y_s)\big|\,\F_{s_t}\big]\leq e^{-(t-s_t)c_{\mathrm{A4}}},\quad\forall u_{s_t:t}\in \Theta_{\gamma_t,s_t:t}.
\end{align}

Assume  Assumption \ref{new}.\ref{A42}), let $\setY_t^{(2)}=\setY^t$ (so that $\P(Y_{1:t}\in\setY_t^{(2)})=1$ for all $t$),   $\theta\not\in  A'_\star$ and $u_{s_t:t}\in \Theta_{\gamma_t,s_t:t}$. Then,
\begin{equation}\label{eq:test_int_22}
\begin{split}
\E\Big[(1-\phi_t(Y_{1:t})) \prod_{s=s_t+1}^t  \frac{\tilde{f}_{\theta-\sum_{i=s}^{t-1}u_i}}{f_{\theta_\star}}(Y_s)\big|\,\F_{s_t}\Big]&\leq \E\Big[\prod_{s=s_t+1}^t  \frac{\tilde{f}_{\theta-\sum_{i=s}^{t-1}u_i}}{f_{\theta_\star}}(Y_s)\big|\,\F_{s_t}\Big]\\
&=\prod_{s=s_t+1}^t\E\Big[ \frac{\tilde{f}_{\theta-\sum_{i=s}^{t-1}u_i}}{f_{\theta_\star}}(Y_s)\big|\,\F_{s_t}\Big]\\
&\leq \big(\sup_{\theta\not\in A_\star}\E[\tilde{f}_\theta/f_{\theta_\star}]\big)^{t-s_t}
\end{split}
\end{equation}
where the last equality holds $t\geq t_1$. Under   Assumption \ref{new}.\ref{A42}),  $\sup_{\theta \not\in A_\star}\E[\tilde{f}_\theta/f_{\theta_\star}]<1$ and therefore \eqref{eq:test_int_22} shows that \eqref{eq:c_stareq} holds with $c_{\mathrm{A4}}=-\log\big(\sup_{\theta \not\in A_\star}\E[\tilde{f}_\theta/f_{\theta_\star}]\big)$.

Lastly, assume    Assumption \ref{new}.\ref{A43}) and remark that under this condition there exists a constant $c>0$   such that  $\log( \sup_{\theta \not\in A_\star}\E[\tilde{f}_\theta])<  \E[\log f_{\theta_\star}] -c$, and let
$$
\setY^{(2)}_t=\Big\{y_{1:t}\in\setY^t:\,\frac{1}{t-s_t}\sum_{s=s_t+1}^t \log\big(f_{\theta_\star}(y_s)\big)>\E[\log (f_{\theta_\star})]-c\Big\}.
$$
Then, by the law of large numbers, $\P(Y_{1:t}\in \setY^{(2)}_t)\rightarrow 1$ while, for every $\theta\not\in A'_\star$ and $u_{s_t:t}\in \Theta_{\gamma_t,s_t:t}$,
\begin{align*}
 \E\Big[\ind_{\setY^{(2)}_t} (Y_{1:t})\prod_{s=s_t+1}^t & (\tilde{f}_{\theta-\sum_{i=s}^{t-1}u_i}/f_{\theta_\star})(Y_s) \ind_{\setY^{(2)}_t}(Y_{1:t})\Big]\\
&\leq e^{(t-s_t)(c-\E[\log (f_{\theta_\star})])}\prod_{s=s_t+1}^t\E\big[  \tilde{f}_{\theta-\sum_{i=s}^{t-1}u_i} (Y_s)\big|\,\F_{s_t}\big]\\
&\leq e^{(t-s_t)(c-\E[\log (f_{\theta_\star})])}\big(\sup_{\theta\not\in A_\star}\E[\tilde{f}_\theta]\big)^{t-s_t}\\
&=e^{-(t-s_t)\big(\E[\log f_{\theta_\star}]-c-\log(\sup_{\theta\not\in A_\star}  \E[\tilde{f}_\theta])\big)}
\end{align*}
where the second equality holds for $t\geq t_1$. This shows that \eqref{eq:c_stareq} holds with $c_{\mathrm{A4}}=\E[\log (f_{\theta_\star})]-c-\log (\sup_{\theta\not\in A_\star}  \E[ \tilde{f}_\theta] )>0$.

The proof of the lemma is complete.\hfill$\square$

\subsubsection{Proof of Lemma \ref{lemma:part_1}\label{p-lemma:part_1}}

Let $(\phi_t)_{t\geq 1}$ and $ \tilde{D}_\star\in \R_{>0}$ be as  Lemma \ref{lemma:test}, $\tilde{C}_\star\in \R_{>0}$ be as in Lemma \ref{lemma:denom}. For   every $t\geq 1$, let $C^{(\mu_{s_t}*\tilde{\pi}_{s_t})}_{\delta, \tilde{\delta}}$   be as defined in Lemma \ref{lemma:denom} and let
\begin{align*}
 A_{t}=& \Big\{y_{1:t}\in\setY^t:\int_{\Theta}\E\Big[(\mu_{s_t}*\tilde{\pi}_{s_t})\big(\theta-\sum_{s=s_t+1}^{t-1}U_s\big)\prod_{s=s_t+1}^{t} \frac{\tilde{f}_{\theta-\sum_{i=s}^{t-1}U_i}}{f_{\theta_\star}}(Y_s)\big| Y_{1:t}=y_{1:t}\Big] \dd\theta\\
 &\leq  \P(U_{s_t:t}\in \Theta_{\tilde{\delta},s_t:t}) C^{(\mu_{s_t}*\tilde{\pi}_{s_t})}_{\delta , \tilde{\delta} }   e^{-2(t-s_t) (\tilde{C}_\star \delta)^2}\Big\}.
\end{align*}

Let  $\phi'_t(Y_{1:t})=\ind_{A _{t}}(Y_{1:t})\vee \phi_t(Y_{1:t})$ and note that by, Lemmas \ref{lemma:denom}-\ref{lemma:test},
\begin{equation}\label{eq:den}
\begin{split}
\limsup_{t\rightarrow\infty}\E\big[\phi'_t(Y_{1:t})\big]&\leq \limsup_{t\rightarrow\infty}\E\big[\phi_t(Y_{1:t})\big]+\limsup_{t\rightarrow\infty}\P\big(Y_{1:t} \in A_t\big)\\
&\leq \limsup_{t\rightarrow\infty}\frac{1}{\delta^2 (t-s_t) \tilde{C}^2_\star }\\
&=0
\end{split}
\end{equation}
as required.

On the other hand  we have,   $\P$-a.s.
\begin{align*}
&\E\big[(1-\phi'_t(Y_{1:t}))\ind_{ \Theta_{\gamma_t,s_t:t}}(U_{s_t:t})\pi'_{s_t,t}(V_{\epsilon})|\F_{s_t}\big]\\
&\leq  \frac{e^{2(t-s_t)   (\tilde{C}_\star \delta)^2}}{\P(U_{s_t:t}\in \Theta_{\tilde{\delta},s_t:t}) C^{(\mu_{s_t}*\tilde{\pi}_{s_t})}_{\delta, \tilde{\delta}}}\\
&\times \E\Big[\ind_{ \Theta_{\gamma_t,s_t:t}}(U_{s_t:t})\int_{V_{\epsilon}}(1-\phi_t(Y_{1:t}))(\mu_{s_t}*\tilde{\pi}_{s_t})\Big(\theta-\sum_{s=s_t+1}^{t-1}U_s\Big)\prod_{s=s_t+1}^{t} \frac{\tilde{f}_{\theta-\sum_{i=s}^{t-1}U_i}}{f_{\theta_\star}}(Y_s) \dd\theta\Big|\F_{s_t}\Big] \\
&\leq \frac{e^{-(t-s_t)(\tilde{D}_\star  -2    \tilde{C}^2_\star \delta^2)})}{\P(U_{s_t:t}\in \Theta_{\tilde{\delta},s_t:t}) C^{(\mu_{s_t}*\tilde{\pi}_{s_t})}_{\delta, \tilde{\delta}}}
\end{align*}
where the second inequality  holds for $t$ large enough and uses  Tonelli's theorem and  Lemma \ref{lemma:test}. This completes the proof of  the lemma with $ C_1=\tilde{D}_\star^{-1}$ and $C_2=2\tilde{C}^2_\star$.

\section{Complement to Section \ref{sub:extension} of the paper\label{psec:extensions}}

\subsection{G-PFSO based on other Markov kernels $(\widetilde{M}_{t})_{t\geq 1}$\label{sub:genMp} }

It should be clear that  if in \eqref{eq:mu_def} we define  $(\widetilde{M}_{t})_{t\geq 1}$ using Gaussian and Student's $t$-distributions only   our theoretical analysis of $\tilde{\pi}_t$ applies more generally for Markov kernels $(\widetilde{M}_{t})_{t\geq 1}$ whose tails verify certain conditions.

Notably, from the proof of Theorem \ref{thm:online} and of Lemma \ref{lem:lower_bound_1} it is direct to see that a sufficient condition on $(\widetilde{M}_{t_p})_{p\geq 0}$ for  the conclusion of these two results (and hence of Theorem \ref{t-thm:online}) to hold is that    there exist  constants  $(C,\bar{\nu},\underline{\nu})\in(0,\infty)^3$ such that, for all $p\in\mathbb{N}_0$ and $\gamma>0$, we have
$$
 \frac{1}{C}\, \P\big(\|\theta_p'\|\geq \gamma\big)\leq \P\big(\|\theta_p\|\geq \gamma\big)\leq C \P\big(\|\theta_p''\|\geq \gamma\big)
$$
where $\theta_p\sim \widetilde{M}_{t_p}(0,\dd\theta)$,  $\theta_p'\sim t_{d,\underline{\nu}}(0,h_{t_p}^2\Sigma)$ and $\theta_p''\sim t_{d,\bar{\nu}}(0,h_{t_p}^2\Sigma)$. 

In particular, from this observation it follows that   Theorem \ref{t-thm:online} also applies if, for some $\mathfrak{w}\in [0,1)$,  for all $p\in\mathbb{N}_0$ the Markov kernel $\widetilde{M}_{t_p}$ is defined by
\begin{enumerate}
\item\label{M1} $\widetilde{M}_{t_p}(\theta',\dd\theta)= \mathfrak{w}\,\mathcal{N}_d(\theta',h_{t_p}^2\Sigma)+(1-\mathfrak{w})t_{d,\nu}(\theta',h_{t_p}^2\Sigma)$,

\item\label{M2} $\widetilde{M}_{t_p}(\theta',\dd\theta)= \mathfrak{w}\,\delta_{\{\theta'\}}(\dd\theta)+(1-\mathfrak{w})t_{d,\nu}(\theta',h_{t_p}^2\Sigma)$,

\item\label{M3} $\widetilde{M}_{t_p}(\theta',\dd\theta)= \mathfrak{w}\, t_{d,\nu'}(\theta',h_{t_p}^2\Sigma)+(1-\mathfrak{w})t_{d,\nu}(\theta',h_{t_p}^2\Sigma)$ for some $\nu'\in (0,\nu)$.
\end{enumerate}

Remark that in the above definitions of  $(\widetilde{M}_{t_p})_{p\geq 0}$ we recover the definition \eqref{eq:mu_def} of this sequence when $\mathfrak{w}=0$. As   explained in Section \ref{sub-discuss_sup}, being able to take $\mathfrak{w}>0$  in the above definitions \ref{M1}-\ref{M3} of $(\widetilde{M}_{t_p})_{p\geq 0}$ may be useful in practice  when a small value of $\nu$ is used.

The constraint that $(\widetilde{M}_t)_{t\not\in (t_p)_{p\geq 0}}$ is a sequence of  Gaussian kernels can be relaxed as well. Informally speaking, for Theorem \ref{t-thm:online} to hold  it is enough that, for every $t\not\in (t_p)_{p\geq 0}$, the tails of $\widetilde{M}_t(\theta',\dd\theta)$ are of the same size as those of the $\mathcal{N}_d(\theta',h_t^2 I_d)$ distribution. There is however no clear practical interest of such a  generalization of the sequence $(\widetilde{M}_t)_{t\not\in (t_p)_{p\geq 0}}$ and, for this reason, we do not discuss   further this extension of G-PFSO.

\subsection{G-PFSO with a provable non-vanishing ability to escape from a mode}

Consider Algorithm \ref{algo:pi_tilde_online} with $\cess=1$ and, for every $(y,\theta^1,\dots,\theta^N)\in\setY \times\Theta^N$, let   $p_{N}\big(y, \{\theta^n\}_{n=1}^N\big)\in[0,1]$ be  the probability to choose $\theta^1$ at least once under the resampling algorithm $\mathcal{R}(\cdot)$   when the observation is $y$ and the particles are $\{\theta^n\}_{n=1}^N$, that is
$$
p_{N}\big(y, \{\theta^n\}_{n=1}^N\big)=\P\bigg(\theta_1\in  \mathcal{R}\Big(\Big\{\theta^n,\frac{f_{\theta^n}(y)}{\sum_{m=1}^N f_{\theta^m}(y)}\Big\}_{n=1}^N\Big)\bigg).
$$

The following proposition shows that if $\alpha\nu\leq 1$ then there is a zero probability to see, for $t$ large enough, all the particles of Algorithm \ref{algo:pi_tilde_online} stuck in a given mode of the objective function.  
\begin{proposition}\label{prop:explore}
Consider Algorithm \ref{algo:pi_tilde_online} with $\cess=1$ and $N\geq 2$. Let $A,B\in\mathcal{B}(\Theta)$ be such that  $A\subset B$ and such that  $\inf_{(\theta,\theta') \in A\times B}\|\theta-\theta'\|>0$. Assume that $f_\theta(y)=0$ for all $(\theta,y)\in (B\setminus A)\times\setY$ and that there exists a  set $C\in\mathcal{B}(\Theta)$ such that $C\cap A=\emptyset$ and such that, for some constant $c>0$, 
$$
\P(\tilde{\setY})>0,\quad \tilde{\setY}:=\Big\{ p_{N}\big(Y_1,\{\theta^n\}_{n=1}^N\big)\geq c,\quad\forall \theta^1\in C \text{ and }\forall (\theta^2,\dots,\theta^N)\in \Theta^{N-1}\Big\}.
$$ 
Assume also that there exists a set $S\in\mathcal{B}(\R^d)$ such that $\int_{S}\dd\theta>0$ and such that $\theta+\epsilon\in C$ for all $(\theta,\epsilon)\in A\times S$. Let $h_t=t^{-\alpha}$ for some $\alpha>0$ and   $(\widetilde{M}_{t_p})_{p\geq 0}$ be as defined in Section \ref{sub:genMp} with $\nu\leq 1/\alpha$. Then,
\begin{align*}
\P\Big(\liminf_{t\rightarrow \infty}\big\{\hat{\theta}^n_{t}\in A,\,\forall n\in\{1,\dots,N\}\big\}\Big)=0.
\end{align*}
\end{proposition}

\begin{remark}
In Proposition \ref{prop:explore} the conditions on  $A$ imply that $A$ contains a local or the global maximizer of the objective function $\theta\mapsto\E[\log f_\theta(Y_1)]$. 
\end{remark}

\begin{remark}
In practice the assumption $\P(\tilde{\setY})>0$  will typically hold if $C$ is a bounded set  and $\sup_{(\theta,y)\in\Theta\times\setY}f_\theta(y)<\infty$, while the existence of a set $S$ as defined in Proposition \ref{prop:explore} requires that $C$ is sufficiently large compared to $A$.
\end{remark}

\noindent\textit{Proof of Proposition \ref{prop:explore}.} Let $1:N=\{1,\dots,N\}$, $(U_t)_{t\geq 1}$ be a sequence of i.i.d.\ $\Unif((0,1))$ random variables (independent of $(Y_t)_{t\geq 1}$) and  $(\epsilon_t^1)_{t\geq 1}$ be a sequence of independent random variables (independent of $(Y_t)_{t\geq 1}$ and of $(U_t)_{t\geq 1}$) such that $\epsilon_t^1\sim\widetilde{M}_t(0,\dd\theta)$ for all $t\geq 1$. Then, for  all $t\geq 2$ we have
\begin{equation}\label{eq:set_prop}
\begin{split}
 \Big\{\exists n\in 1:N \text{ s.t. }  \hat{\theta}^n_{t}\in C\Big\}&\supset  \Big\{\exists n\in 1:N \text{ s.t. } \hat{\theta}^n_{t}\in C,\, \hat{\theta}^{1}_{t-1}\in A \Big\}\\
 &\supset \Big\{\epsilon^1_{t}\in S,\, Y_{t}\in \tilde{\setY}, U_{t}\leq c,\,\hat{\theta}^{1}_{t-1}\in A\Big\}\\
&=\Big\{\epsilon^1_{t}\in S,\, Y_{t}\in \tilde{\setY},\, U_{t} \leq c\Big\}\bigcap\Big\{\hat{\theta}^{1}_{t-1}\in A\Big\}.
\end{split}
\end{equation}

We now let
$$
C_\nu=\frac{\Gamma\big((\nu+d)/2\big)}{\Gamma(\nu/2)(\pi\nu)^{d/2}|\Sigma|^{1/2}}
$$
so that, for some non-empty interval $[a,b]\subset S$ and every $p\geq 0$ we have
\begin{align*}
\P\big(\epsilon^1_{t_p}\in S\big)&\geq \P(\epsilon^1_{t_p}\in [a,b]\big)\\
&=C_\nu h_{t_p}^{-d}\int_{[a,b]}\Big(1+\frac{1}{h^2_{t_p}\nu} x^\top\Sigma^{-1}x\Big)^{-\frac{\nu+d}{2}}\\
&=C_\nu h_{t_p}^{\nu}\int_{[a,b]} \Big(h^2_{t_p}+\frac{1}{\nu} x^\top\Sigma^{-1}x\Big)^{-\frac{\nu+d}{2}}\\
&\geq C_\nu h_{t_p}^{\nu}\int_{[a,b]} \Big(1+\frac{1}{\nu} x^\top\Sigma^{-1}x\Big)^{-\frac{\nu+d}{2}}.
\end{align*}

Since by assumption we have $\alpha  \nu\leq 1$, we have
\begin{align*}
\sum_{t\geq 1}\P\Big(\epsilon^1_{t}\in S,\, Y_{t}\in \tilde{\setY},\, U_{t} \leq c\Big)&=\sum_{t\geq 1} c\,\P\big(\epsilon^1_{t}\in S\big)\P\big(Y_{t}\in \tilde{\setY}\big) \\
&=c\, \P\big(Y_{1}\in \tilde{\setY}\big) \sum_{t\geq 1}  \P\big(\epsilon^1_{t}\in S\big)\\
&\geq c\, \P\big(Y_{1}\in \tilde{\setY}\big) \sum_{p\geq 0}  \P\big(\epsilon^1_{t_p}\in S\big)\\
&=\infty
\end{align*}
and thus, by the second Borel-Cantelli lemma,  there exists a $\Omega'\in\F$   such that $\P(\Omega')=1$ and such that, for all $\omega\in\Omega'$, the event $\big\{\epsilon^1_{t}(\omega)\in S,\, Y_{t}(\omega)\in \tilde{\setY},\, U_{t}(\omega)\leq c\big\}$ occurs for infinitely many $t\in\mathbb{N}$.

To complete the proof remark  first that  if the event $\big\{\exists n\in 1:N \text{ s.t. } \hat{\theta}^n_{t}\in C\big\}$ occurs for infinitely many $t\in\mathbb{N}$ then, since $C\cap A=\emptyset$, the event $\big\{\exists n\in 1:N\text{ s.t. }\hat{\theta}^n_{t}\not\in A\big\}$ also occurs for infinitely many $t\in\mathbb{N}$. Therefore,
\begin{equation}\label{eq:limsup_prob1}
\begin{split}
\P\Big(&\limsup_{t\rightarrow\infty}\big\{\exists n\in 1:N\text{ s.t. }\hat{\theta}^n_{t}\not\in A\big\},\,\limsup_{t\rightarrow\infty}\big\{\exists n\in 1:N \text{ s.t. } \hat{\theta}^n_{t}\in C\big\}\Big)\\
&=\P\Big(\limsup_{t\rightarrow\infty}\big\{\exists n\in 1:N \text{ s.t. } \hat{\theta}^n_{t}\in C\big\}\Big).
\end{split}
\end{equation}
Next, remark that if $w\in\Omega'$ is such that the event  $\big\{\exists n\in 1:N \text{ s.t. } \hat{\theta}^n_{t}(\omega)\in C\big\}$ occurs only for finitely many $t\in\mathbb{N}$ then,   by \eqref{eq:set_prop},  the event $\big\{\hat{\theta}^{1}_{t}(\omega)\not\in A\big\}$ occurs for infinitely many $t$. Hence,
\begin{equation}\label{eq:limsup_prob2}
\begin{split}
&\P\Big(\limsup_{t\rightarrow\infty}\big\{\exists n\in 1:N\text{ s.t. }\hat{\theta}^n_{t}\not\in A\big\},\, \big(\limsup_{t\rightarrow\infty}\big\{\exists n\in 1:N \text{ s.t. } \hat{\theta}^n_{t}\in C\big\}\big)^c\Big)\\
&=\P\Big(\limsup_{t\rightarrow\infty}\big\{\exists n\in 1:N\text{ s.t. }\hat{\theta}^n_{t}\not\in A\big\},\,\big(\limsup_{t\rightarrow\infty}\big\{\exists n\in 1:N \text{ s.t. } \hat{\theta}^n_{t}\in C\big\}\big)^c,\,\,\Omega'\Big)\\
&=\P\Big(\big(\limsup_{t\rightarrow\infty}\big\{\exists n\in 1:N \text{ s.t. } \hat{\theta}^n_{t}\in C\big\}\big)^c\Big). 
\end{split}
\end{equation}
Therefore, by \eqref{eq:limsup_prob1}-\eqref{eq:limsup_prob2},
\begin{align*}
\P\Big(\limsup_{t\rightarrow\infty}&\big\{\exists n\in 1:N\text{ s.t. }\hat{\theta}^n_{t}\not\in A\big\}\Big)=1\Leftrightarrow \P\Big(\liminf_{t\rightarrow\infty}\big\{\hat{\theta}^n_{t}\in A,\,\forall n\in 1:N \big\}\Big)=0.
\end{align*}
 The proof is complete.\hfill $\square$

\subsection{Discussion}\label{sub-discuss_sup}

Fulfilling with $\alpha=0.5$ the condition $\nu \leq 1/\alpha$ imposed in Proposition \ref{prop:explore} requires to let $\nu\leq 2$, that is to use in the definition of $(\widetilde{M}_{t_p})_{p\geq 0}$ Student's $t$-distributions having an infinite variance. As illustrated in Section \ref{sub:mixture} of the paper, G-PFSO may not work well when these kernels are as defined in \eqref{eq:mu_def} for some $\nu\leq 2$, the reason being that when $\nu$ is small there may be a large probability that at time $t\in (t_p)_{p\geq 0}$ the estimator $\tilde{\theta}_{t_p}^N$ is pushed away from $\theta_\star$. For such values of $\nu$,  defining  $(\widetilde{M}_{t_p})_{p\geq 0}$ as  in Section \ref{sub:genMp} for some $\mathfrak{w}>0$ both leads to kernels that verify the conditions Proposition \ref{prop:explore} (assuming $h_t=t^{-1/2}$)  and reduces the probability of this undesirable  event, since as $\mathfrak{w}$ increases  the  probability to have $\theta_{t_p}^n$ close to $\hat{\theta}_{t_p-1}^n$ becomes larger,  for all $t\in (t_p)_{p\geq 0}$ and $n\in\{1,\dots,N\}$. 

We finally recall that the condition  $\nu\leq 1/\alpha$ imposed in Proposition \ref{prop:explore} can  be fulfilled with $\nu>2$, provided that $\alpha<0.5$. However, as illustrated in Section \ref{sub:CQR_5}, when $\alpha<0.5$ the estimator $\bar{\theta}_t^N$ is expected to converge at a sub-optimal rate.

\section{Additional information for the numerical experiments\label{sec:additional_info}}

\subsection{Code}

The code used to produce  all the figures  of the paper  is available  on GitHub   at
\begin{center}
\texttt{https://github.com/mathieugerber/gpfso-paper}
\end{center} 

\subsection{Censored quantile regression model}

Let $\big( Y_t:=(Z_t,X_t)\big )_{t\geq 1}$  be a sequence of i.i.d.\ random variables taking values in $[0,\infty)\times \R^d$  for some $d\geq 1$. Then, with $\Theta=\R^d$, the censored  quantile regression  (CQR) model assumes   that, for every $\tau\in(0,1)$ and $x\in\R^d$, the  $\tau$-th conditional quantile function of  $Z_t$ given $X_t=x$  does not dependent on $t$ and belongs to the set
$\{\max\{x^\top \theta,0\},\, \theta\in\Theta\}$. For every $\tau\in(0,1)$ we let $\rho_\tau:\R\rightarrow\R$ be defined by $\rho_\tau(u)=(|u|+(2\tau-1)u)/2$, $u\in\R$, and  
$
\theta^{(\tau)}_{\star}=\argmin_{\theta\in\Theta}\E\big[\rho_p\big(Z_1-\max\{X_1^\top \theta,0\}\big)\big]
$
be the parameter value of interest. 

In order to cast this estimation problem into the general set-up of this paper, following  \citet{yu2001bayesian}  we let $f_{\tau,\theta}(\cdot|x)$ be the density of  the asymmetric Laplace distribution with location parameter  $\max\{x^\top \theta,0\}$,  scale parameter $\sqrt{\tau(1-\tau)}$ and asymmetry parameter $\sqrt{\tau/(1-\tau)}$; that is
\begin{align}\label{eq:lap}
f_{\tau,\theta}(z|x)=\tau(1-\tau)\exp\big\{-\rho_p(z-\max\{x^\top \theta,0\})\big\},\quad\forall z\in\R.
\end{align}
Then,  $\theta^{(\tau)}_{\star}$ can be redefined as $\theta^{(\tau)}_{\star}=\argmax_{\theta\in\Theta}\E[\log f_{\tau,\theta}(Z_1|X_1)]$. 

We let $d=5$ and simulate $T=10^7$   independent observations $\{(z_t,x_t)\}_{t=1}^T$ according to
\begin{align*}
Z_t=\max(\tilde{Z}_t,0),\quad \tilde{Z}_t|X_t\sim\mathcal{N}_1\big(X_t^\top \theta^{(0.5)}_{\star},4 \big),\quad X_t\sim \delta_{\{1\}}\otimes \mathcal{N}_{d-1}(0,\Sigma_X),\quad t\geq 1 
\end{align*}
where $\Sigma^{-1}_X$ is a random draw from the Wishart distribution with $d-1$ degrees of freedom and  scale matrix  $I_{d-1}$, while   $\theta^{(0.5)}_{\star}=(3,\theta^{(0.5)}_{\star,2:d} )$ with  $\theta^{(0.5)}_{\star,2:d}$ a random draw from the $\mathcal{N}_{d-1}(0,I_{d-1})$ distribution. The resulting  sample $\{(z_t,x_t)\}_{t=1}^T$ is such that about 13\% of the observations are censored.

\subsection{Smooth adaptive Gaussian mixture model}

For this example we generate  $T=2\times 10^6$ independent  observations $\{(z_t,x_t)\}_{t=1}^T$ using
\begin{align}\label{eq:gen_mixture}
Z_t|X_t\sim f_{\theta_\star}(z|X_t)\dd z,\quad X_t\sim \delta_{\{1\}}\otimes\mathcal{N}_{d_x-1}(0, I_{d_x-1}),
\end{align}
where $\theta_\star=(\beta_\star^{\mathrm{w}},\beta_\star^{\mu},\beta_\star^{\sigma})$ with
 $\beta_\star^{\mathrm{w}}=(1,0.1,0.1,-0.1)$, $\beta_{\star,(1)}^{\mu}=(1,1,1,-1)$,   $\beta_{\star,(2)}^{\mu}=(-1,1,1,1)$, $\beta_{\star,(1)}^{\sigma}=(0,1,1,1)$ and   $\beta_{\star,(2)}^{\sigma}=(0.5,-1,-1,1)$. For this choice of $\theta_\star$, and in average  over the $x_t$'s, the $k$-the component of  $f_{\theta_\star}(\cdot|x_t)$ has a mean and  a variance\footnote{To remove the effect of a few extreme variances these numbers are actually the mean values of the variances which are smaller than 100.} respectively equal to  1 and to 7.28 for $k=1$, and respectively equal to $-1$ and to $5.1$ for $k=2$, while the first component of $f_{\theta_\star}(\cdot|x_t)$ has a weight   approximatively equal to 0.27. 
 
\end{document}